\documentclass[10pt,twocolumn,letterpaper]{article}

\usepackage[pagenumbers]{cvpr} % To force page numbers, e.g. for an arXiv version

\usepackage{comment}
\usepackage{amsmath}
\usepackage{amssymb}
\usepackage{enumerate}

\usepackage{amsthm}               % for proof
\usepackage{booktabs}

\usepackage{epsfig}
\usepackage{tabularx}
\usepackage[table,dvipsnames]{xcolor} % for colored stuff
\usepackage{booktabs}
\usepackage{relsize}
\usepackage{multirow}
\usepackage{array}               % Thicker vertical rules in tables, lines in only certain columns
\usepackage{makecell}            % For thicker horizontal rules in tables
\usepackage{subcaption}          % For subfigures
\usepackage{wasysym}             % For right filled arrow
\usepackage{lipsum}              % For random filler text
\usepackage{esvect}              % For vv command arrow above symbols
\usepackage{soul}                % For striked text command \st and hightlight
\usepackage{float}
\usepackage{caption}
\usepackage{bm}
\usepackage{enumitem}            % for noindent itemize
\usepackage{empheq}
\usepackage{gensymb}             % for degree symbol
\usepackage{wrapfig}             % for wraptable
\usepackage{xspace}              % for dynamic controlled space at end of abbreviation
\usepackage[thicklines]{cancel}  % for cancel
\usepackage{arydshln}            % for horizontal dash line
\usepackage{pifont}              % for ticks and cross http://ctan.org/pkg/pifont
\usepackage{hhline}              % for hhline
\usepackage[titles]{tocloft}             % Adjusting TOC
\usepackage{tikz}                % For dashed lines in overview
\usepackage{nicematrix}          % For dashed lines in overview
\usepackage{cuted}               % For 1st figure
\usepackage{stfloats}
\newsavebox{\nvsfigbox}          % For row above qualitative figures
\newlength{\imgh}

\setlist[itemize, 1]{label =\raisebox{-0.1\height}{\scalebox{1.3}{\textbullet}}}
\setlist[itemize]{noitemsep, topsep=-0.05cm, leftmargin=4mm}

\allowdisplaybreaks              % Break aligned aligns

\graphicspath{{images/}}
\newtheorem{theorem}{Theorem}
\newtheorem{corollary}{Corollary}[theorem]
\newtheorem{lemma}{Lemma}

\definecolor{sns_blue}{rgb}{0.21, 0.06, 0.42}        % (54,  115, 107)
\definecolor{sns_violet}{rgb}{0.45, 0.12, 0.51}      % (115,  31, 130)
\definecolor{sns_orange}{rgb}{0.75, 0.22, 0.46}      % (191,  56, 117)
\definecolor{sns_orange_two}{rgb}{0.94, 0.38, 0.36}  % (240,  97,  92)
\definecolor{sns_yellow}{rgb}{0.99, 0.91, 0.66}      % (252, 232, 168)
\definecolor{yellow}{rgb}{1, 1, 0.7}
\definecolor{orange}{rgb}{1, 0.85, 0.7}
\definecolor{tablered}{rgb}{1, 0.7, 0.7}

\definecolor{bestBlue}{RGB}{140,180,255}
\definecolor{secondBlue}{RGB}{200,220,255}
\definecolor{thirdBlue}{RGB}{229,238,255}
\definecolor{bestBlueDark}{RGB}{68,114,196}%93,164,241}

\definecolor{gain}{HTML}{34a853}
\definecolor{lost}{HTML}{ea4335}

\definecolor{my_blue}{rgb}{0.2, 0.6, 1}  % dodgerblue
\definecolor{my_magenta}{rgb}{1.0, 0.2, 0.6} % triadic to dodgerblue
\definecolor{my_yellow}{rgb}{1.0, 0.8, 0.2} % triadic to magenta
\definecolor{my_green}{rgb}{0.0, 0.9, 0.24}
\definecolor{my_green_2}{rgb}{0.0, 0.4, 0.0}

\definecolor{white}{rgb}{1.0, 1.0, 1.0}
\definecolor{darkGreen}{rgb}{0.01, 0.8, 0.24}%{0.31, 0.94, 0.3}%{0.09,0.88,0.00}{0.29, 0.83, 0.38}
\definecolor{darkGreen2}{rgb}{0.22,0.42, 0.33}
\definecolor{darkGreen3}{rgb}{0.20,0.66, 0.33}
\definecolor{cvprblue}{rgb}{0.21,0.49,0.74}
\definecolor{LightCyan}{rgb}{0.88,1,1}
\definecolor{lightgreen}{HTML}{90EE90}
\definecolor{new_green}{rgb}{0.75,0.97,0.44}
\definecolor{Gray}{gray}{0.95}
\definecolor{lightgray}{rgb}{0.96, 0.96, 0.96}

\definecolor{set1_cyan}{rgb}{0.23, 0.87, 1.0}
\definecolor{building}{rgb}{0.2, 0.33, 0.33}

\definecolor{my_violet}{rgb}{0.79, 0.40, 1} %{0.73,0.62,0.91}
\definecolor{my_yellow_2}{rgb}{0.9, 0.8, 0.54}
\definecolor{my_red}{rgb}{1,0,0}
\definecolor{my_purple}{rgb}{0.27,0.8, 0.8}
\definecolor{my_orange}{rgb}{1.0,0.6,0.35}
\definecolor{my_golden}{rgb}{1.0, 0.75, 0.0}
\colorlet{my_gray}{gray!12}

\definecolor{projectionColor}{rgb}{0.2, 0.6, 1}
\definecolor{rayColor}{rgb}{0.0,0.0,0.0}
\definecolor{axisColor}{rgb}{0.0, 0.0, 0.0}
\colorlet{projectionBorderShade}{rayColor!100}
\colorlet{projectionFillShade}{projectionColor!20}
\colorlet{rayShade}{my_yellow}
\colorlet{axisShade}{axisColor!20}
\colorlet{axisShadeDark}{axisColor!100}

\definecolor{backward_color}{rgb}{1.0, 0.6, 0.2}
\definecolor{forward_color}{rgb}{0.2, 1.0, 0.6}
\colorlet{proposedShade}{darkGreen}
\colorlet{vanillaShade}{red!90}
\colorlet{theme_color}{bestBlueDark}%my_yellow
\colorlet{theme_color_light}{sns_yellow!25}
\colorlet{link_color}{Maroon}
\colorlet{ood_color}{gray!20}
\colorlet{methodColor}{cyan!20}

\newcommand{\noIndentHeading}[1]{\noindent\textbf{#1}}

\newcommand{\supp}{Supp\xspace}
\newcommand{\llm}{LLM\xspace}

\definecolor{XLcolor}{rgb}{0.858, 0.188, 0.478}
\newcommand{\forExample}{\textit{e.g.}\xspace}

\newcommand{\cmark}{\checkmark}%\ding{51}}%
\newcommand{\xmark}{\ding{53}}

\newcommand{\myHat}[1]{\widehat{#1}}
\newcommand{\mySp}{\hspace{0.04cm}}

\newcommand{\myTopRule}{\Xhline{2\arrayrulewidth}}

\newcolumntype{t}{!{\vrule width 1.5\arrayrulewidth}}
\newcolumntype{m}{!{\vrule width 2.5\arrayrulewidth}}
\newcolumntype{a}{>{\columncolor{theme_color_light}}l}
\newcolumntype{b}{>{\columncolor{theme_color_light}}c}
\newcolumntype{d}{>{\columncolor{ood_color}}c}

\colorlet{cyan_highlight}{my_blue!85}

\colorlet{darkGreen_highlight}{darkGreen!75}

\colorlet{my_magenta_highlight}{my_magenta!50}

\colorlet{my_yellow_highlight}{my_yellow!55}

\providecommand\rightarrowRHD{\relbar\joinrel\mathrel\RHD}

\newcommand{\uparrowRHD}  {\rotatebox[origin=c]{90}{$\rightarrowRHD$}}
\newcommand{\downarrowRHD}{\rotatebox[origin=c]{270}{$\rightarrowRHD$}}

\newcommand{\uparrowRHDSmall}  {\raisebox{0.05\normalbaselineskip}{\scalebox{0.7}{\uparrowRHD}}}   %$\uparrow$
\newcommand{\downarrowRHDSmall}{\raisebox{0.07\normalbaselineskip}{\scalebox{0.7}{\downarrowRHD}}} %$\downarrow$
\newcommand{\rightarrowRHDSmall}{\raisebox{0.05\normalbaselineskip}{\scalebox{0.75}{$\rightarrowRHD$}}}   %$\uparrow$
\newcommand{\uparrowRHDSmallTwo}  {\raisebox{0.17\normalbaselineskip}{\scalebox{0.56}{\uparrowRHD}~}}  
\newcommand{\downarrowRHDSmallTwo}  {\raisebox{0.17\normalbaselineskip}{\scalebox{0.56}{\downarrowRHD}}}  

\newcommand{\cvpr}{CVPR\xspace}
\newcommand{\eccv}{ECCV\xspace}
\newcommand{\iccv}{ICCV\xspace}
\newcommand{\iclr}{ICLR\xspace}

\newcommand{\nips}{NeurIPS\xspace}
\newcommand{\sigg}{SIGG\xspace}
\newcommand{\arxiv}{ArXiv\xspace}

\newcommand{\venue}[1]{\scriptsize{#1}}

\newcommand{\nvs}{NVS\xspace}
\newcommand{\nerf}{NeRF\xspace}
\newcommand{\nerfs}{NeRFs\xspace}
\newcommand{\gsOnly}{GS\xspace}
\newcommand{\gs}{3DGS\xspace}

\newcommand{\sfm}{SfM\xspace}
\newcommand{\sdf}{SDF\xspace}
\newcommand{\sh}{SH\xspace}

\newcommand{\mipNerf}{Mip-\nerf~{360}\xspace}
\newcommand{\tAndT}{Tank \& Temples\xspace}
\newcommand{\deepBlend}{Deep Blending\xspace}
\newcommand{\ommo}{OMMO\xspace}
\newcommand{\shelly}{Shelly\xspace}
\newcommand{\synNerf}{Syn-\nerf}

\newcommand{\psnr}{PSNR\xspace}
\newcommand{\ssim}{SSIM\xspace}
\newcommand{\lpips}{LPIPS\xspace}
\newcommand{\psnrEdge}{PSNR$_\text{Edge}$\xspace}
\newcommand{\fps}{FPS\xspace}
\newcommand{\db}{dB\xspace}
\newcommand{\sep}{\textbf{\textcolor{black!90}{/}}\xspace}

\newcommand{\firstColor}{bestBlue}%tablered
\newcommand{\secondColor}{secondBlue}%orange
\newcommand{\thirdColor}{thirdBlue}%yellow

\newcommand{\genericB}[2]{%
    \hspace{-0.7\fboxsep}%
    \begingroup%
    \setlength{\fboxsep}{0.06cm}% Controls horizontal width (keep it small)
    \colorbox{#2}{%
        \rule[-0.4ex]{0pt}{2.2ex}% Forces vertical height (adjust 2.6ex for more/less height)
        \hspace{0.1em}{#1}\hspace{0.1em}% Optional: fine-tune internal width
    }%
    \endgroup%
    \hspace{-0.7\fboxsep}%
}
\newcommand{\firstB}[1]{\genericB{$\mathbf{\!#1\!}$}{\firstColor}}
\newcommand{\secondB}[1]{\genericB{${#1}$}{\secondColor}}
\newcommand{\thirdB}[1]{\genericB{${#1}$}{\thirdColor}}

\newcommand{\firstBText}[1]{\hspace{-0.7\fboxsep}\begingroup\setlength{\fboxsep}{0.06cm}\colorbox{\firstColor}{\textbf{#1}}\endgroup\hspace{-0.7\fboxsep}}
\newcommand{\secondBText}[1]{\hspace{-0.7\fboxsep}\begingroup\setlength{\fboxsep}{0.06cm}\colorbox{\secondColor}{{#1}}\endgroup\hspace{-0.7\fboxsep}}
\newcommand{\thirdBText}[1]{\hspace{-0.7\fboxsep}\begingroup\setlength{\fboxsep}{0.06cm}\colorbox{\thirdColor}{{#1}}\endgroup\hspace{-0.7\fboxsep}}

\newcommand{\firstF}[1]{\cellcolor{\firstColor}$\mathbf{\!#1\!}$}
\newcommand{\secondF}[1]{\cellcolor{\secondColor}${#1}$}
\newcommand{\thirdF}[1]{\cellcolor{\thirdColor}${#1}$}

\newcommand{\gain}[1]{\textcolor{gain}{#1}}

\newcommand{\good}[2]{{$\mathbf{#1}$} {({\gain{\textbf{+}$\mathbf{#2}$}})}}

\newcommand{\nothing}[1]{{$#1$} {({{$+0.00$}})}}

\newcommand{\sota}{SoTA\xspace}
\newcommand{\sotaFull}{state-of-the-art\xspace}
\newcommand{\retrained}{$\!^*$}

\newcommand{\taken}{$^\dagger$}
\def\relicon{\resizebox{.009\textwidth}{!}{\includegraphics{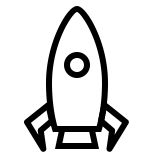}}}
\newcommand{\released}{$\!^{\relicon}$}

\newcommand{\mathDash}{$-$}
\newcommand{\avg}{Avg\xspace}

\newcommand{\twoDGs}{2DGS\xspace}
\newcommand{\mcmc}{MCMC\xspace}
\newcommand{\vdgs}{VDGS\xspace}
\newcommand{\mipSplat}{Mip-Splat\xspace}
\newcommand{\dbs}{$\beta$-Splat\xspace}
\newcommand{\sv}{SV\xspace}
\newcommand{\nht}{NHT\xspace}
\newcommand{\threeDGut}{3DGUT\xspace}
\newcommand{\radMesh}{Radiance Mesh\xspace}
\newcommand{\powerfoam}{PFoam\xspace}

\newcommand{\mipNerfMethod}{Mip-\nerf}
\newcommand{\ever}{EVER\xspace}
\newcommand{\zipNerf}{Zip-\nerf}

\newcommand{\nerfstudio}{Nerfstudio\xspace}

\newcommand{\adam}{Adam\xspace}

\newcommand{\sr}{SR\xspace}
\newcommand{\lr}{LR\xspace}
\newcommand{\hr}{HR\xspace}

\newcommand{\ffhq}{FFHQ\xspace}
\newcommand{\ffhqFull}{Flickr-Faces-HQ\xspace}
\newcommand{\celebAHQ}{CelebA-HQ\xspace}
\newcommand{\divTwoK}{DIV2K\xspace}
\newcommand{\flikr}{Flickr2K\xspace}
\newcommand{\ost}{OST\xspace}

\newcommand{\gan}{GAN\xspace}
\newcommand{\ddpm}{DDPM\xspace}
\newcommand{\realEsrgan}{Real-ESRGAN\xspace}
\newcommand{\srThree}{SR3\xspace}
\newcommand{\invSR}{InvSR\xspace}

\newcommand{\querySet}{\mathcal{Q}}
\newcommand{\query}{\mathbf{q}}
\newcommand{\queryIndex}{\mathbf{q}_{i}}
\newcommand{\outputSet}{\mathcal{O}}
\newcommand{\params}{\theta}
\newcommand{\nField}{\bm{f}}
\newcommand{\nFieldGT}{\bm{f}}

\newcommand{\nFieldPd}{\myHat{\nField}}
\newcommand{\nFieldLP}{\nFieldGT^{low}}
\newcommand{\nFieldGTLP}{\nFieldGT^{\mySp low}}
\newcommand{\nFieldPdLP}{\nFieldPd^{\mySp low}}
\newcommand{\nFieldMLP}{\nFieldPd_\params^{\mySp MLP}}

\newcommand{\nFieldQNN}{\nFieldPd_\params^{\mySp CCNN}}
\newcommand{\nFieldGS}{\myHat{\nField}_\params^{\mySp GS}}
\newcommand{\nFieldGSQNN}{\myHat{\nField}_\params^{\mySp GS+CCNN}}
\newcommand{\nFieldScalar}{{f}}
\newcommand{\nFieldGTScalar}{{f}}
\newcommand{\nFieldGTLPScalar}{\nFieldGTScalar^{\mySp low}}

\newcommand{\coordinates}{\mathbf{x}}
\newcommand{\opacity}{o}
\newcommand{\opacityGT}{\alpha}
\newcommand{\opacityPd}{\myHat{\alpha}}
\newcommand{\vecCenter}{\bm{\mu}}
\newcommand{\matCov}{\bm{\Sigma}}
\newcommand{\vecQuaternion}{\bm{\mathfrak{q}}}
\newcommand{\vecScale}{\bm{s}}
\newcommand{\matRot}{\bm{R}}
\newcommand{\matScale}{\bm{S}}
\newcommand{\vecColor}{\bm{\scalarColor}}
\newcommand{\scalarColor}{{c}}
\newcommand{\matView}{\bm{W}}
\newcommand{\matJacob}{\bm{J}}
\newcommand{\transmittance}{\tau}
\newcommand{\depth}{\mathfrak{d}}

\newcommand{\background}{\bm{b}}
\newcommand{\inDim}{d}
\newcommand{\outDim}{m}
\newcommand{\numLayer}{L}

\newcommand{\encoding}{\gamma}

\newcommand{\erf}{Erf\xspace}
\newcommand{\sinc}{Sinc\xspace}
\newcommand{\siren}{SIREN\xspace}
\newcommand{\qiren}{QIREN\xspace}
\newcommand{\finer}{FINER\xspace}
\newcommand{\inChannels}{C_{in}}
\newcommand{\lowChannels}{C_{low}}
\newcommand{\outChannels}{C_{out}}
\newcommand{\hidChannels}{C_{h}}
\newcommand{\numData}{N}

\newcommand{\concat}{\oplus}

\newcommand{\Raymap}{Raymap\xspace}
\newcommand{\Plucker}{Pl\"ucker\xspace}

\newcommand{\numPrim}{n}
\newcommand{\gind}{k}
\newcommand{\mixAffineMat}{\mathbf{A}}
\newcommand{\mixAffineBias}{\mathbf{b}}
\newcommand{\mixCoeff}{\pi}
\newcommand{\constant}{\kappa}

\newcommand{\wavelet}{\psi}
\newcommand{\sobolev}{r}
\newcommand{\lambert}{W}
\newcommand{\measFunc}{Y}
\newcommand{\vecZ}{\mathbf{z}}
\newcommand{\mse}{MSE\xspace}
\newcommand{\view}{v}
\newcommand{\numViews}{V}
\newcommand{\mseMean}{\overline{\text{MSE}}}
\newcommand{\psnrTarget}{\mathcal{P}}
\newcommand{\neighborhood}{neighborhood\xspace}
\newcommand{\Neighborhood}{Neighborhood\xspace}
\newcommand{\neighborhoods}{neighborhoods\xspace}

\newcommand{\receptive}{\Delta}
\newcommand{\queryInput}{\mathcal{X}}
\newcommand{\riskOptimal}{\mathcal{R}^*}
\newcommand{\parametrize}{\phi}

\newcommand{\LowFreq}{Low-frequency\xspace}

\newcommand{\lf}{LF\xspace}
\newcommand{\hf}{HF\xspace}

\newcommand{\lowFid}{low-fidelity\xspace}

\newcommand{\highFid}{high-fidelity\xspace}
\newcommand{\HighFid}{High-fidelity\xspace}

\newcommand{\qConvFull}{Coordinate-Convolution\xspace}

\newcommand{\coordConv}{CoordConv\xspace}

\newcommand{\mlp}{MLP\xspace}
\newcommand{\mlps}{MLPs\xspace}
\newcommand{\cnn}{CNN\xspace}
\newcommand{\cnns}{CNNs\xspace}

\newcommand{\qnn}{CCNN\xspace}
\newcommand{\qnns}{CCNNs\xspace}
\newcommand{\coordCnn}{Coordinate-\cnn}

\newcommand{\hashGrid}{hash-grid\xspace}
\newcommand{\hashGrids}{hash-grids\xspace}
\newcommand{\HashGrid}{Hash-grid\xspace}

\newcommand{\freqSym}{f}
\newcommand{\power}{\alpha}

\newcommand{\oneD}{$1$D\xspace}
\newcommand{\twoD}{$2$D\xspace}
\newcommand{\threeD}{$3$D\xspace}
\newcommand{\fourD}{$4$D\xspace}
\newcommand{\fiveD}{$5$D\xspace}
\newcommand{\sixD}{$6$D\xspace}

\newcommand{\val}{Val\xspace}

\newcommand{\loss}{\mathcal{L}}

\newcommand{\lOne}{\loss_1}

\newcommand{\relu}{ReLU\xspace}
\newcommand{\relus}{ReLUs\xspace}

\newcommand{\risk}{\mathcal{R}}

\newcommand{\expect}{\mathbb{E}}

\newcommand{\realDomain}{\mathbb{R}}

\newcommand{\paperTitle}{{Beyond Primitive Bloat: High-Fidelity Gaussian Splatting via \textcolor{theme_color}{C}oordinate-\textcolor{theme_color}{C}onvolution \textcolor{theme_color}{N}eural \textcolor{theme_color}{N}etworks}}

\usepackage{listings}

\definecolor{pyBg}{RGB}{250,251,252}
\definecolor{pyFrame}{RGB}{216,222,228}
\definecolor{pyLineNum}{RGB}{110,119,129}

\lstdefinestyle{qnnpython}{
    language=Python,
    backgroundcolor=\color{pyBg},
    basicstyle=\ttfamily\footnotesize\color{Black},
    keywordstyle=\color{Plum}\bfseries,
    emph={self,nn,torch,Tensor,True,False,None,linspace,meshgrid,stack,expand,cat,permute,Sequential,Conv2d,ReLU,Module},
    emphstyle=\color{RoyalBlue},
    moredelim=**[is][\color{TealBlue}\bfseries]{@@}{@@},
    stringstyle=\color{ForestGreen},
    commentstyle=\color{Gray}\itshape,
    numbers=left,
    numberstyle=\tiny\ttfamily\color{pyLineNum},
    numbersep=8pt,
    xleftmargin=18pt,
    framexleftmargin=14pt,
    frame=single,
    rulecolor=\color{pyFrame},
    breaklines=true,
    breakatwhitespace=false,
    showstringspaces=false,
    tabsize=4,
    captionpos=b,
}

\definecolor{cvprblue}{rgb}{0.21,0.49,0.74}
\usepackage[pagebackref,breaklinks,colorlinks,allcolors=cvprblue]{hyperref}

\usepackage[capitalize]{cleveref}
\crefname{section}{Sec.}{Secs.}
\Crefname{section}{Section}{Sections}
\Crefname{table}{Table}{Tables}
\crefname{table}{Tab.}{Tabs.}

\def\confName{CVPR}
\def\confYear{2027}

\title{\paperTitle}

\author{Abhinav Kumar$^1$, Tristan Aumentado-Armstrong$^{2}$\thanks{Equal Second Authors.}~, Lazar Valkov$^{1*}$\\
Gopal Sharma$^{1}$, Alex Levinshtein$^{2}$, Radek Grzeszczuk$^{1,2}$, Suren Kumar$^{1}$\\
$^{1}$Samsung Research America, AI Center – Mountain View, CA, USA \\
$^{2}$Samsung Research, AI Center – Toronto, ON, Canada~~~~~~~~~~~~~~~~~~~~~~\\
{Project Page: {\url{https://abhi1kumar.github.io/qonvolution/}}}
}

\begin{document}

\maketitle
\begin{strip}
    \vspace{-2cm}
    \centering
    \sbox{\nvsfigbox}{\includegraphics[width=\linewidth]{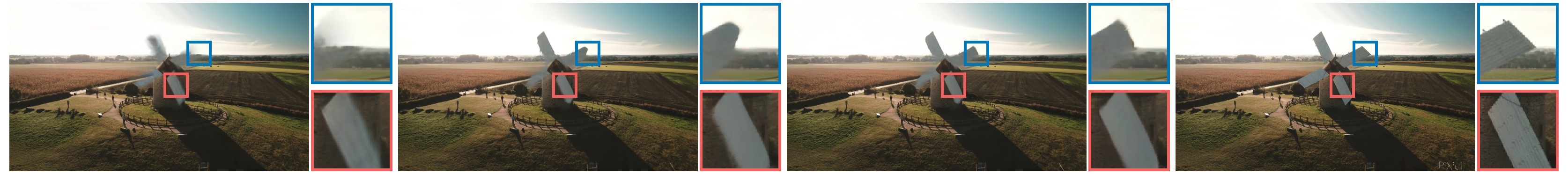}}%
    \setlength{\imgh}{\ht\nvsfigbox}
    \addtolength{\imgh}{\dp\nvsfigbox}

    \hspace*{0.02\linewidth}%
    \makebox[0.243\linewidth]{\small Spherical Voronoi~\cite{di2026spherical}}%
    \makebox[0.243\linewidth]{\small \nht~\cite{condor2026nht}}%
    \makebox[0.243\linewidth]{\small \nht+~\qnn}%
    \makebox[0.243\linewidth]{\small Ground Truth}\par
    \vspace{0.05cm}
    \begin{minipage}[b][\imgh][c]{0.02\linewidth}
        \fontsize{7.5}{9}\selectfont\centering
        \parbox[c][0.2\imgh][c]{\linewidth}{\centering\rotatebox{90}{06~\cite{lu2023large}}}
    \end{minipage}%
    \hspace*{-0.05cm} 
    \usebox{\nvsfigbox}
    \vspace*{-0.6cm} 
    \captionof{figure}{\textbf{High-Fidelity Gaussian Splatting.} 
        While recent baselines successfully increase per-primitive capacity to limit Gaussian bloat, resolving fine structural details remains challenging.
        Augmenting them with \textbf{\qConvFull Neural Network (\qnn) faithfully reconstructs these details}, yielding higher fidelity synthesis than baselines. 
        Inset figures highlight differences.
    }
    \label{fig:teaser}
    \vspace{-0.2cm}
\end{strip}

\addtocontents{toc}{\protect\setcounter{tocdepth}{-2}}
\begin{abstract}
Gaussian Splatting (GS) has revolutionized Novel View Synthesis (NVS) with fast training and real-time rendering, yet pushing its fidelity to arbitrary high levels remains an open challenge. 
We theoretically explain this barrier by proving that discrete Gaussian representations require an exponentially increasing number of primitives to raise training PSNR, which inevitably triggers optimization issues and validation overfitting. 
To bypass this scaling limit without primitive bloat, we formulate detail recovery as a continuous neural field over coarse renderings and prove that monotonically reducing approximation risk to zero requires jointly leveraging low-fidelity signals, local neighborhood, and coordinate queries.
Guided by this result, we augment splatted renderings with an existing Coordinate-Convolution Neural Network (CCNN) that jointly convolves rendered low-fidelity signals with 2D pixel coordinate queries to synthesize residuals.
Integrating CCNN into five diverse Gaussian Splatting baselines establishes new state-of-the-art fidelity across NVS benchmarks while preserving real-time rendering speeds with minimal training and storage overhead. 
Beyond NVS, they also consistently boost performance of 1D/2D regression and 2D super-resolution tasks.
\vspace{-0.65cm}
\end{abstract}

\begin{figure}[!t]
    \centering
    \includegraphics[width=0.8\linewidth]{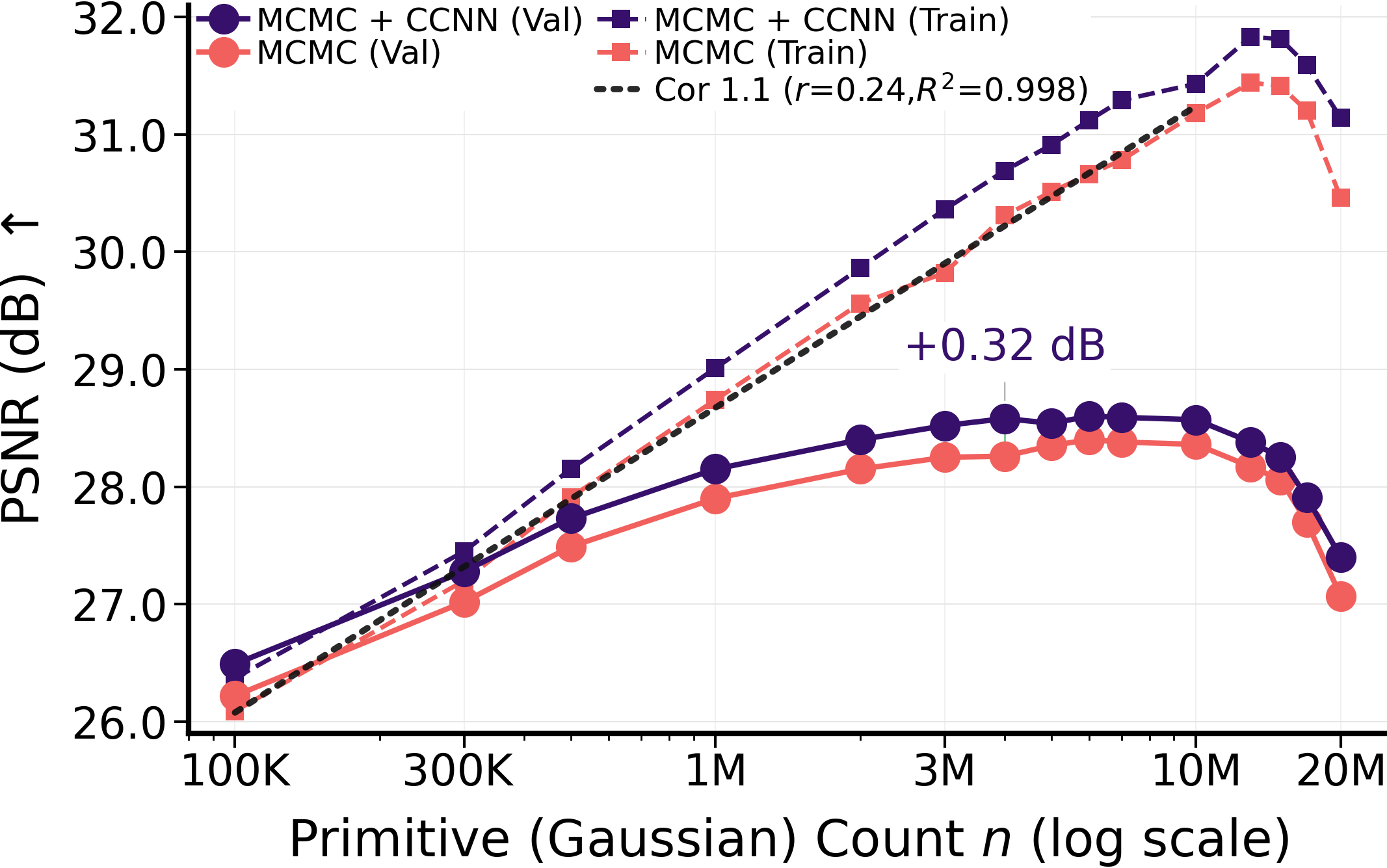}
    \caption{ \textbf{Scaling Bottleneck.}
    Increasing training \psnr requires exponential primitive scaling (\cref{th:num_gaussians}), which the \mcmc \cite{kheradmand20243d} baseline closely follows ($\sobolev=0.24$, $R^2=0.998$) on the \mipNerf dataset \cite{barron2022mipnerf360}.
    However, scaling beyond $4$M primitives triggers optimization issues degrading validation \psnr.
    Continuous \qnn \cite{liu2018intriguing} bypasses this bottleneck, achieving \sota fidelity.
    }
    \label{fig:mcmc_scale_gaussians_trainval}
    \vspace{-0.7cm}
\end{figure}

\begin{figure*}[!t]
    \centering
    \begin{NiceTabular}{ccc}
        \hspace{-0.3cm}
        \Block[borders={right,tikz={densely dashed, thick}}]{1-1}{}%
        \subcaptionbox{\mlp-based Neural Fields}{
            \includegraphics[width=0.28\linewidth]{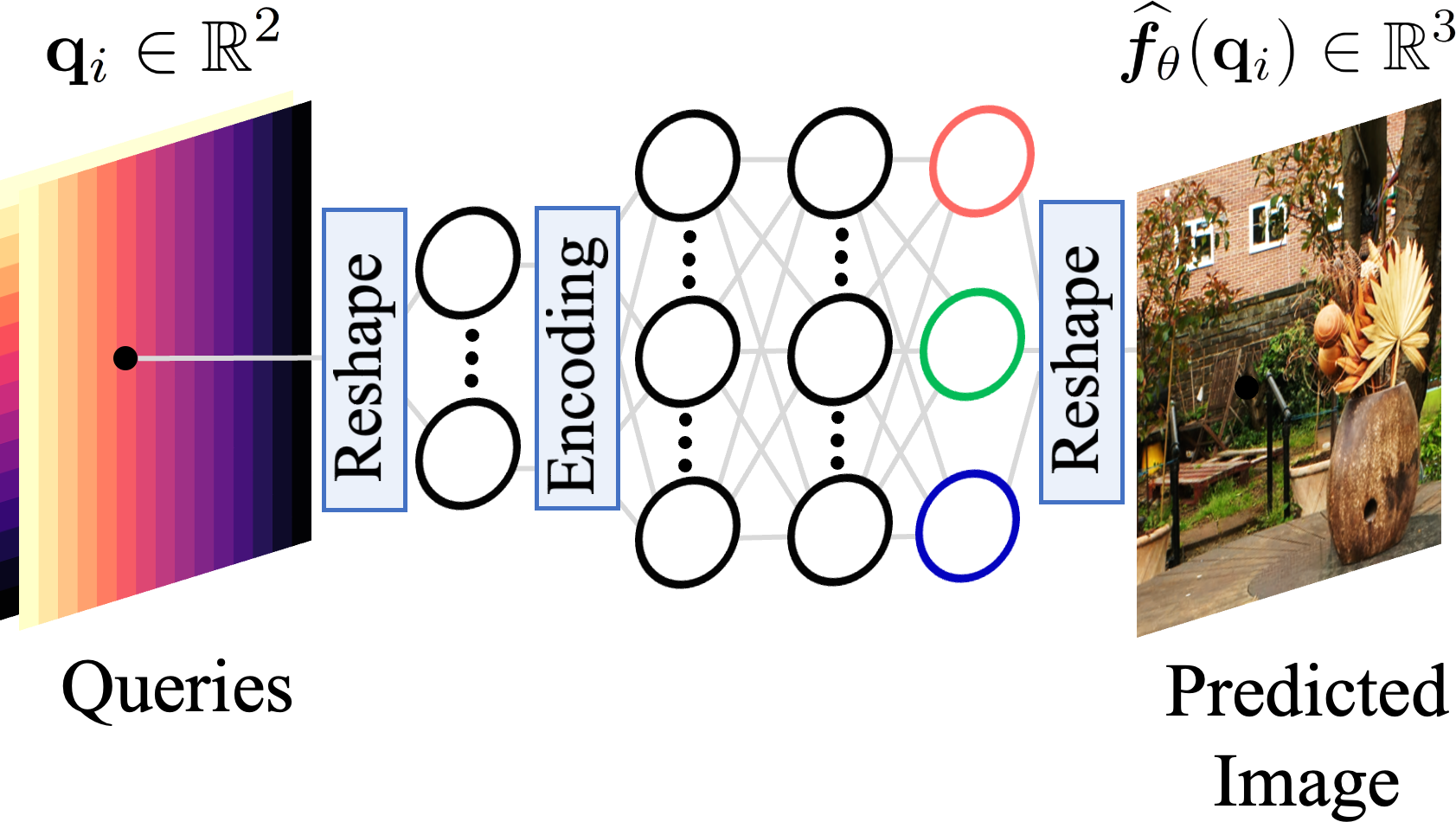}
            \label{fig:mlp}
            } 
        & \Block[borders={right,tikz={densely dashed, thick}}]{1-1}{}% 
        \subcaptionbox{\cnn}{
            \includegraphics[width=0.28\linewidth]{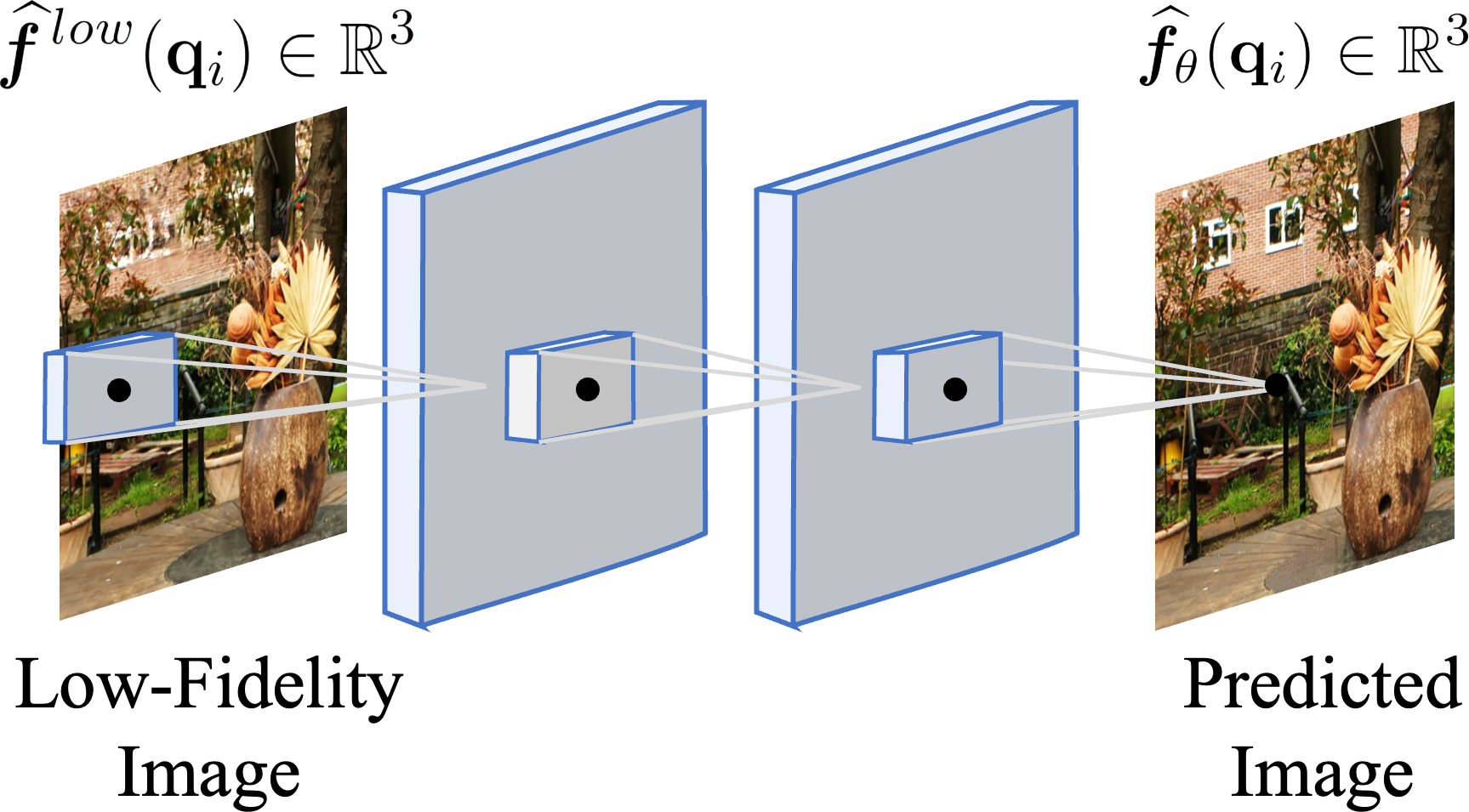}
            \label{fig:cnn}
            } 
        & \subcaptionbox{\qnn}{
            \vspace{-0.11cm}
            \includegraphics[width=0.33\linewidth]{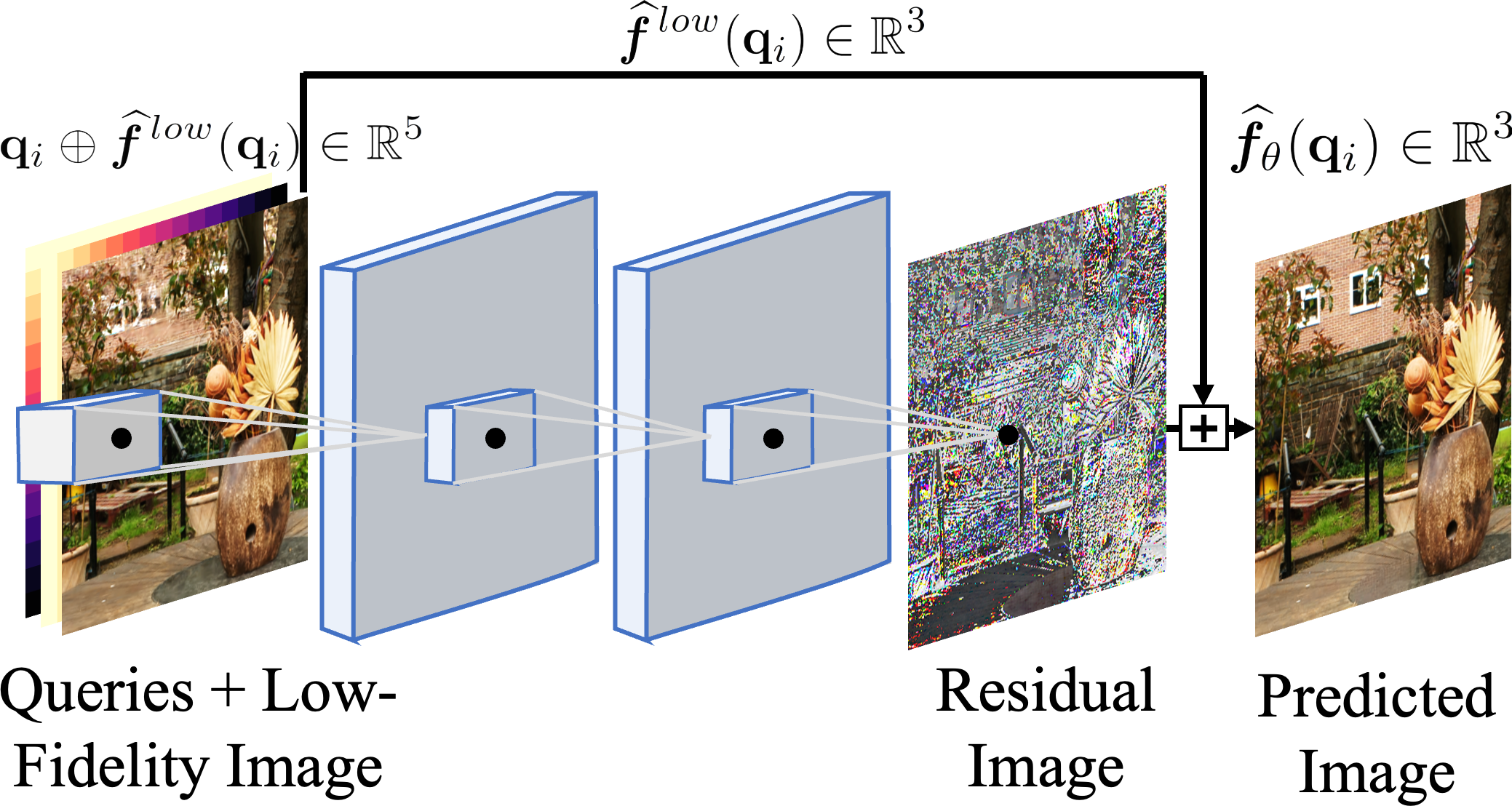}
            \label{fig:qnn}
            }
    \end{NiceTabular}
    \vspace{-0.2cm}
    \caption{\textbf{Overview} of \mlp, \cnn and \qnn. 
    \mlp takes queries $\queryIndex$ and uses linear layers.
    \cnn takes the \lowFid signal $\nFieldPdLP$ and uses convolutions.
    \qnn concatenates the queries $\queryIndex$ to the \lowFid signal $\nFieldPdLP$ and uses convolutions.
    [Key: $\concat$= Concatenation].
    }
    \label{fig:overview}
    \vspace{-0.55cm}
\end{figure*}

%============================================================================
%============================================================================
%============================================================================
\section{Introduction}

    Novel View Synthesis (\nvs) renders unseen viewpoints from posed input images, underpinning applications from dynamic reconstruction \cite{gao2024cat3d} to \threeD content creation \cite{tang2024dreamgaussian,yuan2024gavatar,zou2024triplane}. 
    While Neural Radiance Fields (\nerfs) \cite{mildenhall2020nerf,barron2023zip} pioneered photorealistic synthesis, \threeD Gaussian Splatting (\gs) \cite{kerbl2023gaussians} revolutionized the field with fast training and real-time rendering.

    Subsequent extensions have refined initialization \cite{wang2025stablegs}, location \cite{kheradmand20243d}, multi-scale representations \cite{yu2024mip}, and frequency weighting \cite{zhang2024fregs}. 
    Yet, even as recent variants surpass volumetric fields like \zipNerf \cite{barron2023zip}, pushing Gaussian reconstructions to arbitrary \highFid (\hf) levels faces fundamental bottlenecks (\cref{fig:teaser}). 
    A primary hurdle lies in the Spherical Harmonics (\sh) color parameterization (\cref{fig:oracle_gs}), which struggles to capture sharp details \cite{di2026spherical}.

    The most intuitive remedy, scaling up the primitive count, hits a strict mathematical barrier. 
    Modeling \gsOnly rendering as an additive Gaussian dictionary, we theoretically prove (\cref{th:num_gaussians}) and empirically confirm (\cref{fig:mcmc_scale_gaussians_trainval}) that \textbf{increasing training fidelity requires an exponentially increasing number of standard Gaussian primitives}, which inevitably triggers optimization issues and degrades validation \psnr. 
    To mitigate this bloat, recent works boost per-primitive expressivity via alternate primitives \cite{liu2025deformable,li2025half} or multi-component directional models \cite{di2026spherical}. 
    While all these advances improve per-primitive capacity and \nvs, these methods still optimize millions of discrete primitives and thus remain prone to primitive bloat and overfitting (\cref{tab:nvs_gauss_params}). 

    To overcome this limitation and capture thin structures without overfitting, we propose augmenting discrete primitives with a continuous neural field. 
    However, standard neural architectures face complementary drawbacks when learning \hf details: coordinate-based Multi-Layer Perceptrons (\mlps) (\cref{fig:overview}\textcolor{cvprblue}{a}) process pixel queries in isolation, lacking the inductive biases \cite{cohen2016group,lecun1998gradient} to capture local \neighborhood dependencies. 
    Conversely, Convolutional Neural Networks (\cnns) (\cref{fig:overview}\textcolor{cvprblue}{b}) leverage \neighborhood on \lowFid (\lf) signals \cite{karras2018progressive}, but remain largely agnostic to spatial coordinates \cite{liu2018intriguing}. 
    While \qConvFull \cite{liu2018intriguing} jointly convolves normalized coordinates and features for other tasks, the theoretical foundation and empirical efficacy of jointly convolving \twoD coordinate queries with \lf signals for increasing fidelity in \nvs remain unexplored.

    We first theoretically prove that progressively integrating \lf signals, \neighborhood, and coordinate queries as in \qConvFull \cite{liu2018intriguing} monotonically reduces approximation risk to zero (\cref{th:monotone_feat}). 
    Guided by this result, we use the existing \textcolor{theme_color}{\textbf{C}}oordinate-\textcolor{theme_color}{\textbf{C}}onvolution \textcolor{theme_color}{\textbf{N}}eural \textcolor{theme_color}{\textbf{N}}etwork (\coordCnn or \qnn) (\cref{fig:overview}\textcolor{cvprblue}{c}) to convolve coordinate queries and rendered \lf signals to predict residuals, consistently boosting fidelity across all primitive budgets (\cref{fig:mcmc_scale_gaussians_trainval}).
    Integrating \qnn into five diverse \gsOnly baselines establishes new \sotaFull (\sota) fidelity in \threeD \nvs with sharper details (\cref{fig:teaser,fig:qualitative_nvs}). 
    Beyond \nvs, \qnns also improve \oneD/\twoD regression and \twoD Super-Resolution (\sr) tasks.\footnote{Except for \sr, we focus on coordinate-based neural fields for single-signal fitting rather than amortized models in this paper.}

    In summary, our primary contributions are:
    \begin{itemize}
        \item We theoretically prove and empirically verify that discrete Gaussian primitives require exponential scaling to raise training \psnr (\cref{th:num_gaussians,fig:mcmc_scale_gaussians_trainval}).
        \item We prove that progressively integrating \lf signals, \neighborhood context, and coordinate queries monotonically reduces approximation risk to zero (\cref{th:monotone_feat}), motivating an existing \qnn \cite{liu2018intriguing} that jointly convolves spatial/ray queries with \lf renderings (\cref{sec:qconv}).
        \item Adding \qnn into five \gsOnly baselines (\eg, \nht) sets a new \sota fidelity in \threeD \nvs with real-time rendering and minimal training/storage overhead (\cref{tab:nvs_results,tab:nvs_edge_memory_inference}).
        \item We also demonstrate that \qnns achieve consistent gains in \oneD/\twoD regression and \twoD \sr tasks (\cref{sec:sr}).
    \end{itemize}

%============================================================================
%============================================================================
%============================================================================
\section{Literature Review}

    %============================================================================
    %============================================================================
    \noIndentHeading{Novel-View Synthesis (\nvs).}
        Neural Radiance Fields (\nerfs) model \threeD scenes continuously for \nvs, but ray-marching limits their training and rendering efficiency \cite{mildenhall2020nerf,barron2022mipnerf360}, prompting remedies via better encodings \cite{muller2022instant,tancik2020fourier} and multi-scale formulations \cite{barron2023zip}. 
        Conversely, \gs \cite{kerbl2023gaussians} renders a scene with rasterized \threeD Gaussians, offering faster training and real-time rendering.
        Subsequent variants improve fidelity through better initialization \cite{wang2025stablegs}, exploration \cite{kheradmand20243d}, alternative primitives \cite{arunan2025darb,liu2025deformable,li2025half}, anti-aliasing \cite{yu2024mip}, projection refinements \cite{huang2024error}, and frequency-aware losses \cite{zhang2024fregs}.
        A few works integrate ray-marching with \gsOnly \cite{moenne2024gaussray}, though this typically sacrifices the speed gains.
        While generalizable and feed-forward methods leverage \cnns across scenes \cite{aumentado2023exploring,charatan2024pixelsplat,schnepf2025bringing}, and video diffusion frameworks concatenate coordinate queries with latents for generative view synthesis \cite{gao2024cat3d,guizilini2025zero,hu2024learning,jin2025lvsm,li2025cameras}, they neither employ convolutional query integration nor explain its theoretical capacity. 
        In contrast, \qnn \cite{liu2018intriguing} directly convolves coordinate queries with rasterized \gsOnly images to improve fidelity on dense \nvs benchmarks.

    %============================================================================
    %============================================================================
    \noIndentHeading{\HighFid Learning.}
        Standard \mlps exhibit \lf \cite{rahaman2019spectral} and low-rank \cite{huh2022low} biases, inherently limiting their capacity on complex signals. 
        Existing remedies span input reparameterization such as sinusoidal \cite{vaswani2017attention}, Fourier \cite{rahimi2007random,tancik2020fourier}, \hashGrid encodings \cite{muller2022instant}, or wavelet decompositions \cite{yu2025high}.
        Others use periodic or specialized activations like erf \cite{yang2019fine}, \siren \cite{sitzmann2020siren}, sinc \cite{saratchandran2024activation}, \finer \cite{liu2024finer}, and \qiren \cite{zhao2024quantum}. 
        Complementary techniques explore frequency-domain predictions \cite{cai2020phase,lee2021conditional}, loss weighting \cite{sawada2025frebis,zhang2024fregs}, tuned initializations \cite{saratchandran2024activation,teney2024neural}, and network ensembles \cite{ainsworth2022galerkin,wang2024multi}. 
        While these strategies rely on point-wise \mlps evaluated over independent queries, this paper adapts \qnn \cite{liu2018intriguing} to convolve queries alongside the \lf signal, leveraging \neighborhood for enhanced signal recovery.

\begin{figure*}[!t]
    \centering
    \includegraphics[width=0.7\linewidth]{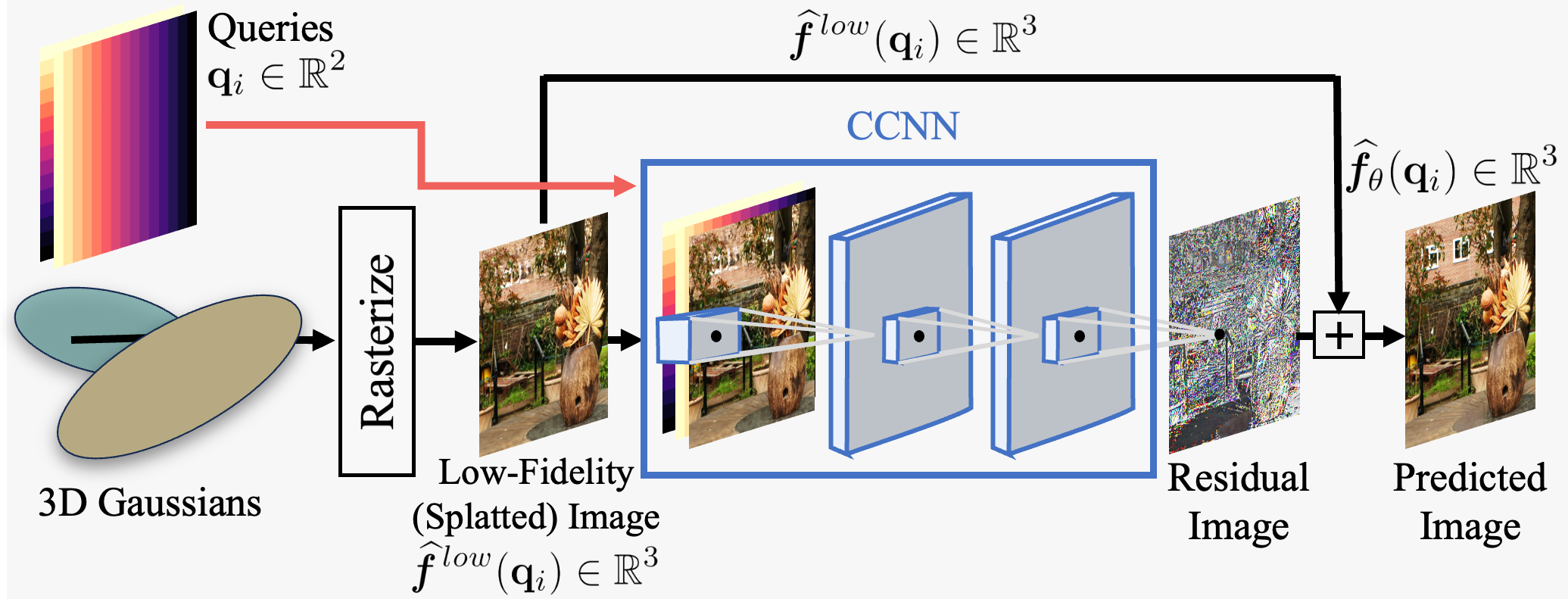}
    \vspace{-0.2cm}
    \caption{\textbf{\gsOnly{}+\qnn Architecture}. 
    \qnn \cite{liu2018intriguing} convolves the coordinate queries $\queryIndex$ and the \lf splatted image $\nFieldPdLP(\queryIndex)$ to output the residual image, which is then added to the \lf splatted image to get the final predicted novel view image $\nFieldPd_{\params}(\queryIndex)$ in training and inference.}
    \label{fig:gs_qnn}
    \vspace{-0.4cm}
\end{figure*}

    %============================================================================
    %============================================================================
    \noIndentHeading{Image Super-Resolution (\sr).}
        While \sr literature is shifting from \gan{}s \cite{wang2021realesrgan} to diffusion models \cite{saharia2022image,yue2023resshift}, continuous-scale \sr predominantly uses \cnns and sometimes uses implicit representations \cite{chen2021liif,lee2022local}. 
        Unlike approaches that employ point-wise \mlps over coordinates and latents, \qnn combines queries and encoded images via convolutions. 
        Although Lee and Jin \cite{lee2022local} also combine \cnns with queries, it targets Fourier-space coefficient prediction; in contrast, \qnn introduces a general convolutional query-integration mechanism.

    %============================================================================
    %============================================================================
    \noIndentHeading{Theoretical Results.}
        Prior works show that \cnns achieve superior sample complexity over \mlps in classification \cite{bietti2021approximation,li2021convolutional} and regression \cite{du2018many,misiakiewicz2022learning,wang2023theoretical} under conditions met by vision tasks like \nvs \cite{ulyanov2018deep}. 
        This advantage stems from weight sharing \cite{bruna2013invariant,lenc2015understanding,mei2021learning} and local \neighborhood \cite{pogodin2021towards,malach2021computational,lahoti2024role}, which outperform \mlps when targets depend on local structure \cite{malach2021computational}. 
        However, prior analyses overlook how local \neighborhood interacts with coordinate queries when refining \lf signals. 
        We bridge this gap by proving a monotone risk reduction across \lf-conditioned \mlp, \cnn, and \qnn in regression (\cref{th:monotone_feat}).

% \begin{align}
%     \queryIndex \in \realDomain^2 \\
%     \nFieldPd_{\params}(\queryIndex) \in \realDomain^3 \\
%     \nFieldPdLP(\queryIndex) \in \realDomain^3 \\
%     \queryIndex \concat \nFieldPdLP(\queryIndex) \in \realDomain^5
%     % \nFieldPd(\queryIndex) \in \realDomain^3 \\
% \end{align}

%============================================================================
%============================================================================
%============================================================================
\section{Background}\label{sec:background}

    We begin by establishing the basics in the next paragraphs.

    %============================================================================
    %============================================================================
    \noIndentHeading{3D Gaussian Splatting.}
        \gs \cite{kerbl2023gaussians} models a scene with learnable \threeD Gaussian primitives $\mathcal{G}{=}\{G_\gind\}_{\gind=1}^\numPrim$, parameterized by opacity $\opacity_\gind{\in}[0,1]$, world-space center $\vecCenter^{3D}_\gind{\in}\realDomain^3$, rotation quaternion $\vecQuaternion_\gind{\in}\realDomain^4$, log-scale $\vecScale_\gind{\in}\realDomain^3$, and spherical harmonics (\sh) color $\vecColor_\gind{\in}\realDomain^{48}$.
        During rasterization, \gs projects each \threeD Gaussian onto the image plane with \twoD mean $\vecCenter_\gind \in \realDomain^2$ and projected \twoD covariance $\matCov_\gind=\matJacob\matView\matRot_\gind\matScale_\gind\matScale_\gind^\top\matRot_\gind^\top\matView^\top\matJacob^\top \in \realDomain^{2\times2}$, where $\matView$ and $\matJacob$ are the view transformation and Jacobian matrices, respectively.
        The effective opacity $\opacityGT_\gind$ of Gaussian $G_\gind$ at a pixel $\coordinates \in \realDomain^2$ is
        $\opacityGT_\gind(\coordinates) = \opacity_\gind \exp\left(-\frac{1}{2}(\coordinates - \vecCenter_\gind)^\top \matCov_\gind^{-1} (\coordinates - \vecCenter_\gind)\right)$.
        For ray $r$ at pixel $\coordinates$, sorting $K_r$ intersecting Gaussians by ascending depth $d_\gind^r$ yields the rendering equations
        \begin{equation}
        \setlength\abovedisplayskip{3pt}
        \setlength\belowdisplayskip{3pt}
        \begin{aligned}
            \mathbf{\nField}(\coordinates){=}\sum_{k=1}^{K_r} \opacityGT_\gind(\coordinates)\transmittance_\gind(\coordinates) \vecColor_\gind,~ 
            \depth(\coordinates){=}\sum_{k=1}^{K_r} \opacityGT_\gind(\coordinates)\transmittance_\gind(\coordinates) d_\gind^r,
            \label{eq:gs_alphacomposite}
        \end{aligned}
        \end{equation}
        where $\transmittance_\gind(\coordinates){=}\prod_{j=1}^{\gind-1} \left(1{-}\opacityGT_j(\coordinates)\right)$ is the transmittance, and $\mathbf{\nField}(\coordinates)$ and $\depth(\coordinates)$ are the color and depth along ray $r$, respectively.
        See \cite{kerbl2023gaussians,kopanas2026what3dgs} for more details.

    %============================================================================
    %============================================================================
    \noIndentHeading{Neural Field.}\label{sec:neuralfields}
        A neural field, denoted as $\nFieldPd_{\params}: \querySet{\rightarrow}\outputSet$, with parameters $\params$, represents a continuous signal by mapping a bounded set of queries $\querySet \subset \realDomain^\inDim$ to outputs $\outputSet \subset \realDomain^\outDim$.
        Specifically, for an input query $\queryIndex\in\querySet$, the neural field outputs $\nFieldPd_{\params}(\queryIndex)\in\outputSet$.
        Illustrative examples of tasks, inputs and their corresponding outputs include:
        \begin{itemize} 
            \item \oneD Regression: \oneD coordinate queries $\queryIndex\in\realDomain^1$, outputting \oneD scalar values \cite{tancik2020fourier};
            \item \twoD Image Regression: \twoD coordinate queries $\queryIndex\in\realDomain^2$, % representing pixel locations, 
            outputting \threeD RGB pixel colors \cite{stanley2007compositional};
            % \item \threeD Shape Regression: \threeD coordinate queries $\queryIndex{\in}\realDomain^3$, predicting a \oneD signed distance function for a given shape \cite{park2019deepsdf};
            % \item Lightfield Network: \fourD queries $\queryIndex{\in}\realDomain^4$, encoding camera rays, generating \threeD RGB pixel colors \cite{sitzmann2021light};
            \item \nvs with \nerf: \fiveD queries\footnote{In practice, \nerf uses \sixD queries with \threeD view directions of norm $1$.} $\queryIndex\in\realDomain^5$, combining \threeD camera coordinates with \twoD view directions, outputting \fourD color and density information \cite{mildenhall2020nerf}.
        \end{itemize}

        Most neural fields employ \mlps over an encoding $\encoding$ such as Fourier encodings \cite{tancik2020fourier} or \hashGrids \cite{muller2022instant} (with vanilla encoding $\encoding(\queryIndex){=}\queryIndex$):
        \begin{align}
            \nFieldMLP(\queryIndex) &= \mlp (\encoding(\queryIndex)).
        \end{align}

        Processing $\numData$ queries of input dimensions (channels) $\inDim{=}\inChannels$ through $\hidChannels$ hidden channels transforms a tensor of shape $\numData{\times}\inChannels$ into $\numData{\times}\outChannels$, with  $\outDim{=}\outChannels$ output dimensions (channels). 
        One then reshapes the output tensor to the target signal and optimizes via $\lOne$ or squared loss.

\begin{table*}[!t]
    \centering
    \caption{
        \textbf{Experimental Settings} for tasks. 
        We add the predicted/GT \lf signal to the network's residual output $\nFieldPd$ to get the final prediction $\nFieldPdLP{+}\nFieldPd(\query\concat\nFieldPdLP)$. 
        We compute the loss between the final prediction $\nFieldPdLP{+}\nFieldPd(\query\concat \nFieldPdLP)$ and the GT signal $\nField$.
        [Key: Pred= Predicted].
    }
    \label{tab:exp_setting}
    \vspace{-0.2cm}
    \scalebox{0.8}{
    \setlength{\tabcolsep}{0.12cm}
    \begin{tabular}{l m cccc m c cc m c c} 
        \multirow{2}{*}{\textbf{Task}} & \multicolumn{4}{cm}{\textbf{Dimensionality} (Channels)} & \textbf{\lf Input } & \textbf{Network} & \textbf{Final} & \textbf{Network} & \multirow{2}{*}{\textbf{Loss}} \\
         & \textbf{Query $\query~(\inDim)$} & \textbf{LF $(\lowChannels)$} & \textbf{Convolution} & \textbf{Output $(\outDim)$} &  $(\nFieldLP/\nFieldPdLP)$ & \textbf{Output} & \textbf{Output} $(\nFieldPd_{\params})$ & $(\nFieldPd)$\\
        \myTopRule
        \threeD \nvs & 2 & 3 & \twoD & 3 & \multirow{2}{*}{Pred $\nFieldPdLP$} & \multirow{2}{*}{$\nFieldPd(\query\concat\nFieldPdLP)$} & \multirow{2}{*}{$\nFieldPdLP{+}\nFieldPd(\query\concat\nFieldPdLP)$} & \qnn, \mlp, \cnn & $\lOne$ variant\\
        \twoD Reg & 2 & 3 & \twoD & 3 & & & & \qnn, \mlp\!\phantom{,\cnn} &$\lOne$\!\phantom{variant~}\\
        \hline
        \oneD Reg & 1 & 1 & \oneD & 1 & \multirow{2}{*}{GT $\nFieldLP$} & \multirow{2}{*}{$\nFieldPd(\query\concat\nFieldLP)$} & \multirow{2}{*}{$\nFieldLP{+}\nFieldPd(\query\concat\nFieldLP)$} & \qnn, \mlp, \cnn & $\lOne$\!\phantom{variant~}\\
        \twoD \sr & 2 & 3 & \twoD & 3 & & & & \qnn, \cnn\!\phantom{,\cnn} & $\lOne$ {variant~}\\
    \end{tabular}
    }
    \vspace{-0.3cm}
\end{table*}

    %============================================================================
    %============================================================================
    \noIndentHeading{Example.} 
        Consider \twoD image regression with \twoD pixel coordinates as queries. 
        Here, the input channels $\inChannels = 2$ correspond to $(u, v)$ coordinates, and the output channels $\outChannels = 3$ correspond to RGB colors. 
        Each image contains $\numData{=}H{\times}W$ samples (height times width of the image).

    %============================================================================
    %============================================================================
    \noIndentHeading{\qConvFull Neural Network (\qnn)}\label{sec:qconv}{\cite{liu2018intriguing}\textbf{.}}
        Let $\nFieldPdLP$ and $\nFieldGTLP$ respectively denote the predicted and GT \lf components of the target signal $\nFieldGT$. 
        While coordinate \mlps process each query $\queryIndex$ in isolation, standard \cnns aggregate local \neighborhood context over a coarse \lf signal $\nFieldPdLP$ (\eg, a coarse input image in \sr) via $\numLayer$ alternating convolutional layers and non-linear activations, remaining agnostic to explicit spatial coordinates.
        Building on \coordConv \cite{liu2018intriguing}, a \qnn bridges both paradigms by concatenating the encoded queries $\encoding(\query)$ with the corresponding \lf signal $\nFieldPdLP(\query)$ along the channel dimension before passing them through the $\numLayer$-layer \cnn:
        \begin{align}
            \nFieldQNN(\queryIndex) &= \left[\cnn (\encoding(\query) \concat \nFieldPdLP(\query))\right]_i,
        \end{align}
        with $\concat$ denoting channel concatenation, and $\mathbf{\query}$ and $\mathbf{\nFieldPdLP}$ representing the full spatial grids. 
        The \cnn's receptive field naturally ensures that the output at index $\queryIndex$ depends only on its local \neighborhood.

        An advantage of the \qnn architecture \cite{liu2018intriguing} is that it needs a single evaluation for the entire signal, rather than once per coordinate.  
        However, a downside is that \qnn cannot be used in tasks without structured \neighborhoods, such as casting isolated rays in \nerf \cite{mildenhall2020nerf}.

        %============================================================================
        \noIndentHeading{Remark.}
            The \qnn \cite{liu2018intriguing} unifies existing architectures:
            \begin{itemize}
                \item Replacing the \cnn with an \mlp and excluding the \lf signal $\nFieldPdLP$ reduces \qnn to conventional neural fields.
                \item Omitting the queries $\queryIndex$ simplifies \qnn to a \cnn.
            \end{itemize}

%============================================================================
%============================================================================
%============================================================================
\section{Overcoming Primitive Bloat with \qnn}\label{sec:qnn}

    We first formalize the exponential scaling limits of explicit Gaussian primitives (\cref{sec:scaling_bottleneck}), then prove that jointly leveraging \lf signals, local \neighborhoods, and coordinate queries via \qnn monotonically reduces approximation risk to zero (\cref{sec:qnn_theory}), and finally detail our integration of \qnn with \gsOnly (\cref{sec:nvs_architecture}).

    %============================================================================
    %============================================================================
    \subsection{Scaling Bottleneck of Gaussian Primitives}\label{sec:scaling_bottleneck}

        Because \sh color parameters do not model sharp details \cite{di2026spherical}, \gsOnly decomposes a scene into millions of tiny primitives or relies on heavier per-primitive parameterizations (\eg, \dbs \cite{liu2025deformable}).
        Below, we state that increasing the best-approximation \textit{training} fidelity (\psnr) of a Gaussian-based representation requires an exponential increase in primitive count (proof in \cref{sec:shortcoming_gaussian}; empirical in \cref{fig:mcmc_scale_gaussians_trainval}).

        \addtocounter{theorem}{1}   % Bump theorem counter from 0 to 2
        \setcounter{corollary}{0}
        \begin{corollary}[of \cref{th:approx_wavelet} in \supp, \textbf{\# Gaussians}]
            \label{th:num_gaussians}
            \vspace{-0.2cm}
            Let a \threeD scene be rendered into $\numViews$ training views under an additive Gaussian dictionary model, where the normalized GT image for each training view attains the minimax complexity of fractional Sobolev class with smoothness $\sobolev > 0$ satisfying the conditions of \cref{lem:single_view_mse}. 
            To guarantee a mean training \psnr across views of at least $\psnrTarget$, the number of \threeD Gaussians $\numPrim$ must satisfy
            \vspace{-0.2cm}
            \begin{align}
                \numPrim &\ge \exp\left[ \lambert_0\left( \left[ \constant^2 10^{\frac{\psnrTarget}{10}} \right]^{\frac{1}{\sobolev}} \right) \right],
                \label{eq:multi_view_gaussian_avg_mse_count}
            \end{align}
            where $\lambert_0$ is the principal branch of the Lambert $\lambert$ function and $\constant>0$ is a constant independent of $\numPrim$.
        \end{corollary}

        \cref{tab:theory_vs_gs} summarizes the assumptions to derive the above lower bound, which is a capacity ceiling for splatting.
        Fitting \cref{eq:multi_view_gaussian_avg_mse_count} with the \mcmc \cite{kheradmand20243d} on \mipNerf \cite{barron2022mipnerf360} yields $\sobolev{=}0.24$ ($R^2{=}0.998$), matching the classical bound $\sobolev < 0.5$ for sharp-edged images \cite[Sec. I]{devore1992image}.
        In practice, attempting to satisfy \cref{eq:multi_view_gaussian_avg_mse_count} via primitive bloat triggers severe overfitting to training views first followed by optimization issues, which explains the poor \textit{validation} \psnr (\cref{fig:mcmc_scale_gaussians_trainval}).
        Rather than inflating discrete primitives, we treat the coarse rendering as a base \lf signal and synthesize the missing residual via a continuous \qnn \cite{liu2018intriguing} (\cref{sec:qconv}).

\begin{table*}[!t]
    \caption{
        \textbf{\nvs Results.} 
        \textbf{Adding \qnn outperforms all five baselines on all six datasets.}
        Numbers are from respective papers.
        Detailed comparison in \cref{tab:nvs_results_detail} and performance in \cref{tab:nvs_edge_memory_inference}.
        [Key: \firstBText{Best}~, \secondBText{Second-best}~, \thirdBText{Third-best}~, \retrained= Retrained, \taken= Reported in \gs \cite{kerbl2023gaussians}]. 
    }
    \label{tab:nvs_results}
    \vspace{-0.3cm}
    \centering
    \scalebox{0.73}{
    \setlength{\tabcolsep}{0.1cm}
    \begin{tabular}{c @{\hspace{0.3cm}} lc c @{\hspace{0.3cm}} c c @{\hspace{0.3cm}} c @{\hspace{0.3cm}} c @{\hspace{0.3cm}} c} 
        \myTopRule
        & \multirow{2}{*}{Method} & \multirow{2}{*}{Venue} 
        & \textbf{\mipNerf \cite{barron2022mipnerf360}} 
        & \textbf{\tAndT \cite{knapitsch2017tanks}} 
        & \textbf{\deepBlend \cite{hedman2018deep}} 
        & \textbf{\ommo \cite{lu2023large}} 
        & \textbf{\shelly \cite{wang2023adaptiveshells}} 
        & \textbf{\synNerf \cite{mildenhall2020nerf}} \\
        & & & \footnotesize{\psnr\!$\uparrowRHDSmallTwo$\sep\ssim\!$\uparrowRHDSmallTwo$\sep\lpips\!$\downarrowRHDSmallTwo$}
        & \footnotesize{\psnr\!$\uparrowRHDSmallTwo$\sep\ssim\!$\uparrowRHDSmallTwo$\sep\lpips\!$\downarrowRHDSmallTwo$}
        & \footnotesize{\psnr\!$\uparrowRHDSmallTwo$\sep\ssim\!$\uparrowRHDSmallTwo$\sep\lpips\!$\downarrowRHDSmallTwo$}
        & \footnotesize{\psnr\!$\uparrowRHDSmallTwo$\sep\ssim\!$\uparrowRHDSmallTwo$\sep\lpips\!$\downarrowRHDSmallTwo$}
        & \footnotesize{\psnr\!$\uparrowRHDSmallTwo$\sep\ssim\!$\uparrowRHDSmallTwo$\sep\lpips\!$\downarrowRHDSmallTwo$}
        & \footnotesize{\psnr\!$\uparrowRHDSmallTwo$\sep\ssim\!$\uparrowRHDSmallTwo$\sep\lpips\!$\downarrowRHDSmallTwo$} \\
        \myTopRule 
        \multirow{5}{*}{\rotatebox{90}{\textbf{Ray}}}
        & \ever~\cite{mai2025ever} & \venue{\iccv25}  & $27.51$ \sep \thirdB{0.83} \sep $0.23$ & \mathDash & \mathDash & \mathDash & \mathDash & \mathDash \\
        & Mip-\nerf\!\taken \cite{barron2021mip} & \venue{\cvpr22} & $27.69$ \sep $0.79$ \sep $0.24$ & $22.22$ \sep $0.76$ \sep $0.26$ & $29.40$ \sep \thirdB{0.90} \sep \thirdB{0.25}  & \mathDash & \mathDash & \mathDash \\
        & Adaptive \cite{wang2023adaptiveshells} & \venue{\sigg23} & \mathDash & \mathDash & \mathDash & \mathDash & $36.02$ \sep $0.95$ \sep $0.08$ & $32.51$ \sep $0.96$ \sep $0.05$ \\
        & QFields \cite{sharma2024volumetric} & \venue{\eccv24} & \mathDash & \mathDash & \mathDash & \mathDash & $37.29$ \sep $0.95$ \sep \thirdB{0.07} & $31.00$ \sep $0.95$ \sep $0.07$ \\
        & \zipNerf \!\cite{barron2023zip}           & \venue{\iccv23} & $28.54$ \sep \thirdB{0.83} \sep ${0.19}$ & \mathDash & \mathDash & \mathDash & \mathDash & $33.69$ \sep \thirdB{0.97} \sep ~~\mathDash~~ \\
        \myTopRule
        \multirow{17}{*}{\rotatebox{90}{\textbf{Raster}}}
        & Tex-GS \cite{chao2025textured} & \venue{\cvpr25} & $27.35$ \sep \thirdB{0.83} \sep ${0.19}$ & $24.26$ \sep $0.85$ \sep $0.17$ & $28.33$ \sep ${0.89}$ \sep $0.36$ & \mathDash & \mathDash & $33.24$ \sep \thirdB{0.97} \sep $0.04$\\
        & 3D-HGS \cite{li2025half} & \venue{\cvpr25} & \mathDash           & ${25.08}$ \sep $0.84$ \sep \thirdB{0.14} & $29.80$ \sep \thirdB{0.90} \sep \thirdB{0.25} & \mathDash & \mathDash & \mathDash \\
        % & Project \cite{huang2024error} & \venue{\eccv24} & $27.48$ \sep $0.82$ \sep $0.21$ & $23.43$ \sep ~~\mathDash~~ \sep ~~\mathDash~~ & $29.51$ \sep  ~~\mathDash~~ \sep ~~\mathDash~~ & \mathDash & \mathDash & \mathDash \\
        & \vdgs~\cite{malarz2025gaussian} & \venue{CVIU25} & $27.64$ \sep $0.81$ \sep $0.22$ & $24.02$ \sep $0.85$ \sep $0.18$ & $29.54$ \sep \secondB{0.91} \sep \secondB{0.24} & \mathDash & \mathDash & \thirdB{35.97} \sep \firstB{0.99} \sep \firstB{0.01} \\ 
        & Revise \cite{rota2024revising} & \venue{\eccv24} & $27.70$ \sep $0.82$ \sep $0.22$ & $24.10$ \sep \thirdB{0.86} \sep $0.18$ & $29.64$ \sep \secondB{0.91} \sep $0.30$ & \mathDash & \mathDash & \mathDash \\
        & HyRF \cite{wang2025hyrf} & \venue{\nips25} & $27.78$ \sep $0.82$ \sep $0.21$ & $24.02$ \sep $0.84$ \sep $0.18$ & \thirdB{30.37} \sep \secondB{0.91} \sep \secondB{0.24} & \mathDash & \mathDash & \mathDash \\
        & \mipSplat~\cite{yu2024mip} & \venue{\cvpr24} & $27.79$ \sep \thirdB{0.83} \sep $0.20$ & \mathDash & \mathDash & \mathDash & \mathDash & \mathDash \\
        & FreGS \cite{zhang2024fregs} & \venue{\cvpr24} & $27.85$ \sep \thirdB{0.83} \sep $0.21$ & $23.96$ \sep $0.85$ \sep $0.18$ & ${29.93}$ \sep \thirdB{0.90} \sep \secondB{0.24} & \mathDash & \mathDash & \mathDash \\
        \hhline{|~|--------|}
        % & \gs \released \cite{kerbl2023gaussians}      & \venue{\sigg23} & $27.20$ \sep $0.82$ \sep $0.21$  & $23.14$ \sep $0.84$ \sep $0.18$ & $29.41$ \sep \secondB{0.90} \sep \secondB{0.24} & \mathDash & \mathDash & $33.31$ \sep ~~\mathDash~~ \sep ~~\mathDash~~ \\
        & \gs \retrained \cite{kerbl2023gaussians} & \venue{\sigg23} & $27.67$ \sep $0.82$ \sep $0.20$  & $23.52$ \sep $0.84$ \sep $0.18$ & $29.50$ \sep \thirdB{0.90} \sep \secondB{0.24} & $29.03$ \sep ${0.90}$ \sep ${0.16}$ & ${37.56}$ \sep \thirdB{0.96} \sep ${0.08}$ & $34.83$ \sep \secondB{0.98} \sep \thirdB{0.03} \\
        & \gs + \qnn & \mathDash & $27.96$ \sep \thirdB{0.83} \sep $0.20$ & $24.11$ \sep $0.85$ \sep $0.17$  & $29.67$ \sep \secondB{0.91} \sep \thirdB{0.25} & $29.37$ \sep \thirdB{0.91} \sep ${0.15}$ & ${37.99}$ \sep \thirdB{0.96} \sep \thirdB{0.07} & $35.36$ \sep \secondB{0.98} \sep \thirdB{0.03} \\
        \hhline{|~|--------|}
        % & \mcmc \released \cite{kheradmand20243d}       & \venue{\nips24} & \mathDash            & $24.29$ \sep \thirdB{0.86} \sep $0.19$ & $29.67$ \sep $0.89$ \sep $0.32$ & \thirdB{29.52} \sep \secondB{0.91} \sep $0.20$ & \mathDash & $33.80$ \sep \thirdB{0.97} \sep $0.04$ \\
        & \mcmc \retrained \cite{kheradmand20243d}  & \venue{\nips24} & $28.26$ \sep \secondB{0.84} \sep \secondB{0.17}  & ${24.53}$ \sep \secondB{0.87} \sep \secondB{0.13} & $29.32$ \sep \secondB{0.91} \sep \secondB{0.24} & ${30.10}$ \sep \secondB{0.92} \sep \secondB{0.12} & $35.14$ \sep ${0.94}$ \sep ${0.09}$ & \secondB{36.14} \sep \secondB{0.98} \sep \secondB{0.02} \\
        & \mcmc + \qnn                                   & \mathDash & $28.58$ \sep \secondB{0.84} \sep \firstB{0.16}  & ${24.87}$ \sep \firstB{0.88} \sep \firstB{0.12} & $29.76$ \sep \secondB{0.91} \sep \firstB{0.23} & \thirdB{30.34} \sep \secondB{0.92} \sep \secondB{0.12} & $35.57$ \sep $0.95$ \sep $0.09$ &  \firstB{36.58} \sep \secondB{0.98} \sep \secondB{0.02} \\
        \hhline{|~|--------|}
        % & \dbs \released \cite{liu2025deformable} & \venue{\sigg25}  & $28.60$ \sep \secondB{0.84} \sep $0.18$ & $24.79$ \sep \secondB{0.87} \sep $0.15$ & $30.10$ \sep $0.91$ \sep$0.24$ & \mathDash  & \mathDash & $34.64$ \sep \thirdB{0.97} \sep $0.03$\\
        & \dbs \retrained \cite{liu2025deformable} & \venue{\sigg25}  & $28.38$ \sep \secondB{0.84} \sep \thirdB{0.18} & $24.70$ \sep \secondB{0.87} \sep $0.15$ & $29.75$ \sep \secondB{0.91} \sep $0.24$ & $28.67$ \sep $0.89$ \sep$0.15$ & \secondB{42.16} \sep \firstB{0.98} \sep \firstB{0.04} & $34.57$ \sep \thirdB{0.97} \sep \thirdB{0.03}\\
        & \dbs+ \qnn & \mathDash & $28.72$ \sep \secondB{0.84} \sep \secondB{0.17} & ${25.16}$ \sep \secondB{0.87} \sep \thirdB{0.14} & ${29.94}$ \sep \secondB{0.91} \sep \firstB{0.23} & ${28.97}$ \sep ${0.90}$ \sep ${0.14}$ & \firstB{42.98} \sep \firstB{0.98} \sep \firstB{0.04} & $34.76$ \sep \thirdB{0.97} \sep \thirdB{0.03}\\
        \hhline{|~|--------|}
        % & \sv \released \cite{di2026spherical} & \venue{\cvpr26} & $28.71$ \sep \secondB{0.84} \sep $0.23$ & $25.00$ \sep \secondB{0.87} \sep $0.17$ & $30.63$ \sep $0.92$ \sep $0.30$ & \mathDash & \mathDash &  $34.58$ \sep \thirdB{0.97} \sep $0.03$\\ 
        & \sv \retrained \cite{di2026spherical} & \venue{\cvpr26} & $28.71$ \sep \secondB{0.84} \sep $0.19$ & $24.98$ \sep \secondB{0.87} \sep \thirdB{0.14} & $30.00$ \sep \secondB{0.91} \sep \secondB{0.24} & $29.94$ \sep \thirdB{0.91} \sep \thirdB{0.13} & $41.17$ \sep \secondB{0.97} \sep \firstB{0.04} & $34.82$ \sep \thirdB{0.97} \sep \thirdB{0.03}\\
        & \sv + \qnn & \mathDash & \thirdB{28.83} \sep \secondB{0.84} \sep $0.19$ & \thirdB{25.33} \sep \firstB{0.88} \sep \secondB{0.13} & \firstB{30.71} \sep \firstB{0.92} \sep \firstB{0.23} & $30.19$ \sep \thirdB{0.91} \sep \thirdB{0.13} & \thirdB{41.58} \sep \secondB{0.97} \sep \firstB{0.04} & $34.90$ \sep \thirdB{0.97} \sep \thirdB{0.03}\\
        \hhline{|~|--------|}
        % & \nht \released \cite{condor2026nht} & \venue{\eccv26} & $28.74$ \sep \secondB{0.84} \sep $0.22$ & $25.68$ \sep \firstB{0.88} \sep $0.14$ & $30.94$ \sep $0.92$ \sep $0.30$  & \mathDash & \mathDash \\
        & \nht \retrained \cite{condor2026nht} & \venue{\eccv26} & \secondB{29.04} \sep \firstB{0.85} \sep \secondB{0.17} & \secondB{25.72} \sep \firstB{0.88} \sep \firstB{0.12} & ${30.17}$ \sep \secondB{0.91} \sep \thirdB{0.25} & \secondB{31.50} \sep \firstB{0.93} \sep \secondB{0.12} 
        & $37.64$ \sep \thirdB{0.96} \sep \secondB{0.06} & $33.19$ \sep \thirdB{0.97} \sep $0.04$ \\
        & \nht + \qnn & \mathDash                                & \firstB{29.19} \sep \firstB{0.85} \sep \secondB{0.17} & \firstB{25.79} \sep \firstB{0.88} \sep \firstB{0.12} & \secondB{30.64} \sep \firstB{0.92} \sep \secondB{0.24} & \firstB{31.74} \sep \firstB{0.93} \sep \firstB{0.11} & $37.67$ \sep \thirdB{0.96} \sep \secondB{0.06} & $33.33$ \sep \thirdB{0.97} \sep $0.04$ \\
        \myTopRule
    \end{tabular}
    }
    \vspace{-0.5cm}
\end{table*}

    %============================================================================
    %============================================================================
    \subsection{Monotone Risk Reduction with \qnn}\label{sec:qnn_theory}
        We next theoretically justify why \qnn combines \lf signals, local \neighborhoods, and coordinate queries.
        Analyzing real-valued function regression in the asymptotic regime (unbounded data and compute), we prove that progressively augmenting a \lf signal with local \neighborhood context and coordinate queries never increases the minimum achievable risk, and that incorporating queries guarantees exact target recovery. 
        See \cref{sec:supp_qnn_theory} for the proof.

        \addtocounter{theorem}{0}
        \begin{theorem}[\textbf{Monotone Risk Reduction}]\label{th:monotone_feat}
            Let $\queryInput$ be an input space (either $\realDomain^{\inDim}$ or a discrete integer lattice) equipped with a probability distribution $P_\queryInput$.
            Let $\nFieldGTScalar{:}~\queryInput\to[0, 1]$ be a deterministic, measurable target function. 
            A feature $\parametrize{:}~\queryInput{\to}\mathcal{Z}$ induces a hypothesis class $\mathcal{H}(\parametrize)$ of all deterministic, measurable functions $h{:}~\mathcal{Z}{\to}[0,1]$, with optimal expected squared error
            \begin{align}
                \vspace{-0.2cm}
                \riskOptimal(\parametrize, \nFieldGTScalar) = \min_{h \in \mathcal{H}(\parametrize)} \mathop{\expect}\left[ (\nFieldGTScalar(\coordinates) - h(\parametrize(\coordinates)))^2 \right].
            \end{align}
            Given a deterministic, measurable \lf component $\nFieldGTLPScalar{:}~\queryInput\to[0, 1]$ and a fixed $\receptive > 0$, consider the following nested feature hierarchy:
            \begin{itemize}
                \item \textbf{Feature 1 \mlp~(\lf):} $\parametrize_1(\coordinates) = (\nFieldGTLPScalar(\coordinates))$
                \item \textbf{Feature 2 (\lf + \Neighborhood) or \cnn~(\lf):}\\$\parametrize_2(\coordinates) = (\parametrize_1(\coordinates - \receptive), \parametrize_1(\coordinates), \parametrize_1(\coordinates + \receptive))$
                \item \textbf{Feature 3 (\lf + \Neighborhood + Coordinate Queries) or \qnn:} $\parametrize_3(\coordinates) = (\parametrize_2(\coordinates), \coordinates-\receptive, \coordinates, \coordinates+\receptive)$.
            \end{itemize}
            Then, the optimal expected errors for these features are monotonically non-increasing and bounded by zero:
            \begin{align}
                \vspace{-0.2cm}
                \riskOptimal(\parametrize_1, \nFieldGTScalar) \geq \riskOptimal(\parametrize_2, \nFieldGTScalar) \geq \riskOptimal(\parametrize_3, \nFieldGTScalar) = 0.
            \end{align}
        \end{theorem}
        \addtocounter{theorem}{-1}   % Roll back theorem counter to 1
        % \noindent Please see \cref{sec:supp_qnn_theory} for the proof.%, which utilizes optimal squared error predictors \cite{bishop2006pattern} and the non-increasing error property with added random variables \cite{xu2022minimum}.

    %============================================================================
    %============================================================================
    \subsection{Combining Gaussian Splatting with \qnn}\label{sec:nvs_architecture}
        Guided by \cref{th:monotone_feat}, we pair \qnn (\cref{sec:qconv}) with a base \lf signal $\nFieldPdLP$.
        Because $\nFieldPdLP$ already captures coarse structure (\eg, the rasterized \gsOnly image in \nvs), having the network re-synthesize $\nFieldPdLP$ from scratch is redundant.
        We, therefore, formulate detail recovery as a residual refinement task where $\nFieldQNN(\queryIndex)$ predicts the missing residual:
        $
            \nFieldPd_{\params}(\queryIndex) = \nFieldPdLP(\queryIndex) + \nFieldQNN(\queryIndex).
        $
        When integrating this residual formulation with \gsOnly (\cref{fig:gs_qnn}) to enhance \nvs,
        \gsOnly first renders the initial base \lf image $\nFieldPdLP \in \realDomain^{\lowChannels\times H\times W}$ ($\lowChannels{=}3$).\footnote{We drop the subscript $i$ in query $\queryIndex$ and $\queryIndex$ in $\nFieldPdLP(\queryIndex)$ for brevity.}
        Next, \qnn pairs $\nFieldPdLP$ with a coordinate query map $\query \in \realDomain^{\inChannels\times H\times W}$ to form a $(\lowChannels{+}\inChannels){\times}H{\times}W$ input tensor, jointly convolving them across its convolution layers in a single feed-forward pass to predict the residual image $\nFieldQNN = \cnn(\query \concat\nFieldPdLP) \in \realDomain^{\outChannels\times H\times W}$ ($\outChannels{=}3$), yielding the final novel view:
        \begin{align}
            \nFieldGSQNN &= \nFieldPdLP + \cnn(\query \concat \nFieldPdLP).
        \end{align}
        While \qnn supports normalized \twoD pixel coordinates and general \threeD or ray-space queries, we ablate query choices $\query$ in \cref{sec:ablation} and find that normalized \twoD pixel coordinates ($\inChannels{=}2$) offer the best validation metric.
        We supervise $\nFieldGSQNN$ end-to-end against the ground-truth image $\nFieldGT$ using the exact same photometric loss as the corresponding baseline (details in \cref{sec:supp_impl_nvs}), and analyze how \qnn preserves multi-view \threeD consistency in \cref{sec:supp_discussions}.
        Unlike ray-marching \mlps (\nerf) that require hundreds of queries per pixel, \gsOnly+\qnn evaluates a single query per pixel in one feed-forward pass, preserving real-time rendering speeds with minimal training overhead (\cref{tab:nvs_edge_memory_inference}).

\begin{figure*}[!t]
    \centering
    \sbox{\nvsfigbox}{\includegraphics[width=0.98\linewidth]{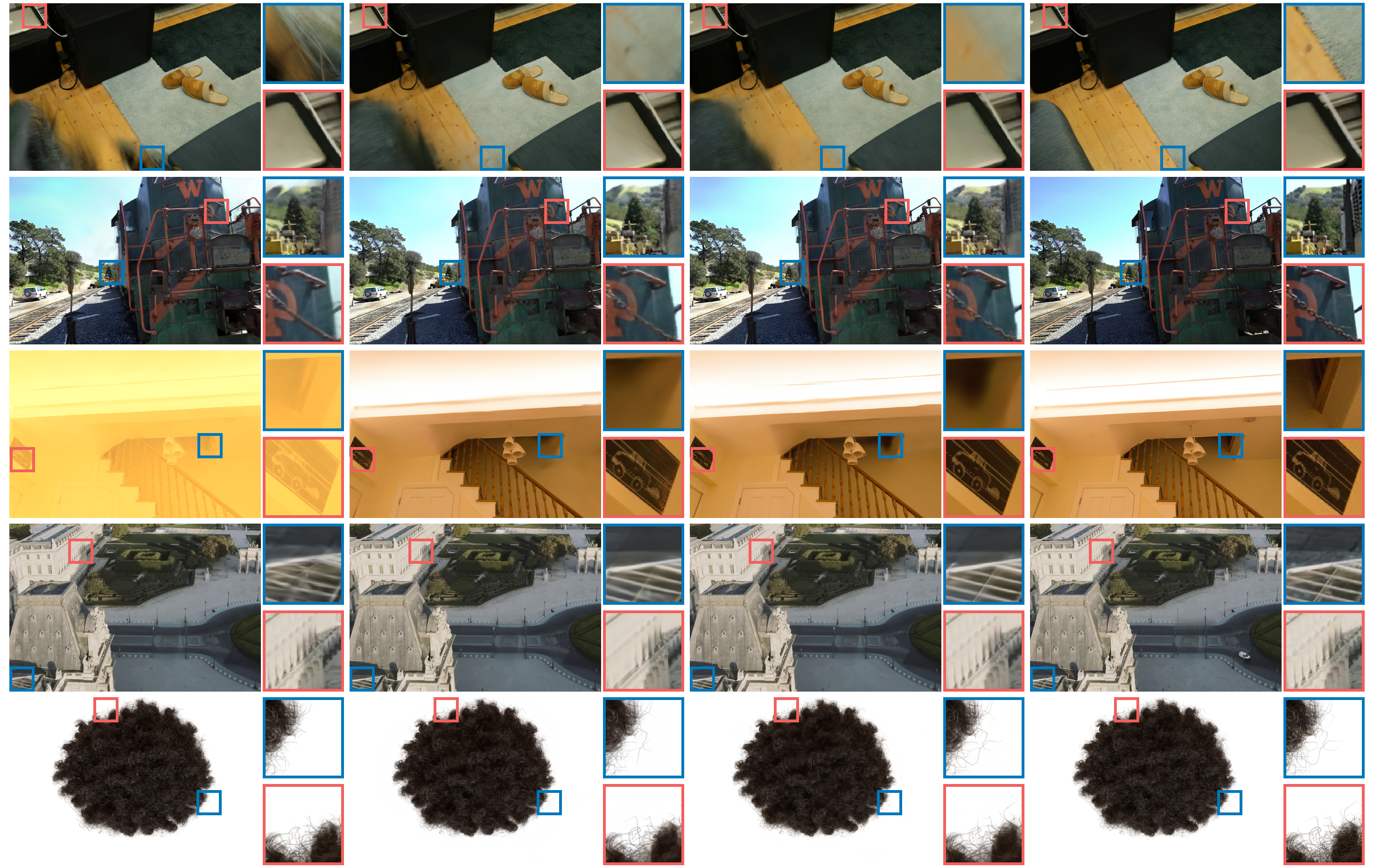}}%
    \setlength{\imgh}{\ht\nvsfigbox}
    \addtolength{\imgh}{\dp\nvsfigbox}

    \hspace*{0.02\linewidth}%
    \makebox[0.243\linewidth]{\small Spherical Voronoi~\cite{di2026spherical}}%
    \makebox[0.243\linewidth]{\small \nht~\cite{condor2026nht}}%
    \makebox[0.243\linewidth]{\small \nht+~\qnn}%
    \makebox[0.243\linewidth]{\small Ground Truth}\par
    \vspace{0.02cm}
    \begin{minipage}[b][\imgh][c]{0.02\linewidth}
        \fontsize{7.5}{9}\selectfont\centering
        \parbox[c][0.2\imgh][c]{\linewidth}{\centering\rotatebox{90}{Room~\cite{barron2022mipnerf360}}}\par
        \parbox[c][0.2\imgh][c]{\linewidth}{\centering\rotatebox{90}{Train~\cite{knapitsch2017tanks}}}\par
        \parbox[c][0.2\imgh][c]{\linewidth}{\centering\rotatebox{90}{Playroom~\cite{hedman2018deep}}}\par
        \parbox[c][0.2\imgh][c]{\linewidth}{\centering\rotatebox{90}{13~\cite{lu2023large}}}\par
        \parbox[c][0.2\imgh][c]{\linewidth}{\centering\rotatebox{90}{Khady~\cite{wang2023adaptiveshells}}}
    \end{minipage}%
    \hspace*{-2pt} 
    \usebox{\nvsfigbox}
    \caption{\textbf{\nvs Qualitative Results} between previous works and adding \qnn with \nht \cite{condor2026nht} baseline on five datasets. 
        \textbf{Adding \qnn results in higher fidelity synthesis visually}. 
        Inset figures highlight differences.
        More qualitative results are in \cref{fig:qualitative_nvs_1,fig:qualitative_nvs_2} in \supp.
    }
    \label{fig:qualitative_nvs}
    \vspace{-0.5cm}
\end{figure*}

%============================================================================
%============================================================================
%============================================================================
\section{Experiments}

    We evaluate \qnns across four tasks summarized in \cref{tab:exp_setting}: \threeD \nvs, \oneD and \twoD regression, and \twoD \sr.
    The \nvs, \oneD, and \twoD regression experiments compare \qnns against \mlp-based neural fields and \cnns, while \twoD \sr compares \qnns against \cnns.

    %============================================================================
    %============================================================================
    \subsection{3D \nvs}\label{sec:nvs}

        %============================================================================
        \noIndentHeading{Datasets.}
            Our \nvs experiments use six datasets: \mipNerf \cite{barron2022mipnerf360}, \tAndT \cite{knapitsch2017tanks}, \deepBlend \cite{hedman2018deep}, \ommo \cite{lu2023large}, \shelly \cite{wang2023adaptiveshells}, and \synNerf \cite{mildenhall2020nerf}.

        %============================================================================
        \noIndentHeading{Evaluation Metrics and Baselines.}
            We report \psnr, \ssim \cite{wang2004image} and VGG-16 \lpips \cite{zhang2018unreasonable} following \cite{kerbl2023gaussians}.
            We integrate \qnn (\cref{fig:gs_qnn}) into five \gsOnly{} baselines: \gs \cite{kerbl2023gaussians}, \mcmc \cite{kheradmand20243d}, \dbs \cite{liu2025deformable}, \sv \cite{di2026spherical} and \nht \cite{condor2026nht}.
            We apply \qnn with lr${=}1{e-}4$ after $75\%$ of iterations without changing total iterations.
            See \cref{sec:supp_impl_nvs} for more details.

\begin{table}[!t]
    \centering
    \caption{\textbf{\mipNerf Performance}. 
    \textbf{Adding \qnn consistently boosts} \psnr and \psnrEdge across all baselines \textbf{with marginal training overhead and virtually no storage overhead, maintaining real-time rendering speeds}. 
    GPU-hr and Stor denote training time and disk storage per scene respectively.
    [Key: \firstBText{Best}~, \secondBText{Second-best}~, \thirdBText{Third-best}~, \taken= From Tab. 1 of \ever \cite{mai2025ever}].
    }
    \label{tab:nvs_edge_memory_inference}
    \vspace{-0.1cm}
    \scalebox{0.7}{
    \setlength\tabcolsep{0.1cm}
    \begin{tabular}{l m c  c c c c}
        Method & \psnr(\uparrowRHDSmall) & \psnrEdge(\uparrowRHDSmall) & GPU-hr (\downarrowRHDSmall) & Stor MB (\downarrowRHDSmall) & ~~\fps (\uparrowRHDSmall)~~ \\
        \myTopRule
        \zipNerf\cite{barron2023zip} & $28.54$         & $26.34$         & $32.00$\taken  & ~$903$\taken & ~~~~~~~$0.5$\taken \\
        \hline
        \gs\cite{kerbl2023gaussians} & $27.67$         & $25.76$         & \firstF{0.24}  & \secondF{736} & \firstF{504.4} \\
        \gs + \qnn                   & $27.96$         & $25.89$         & \secondF{0.26} & \secondF{736} & \thirdF{190.9}\\
        \hline 
        \mcmc\cite{kheradmand20243d} & $28.26$         & $26.28$         & \thirdF{0.38}  & \thirdF{900} & \secondF{428.3} \\
        \mcmc + \qnn                 & \thirdF{28.58}  & \thirdF{26.45}  & $0.40$         & $901$ & $177.1$ \\
        \hline 
        \nht\cite{condor2026nht}     & \secondF{29.04} & \secondF{26.86} & $0.63$         & \firstF{490} & $176.5$\\
        \nht + \qnn                  & \firstF{29.19}  & \firstF{26.96}  & $0.66$         & \firstF{490} & ~~$87.1$\\
    \end{tabular}
    }
    \vspace{-0.5cm}
\end{table}

        %============================================================================
        \noIndentHeading{Quantitative Results.}
            \cref{tab:nvs_results} shows that integrating \qnn consistently outperforms all five baselines across all six datasets. 
            Crucially, \nht + \qnn establishes a new \sota ($29.19$\,\db) on \mipNerf \cite{barron2022mipnerf360}, while even \mcmc + \qnn outperforms \zipNerf \cite{barron2023zip} ($28.58$\,\db \vs $28.54$\,\db).

        %============================================================================
        \noIndentHeading{Qualitative Results.}
            \cref{fig:teaser,fig:qualitative_nvs} show comparisons using the \nht \cite{condor2026nht} across five datasets, confirming that adding \qnn recovers sharper details and thin structures. 

        % %============================================================================
        % \noIndentHeading{\HighFid Comparison.}
        %     \cref{tab:nvs_edge_memory_inference} reports \psnrEdge computed over dilated Canny edges \cite{canny1986computational,feng2025sasnet} on the \mipNerf dataset \cite{barron2022mipnerf360}. 
        %     Adding \qnn enhances the edge fidelity of all baselines. 
        %     Crucially, \nht+ \qnn achieves a new \sota \psnrEdge ($26.96$\,\db), while even \mcmc+ \qnn outperforms \zipNerf \cite{barron2023zip} ($26.45$\,\db \vs $26.34$\,\db).

        % %============================================================================
        % \noIndentHeading{Efficiency.}
        %     \cref{tab:nvs_edge_memory_inference} shows that adding \qnn consistently boosts fidelity across all baselines with minimal training overhead ($<\!0.04$ additional GPU-hours) and negligible storage overhead ($<\!0.2\%$). 
        %     Although evaluating the \qnn moderately reduces raw frame rates, all variants sustain real-time rendering ($87$--$191$\,\fps). 
        %     Neural acceleration techniques such as quantized fine-tuning \cite{chen2023mobilenerf} and fused kernels implementations \cite{muller2022instant} can further reduce latency, which we leave to future work.
        %     Remarkably, \mcmc+ \qnn exceeds \zipNerf \cite{barron2023zip} in both overall \psnr and \psnrEdge, while requiring $80\times$ fewer training GPU-hours per scene ($0.40$ \vs $32.00$) and rendering more than $350\times$ faster ($177.1$ \vs $0.5$ \fps). 

        %============================================================================
        \noIndentHeading{\HighFid Comparison and Efficiency.}
            To directly quantify \hf synthesis, \cref{tab:nvs_edge_memory_inference} reports \psnrEdge over dilated Canny edges \cite{canny1986computational,feng2025sasnet} on \mipNerf \cite{barron2022mipnerf360} alongside computational costs on an H100 GPU. 
            Adding \qnn consistently boosts overall \psnr and \psnrEdge across all baselines, setting a new \sota \psnrEdge of $26.96$\,\db with \nht + \qnn, while incurring minimal training overhead ($<\!0.04$ GPU-hours) and negligible storage overhead ($<\!0.2\%$). 
            Although evaluating the \qnn reduces raw frame rates, all variants sustain real-time rendering ($87$--$191$\,\fps). 
            Neural acceleration techniques such as quantized fine-tuning \cite{chen2023mobilenerf} and fused kernels \cite{muller2022instant} can further reduce latency, which we leave to future work.
            Remarkably, \mcmc + \qnn exceeds \zipNerf \cite{barron2023zip} in both \psnr and \psnrEdge ($26.45$ \vs $26.34$\,\db) while requiring $80\times$ fewer training GPU-hours per scene ($0.40$ \vs $32.00$) and rendering over $350\times$ faster ($177.1$ \vs $0.5$\,\fps).

    %============================================================================
    %============================================================================
    \subsection{Alternatives to \qnn}\label{sec:nvs_alter}

        We next evaluate alternative strategies for boosting \gsOnly fidelity on \mipNerf across validation (generalization) and training (fitting) splits, using the vanilla \gs \cite{kerbl2023gaussians} baseline (and \mcmc \cite{kheradmand20243d} when sweeping fixed Gaussian primitive budgets).
        
        %============================================================================
        \noIndentHeading{Increasing Primitive Count.}
            We first test whether brute-force primitive scaling resolves fine details. 
            Since vanilla \gs \cite{kerbl2023gaussians} does not fix the final primitive count, and \dbs \cite{liu2025deformable}, \sv \cite{di2026spherical}, and \nht \cite{condor2026nht} use scene-dependent budgets with non-Gaussian $\beta$-primitives (\qnn is orthogonal and boosts all three in \cref{tab:nvs_results}), we use \mcmc \cite{kheradmand20243d} to sweep exact Gaussian budgets from $100$K to $20$M per scene on \mipNerf.
            \cref{fig:mcmc_scale_gaussians_trainval} shows that training \psnr closely follows our exponential bound (\cref{th:num_gaussians}).
            However, its zoomed validation view (\cref{fig:mcmc_scale_gaussians} in \supp) shows that \mcmc \cite{kheradmand20243d} validation falters beyond $4$M Gaussians and never reaches \zipNerf \cite{barron2023zip} ($28.54$\,\db), whereas \mcmc + \qnn improves all budgets ($+0.19$ to $+0.33$\,\db) and surpasses \zipNerf \cite{barron2023zip} at $4$M--$10$M primitives.

\begin{figure*}[!t]
    \centering
    \sbox{\nvsfigbox}{\includegraphics[width=0.98\linewidth]{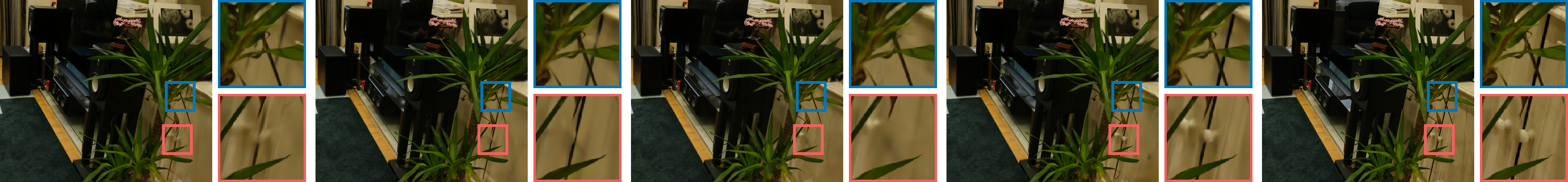}}%
    \setlength{\imgh}{\ht\nvsfigbox}%
    \addtolength{\imgh}{\dp\nvsfigbox}%

    \hspace*{0.02\linewidth}%
    \makebox[0.193\linewidth]{\small No Coordinate Queries}%
    \makebox[0.193\linewidth]{\small No \lf}%
    \makebox[0.193\linewidth]{\small No \Neighborhood}%
    \makebox[0.193\linewidth]{\small Full \qnn}%
    \makebox[0.193\linewidth]{\small Ground Truth}\par\unskip
    \vspace{-0.03cm}%
    \begin{minipage}[b][\imgh][c]{0.02\linewidth}
        \fontsize{7.5}{9}\selectfont\centering
        \parbox[c][\imgh][c]{\linewidth}{\centering\rotatebox{90}{Room~\cite{barron2022mipnerf360}}}%
    \end{minipage}%
    \hspace*{0pt}%
    \usebox{\nvsfigbox}\par\unskip
    \vspace{-0.22cm}%
    \caption{\textbf{Qualitative Ablation} of adding \qnn to \gs \cite{kerbl2023gaussians} on the \mipNerf dataset \cite{barron2022mipnerf360}. 
        \textbf{\qnn needs both coordinate queries and \lf signal as inputs and also needs \neighborhood} to reconstruct thin structures.         
        Inset figures highlight differences.
    }
    \label{fig:nvs_ablation_qualitative}
    \vspace{-0.55cm}
\end{figure*}

\begin{table}[!t]
    \centering
    \caption{\textbf{Adding parameters to primitives} of \gs \cite{kerbl2023gaussians} on the \mipNerf dataset \textbf{overfits}, and does not generalize well.
    \textbf{\qnn shows good generalization (val) and fitting (train)} performance.
    [Key: \firstBText{Best}~, \secondBText{Second-best}~, \thirdBText{Third-best}~].
    }
    \label{tab:nvs_gauss_params}
    \vspace{-0.2cm}
    \scalebox{0.7}{
    \setlength\tabcolsep{0.05cm}
    \begin{tabular}{l m ccc m ccc }
        \multirow{2}{*}{Method} &  \multicolumn{3}{cm}{\textbf{Val}} & \multicolumn{3}{c}{\textbf{Train}}\\
        & \psnr (\uparrowRHDSmall) & \ssim (\uparrowRHDSmall) & \lpips (\downarrowRHDSmall) & \psnr (\uparrowRHDSmall) & \ssim (\uparrowRHDSmall) & \lpips (\downarrowRHDSmall)\\
        \myTopRule
        \gs~\cite{kerbl2023gaussians} & \secondF{27.67} & \secondF{0.82} & \firstF{0.20} & $29.71$ & \thirdF{0.90} & \thirdF{0.16}\\
        \gs + $1$-layer \cnn & \thirdF{27.47} & \secondF{0.82} & \secondF{0.21} & \thirdF{30.12} &  \secondF{0.91} &  \secondF{0.15}\\
        \gs + $2$-layer \cnn & $26.83$ & \thirdF{0.81} & \thirdF{0.23} & \secondF{30.56} &  \firstF{0.92} &  \firstF{0.14}\\
        \gs + $4$-layer \cnn & $27.42$ & \thirdF{0.81} & \thirdF{0.23} & \firstF{31.23} &  \firstF{0.92} & \firstF{0.14}\\
        \hline
        \gs + \qnn & \firstF{27.96} & \firstF{0.83} & \firstF{0.20} & $30.11$ & \thirdF{0.90} & \thirdF{0.16}\\
    \end{tabular}
    }
    \vspace{-0.3cm}
\end{table}

\begin{table}[!t]
    \centering
    \caption{\textbf{Correcting \gs outputs} with an off-the-shelf \sr network (\invSR) on \mipNerf.
    \textbf{Adding \qnn outperforms} pretrained \sr.
    [Key: \firstBText{Best}~, \secondBText{Second-best}~, \thirdBText{Third-best}~].
    }
    \label{tab:nvs_sr}
    \vspace{-0.2cm}
    \scalebox{0.7}{
    \setlength\tabcolsep{0.1cm}
    \begin{tabular}{l m ccc}
        Method & \psnr (\uparrowRHDSmall) & \ssim (\uparrowRHDSmall) & \lpips (\downarrowRHDSmall) \\
        \myTopRule
        \gs~\cite{kerbl2023gaussians} & \secondF{27.67} & \secondF{0.82} & \firstF{0.20} \\
        \gs + \invSR \cite{yue2025arbitrary} & \thirdF{24.10} & \thirdF{0.70} & \thirdF{0.32} \\
        \hline
        \gs + \qnn & \firstF{27.96} & \firstF{0.83} & \firstF{0.20} \\
%         \mcmc~\cite{kheradmand20243d} & \secondF{28.26} & \firstF{0.84} & \secondF{0.17} \\
%         \mcmc + \invSR \cite{yue2025arbitrary} & \thirdF{24.27} & \secondF{0.70} & \thirdF{0.30} \\
%         \hline
%         \mcmc + \qnn & \firstF{28.58} & \firstF{0.84} & \firstF{0.16} \\
    \end{tabular}
    }
    \vspace{-0.4cm}
\end{table}

        %============================================================================
        \noIndentHeading{Adding Parameters to Primitives.}
            Rather than scaling the primitive count in \threeD, a second alternative is expanding per-primitive capacity by attaching $32$ extra parameters directly to each \threeD Gaussian in \gs (similar to Textured-GS \cite{chao2025textured}), splatting them to \twoD, and decoding residuals via $1$- to $4$-layer \cnns (\cref{tab:nvs_gauss_params}).  
            \cref{tab:nvs_gauss_params} shows that while deeper \cnns boost training \psnr up to $31.23$\,\db, validation \psnr drops to $26.83$--$27.47$\,\db, confirming that per-primitive parameter expansion severely overfits training views compared to \qnn refinement ($30.11$\,\db train, $27.96$\,\db validation).

        %============================================================================
        \noIndentHeading{Off-the-Shelf Image \sr.}
            Turning to \twoD post-processing, a third alternative is refining rendered \gs views with a pretrained image \sr network such as \invSR \cite{yue2025arbitrary}. 
            \cref{tab:nvs_sr} shows that \invSR hallucinates view-inconsistent details and degrades validation \psnr by over $3.5$\,\db, whereas adding \qnn faithfully boosts fidelity ($27.96$\,\db).

        %============================================================================
        \noIndentHeading{\mlp-based Neural Fields.}
            Finally, we replace \qnn on \gs with coordinate-based \mlps taking \twoD queries as input, including vanilla \mlps \cite{mildenhall2020nerf} (with same or larger channels as \qnn), Fourier \mlp \cite{tancik2020fourier}, \HashGrid \mlp \cite{muller2022instant}, and \mlps with \erf \cite{yang2019fine}, \sinc \cite{saratchandran2024activation}, \siren \cite{sitzmann2020siren}, and \finer \cite{liu2024finer} activations. 
            None of these architectures meaningfully improve validation or training metrics in \cref{tab:nvs_arch}. 
            Notably, the negligible training gain of Fourier \mlp ($29.77$\,\db) over \gs ($29.71$\,\db) demonstrates that Fourier \mlps \textbf{underfit} training views, whereas \qnn reaches $30.11$\,\db training and $27.96$\,\db validation \psnr.

\begin{table}[!t]
    \centering
    \caption{\textbf{\mlp-based Neural Fields} with \gs \cite{kerbl2023gaussians} on the \mipNerf dataset.
    \textbf{Adding \qnn outperforms all \mlps} on both val generalization and training fit.
    We do not highlight if more than $5$ entries share the same value.
    [Key: \firstBText{Best}~, \secondBText{Second-best}~, \thirdBText{Third-best}~].
    }
    \label{tab:nvs_arch}
    \vspace{-0.2cm}
    \scalebox{0.7}{
    \setlength\tabcolsep{0.05cm}
    \begin{tabular}{l m ccc m ccc }
        \multirow{2}{*}{Method} &  \multicolumn{3}{cm}{\textbf{Val}} & \multicolumn{3}{c}{\textbf{Train}}\\
        & \psnr (\uparrowRHDSmall) & \ssim (\uparrowRHDSmall) & \lpips (\downarrowRHDSmall) & \psnr (\uparrowRHDSmall) & \ssim (\uparrowRHDSmall) & \lpips (\downarrowRHDSmall)\\
        \myTopRule
        \gs~\cite{kerbl2023gaussians} & $27.67$ & $0.82$ & $0.20$ & $29.71$ & \firstF{0.90} & $0.16$\\
        + Vanilla \mlp$\!_{\text{same}}$ & $27.71$ & \firstF{0.83} & $0.20$ & $29.70$ & $0.89$ & $0.16$\\
        + Vanilla \mlp~\cite{mildenhall2020nerf} & $27.74$ & \firstF{0.83} & $0.20$ & \thirdF{29.72} & $0.89$ & $0.16$\\
        + Fourier \mlp~\cite{tancik2020fourier} & \thirdF{27.78} & $0.82$ & $0.20$ & \secondF{29.77} & $0.89$ & $0.16$\\
        + \HashGrid \mlp~\cite{muller2022instant} & \secondF{27.79} &  $0.82$ & $0.20$ & \secondF{29.77} & $0.89$  & $0.16$ \\
        + \erf \mlp~\cite{yang2019fine} & $27.68$ & $0.82$ & $0.20$ & $29.66$ & $0.89$ & $0.16$ \\
        + \sinc \mlp~\cite{saratchandran2024activation} & $27.68$ & $0.82$ & $0.20$ & $29.66$ & $0.89$ & $0.16$ \\
        + \siren \mlp~\cite{sitzmann2020siren} & $14.49$ & $0.62$ & $0.51$ & $14.66$ & $0.67$ & $0.49$ \\
        + \finer \mlp~\cite{liu2024finer} & $27.64$ & $0.82$ & $0.20$ & $29.66$ & $0.89$ & $0.16$ \\
        \hline
        + \qnn & \firstF{27.96} & \firstF{0.83} & \firstF{0.20} & \firstF{30.11} & \firstF{0.90} & \firstF{0.16}\\
    \end{tabular}
    }
    \vspace{-0.2cm}
\end{table}

\begin{table}[!t]
    \centering
    \caption{\textbf{Ablations} of \qnn with \gs on \mipNerf.
    \qnn achieves the best validation metrics across all choices.
    We do not highlight if more than $5$ entries share the same value.
    [Key: \firstBText{Best}~, \secondBText{Second-best}~, \thirdBText{Third-best}~, Nbd= \Neighborhood{}].
    }
    \label{tab:ablation}
    \vspace{-0.3cm}
    \scalebox{0.64}{
    \setlength\tabcolsep{0.03cm}
    \begin{tabular}{l l m ccc m ccc }
        \multirow{2}{*}{Change} & \multirow{2}{*}{From \rightarrowRHDSmall~To} &  \multicolumn{3}{cm}{\textbf{Val}} & \multicolumn{3}{c}{\textbf{Train}}\\
        & & \psnr\!\uparrowRHDSmall & \ssim\!\uparrowRHDSmall & \lpips\!\downarrowRHDSmall & \psnr\!\uparrowRHDSmall & \ssim\!\uparrowRHDSmall & \lpips\!\downarrowRHDSmall\\
        \myTopRule
        \gs & \mathDash & $27.67$ & $0.82$ & $0.20$ & $29.71$ & $0.90$ & $0.16$\\
        \hline
        \multirow{4}{*}{Core} & Coord,Nbd\rightarrowRHDSmall No Coord,No Nbd & $27.63$ & $0.82$ & $0.20$ & $29.69$ & $0.90$ & $0.16$ \\
        & Coord \rightarrowRHDSmall~No Coord & $27.67$ & $0.82$ & $0.20$ & $30.08$ & $0.90$ & $0.16$\\
        & \lf \rightarrowRHDSmall~No \lf & $27.77$ & $0.82$ & $0.20$ & $29.73$ & $0.90$ & $0.16$\\
        & Nbd \rightarrowRHDSmall~No Nbd & $27.85$ & $0.82$ & $0.20$ & $29.88$ & $0.90$ & $0.16$ \\  
        \hline
        \multirow{4}{*}{Queries} & \twoD \rightarrowRHDSmall~\threeD location & $27.59$ & $0.82$ & \firstF{0.19} & \firstF{30.74} & \firstF{0.91} & \firstF{0.14}\\
        & \twoD \rightarrowRHDSmall~\threeD location\hspace{0.06cm}$\concat$ Direction & $27.66$ & $0.82$ & $0.20$ & $29.87$ & $0.90$ & $0.16$\\
        & \twoD \rightarrowRHDSmall~\Raymap~\cite{mildenhall2020nerf} & $27.73$ & $0.82$ & $0.20$ & $29.85$ & $0.90$ & $0.16$\\
        % & \twoD \rightarrowRHDSmall~\Raymap $\concat$ \twoD & \thirdF{27.88} & $0.82$ & $0.20$ & \thirdF{30.03} & $0.90$ & $0.16$\\
        & \twoD \rightarrowRHDSmall~\Plucker~\cite{plucker1828analytisch} & $27.67$ & $0.82$ & $0.20$ & $29.85$ & $0.90$ & $0.16$\\
        \hline
        \multirow{3}{*}{Activation} & \relu \rightarrowRHDSmall~\erf~\cite{yang2019fine} & $26.43$ & $0.80$ & $0.22$ & $27.98$ & $0.87$ & $0.18$\\
        & \relu \rightarrowRHDSmall~\sinc~\cite{saratchandran2024activation} & $27.70$ & $0.82$ & $0.20$ & $29.79$ & $0.90$ & $0.16$\\
        & \relu \rightarrowRHDSmall~\finer~\cite{liu2024finer} & $27.33$ & $0.81$ & $0.22$ & $28.62$ & $0.87$ & $0.19$ \\
        \hline
        \multirow{3}{*}{Encoding} & Vanilla \rightarrowRHDSmall~Fourier \cite{tancik2020fourier} & \thirdF{27.85} & $0.82$ & $0.20$ & $29.90$ & $0.90$ & $0.16$ \\
        & Vanilla \rightarrowRHDSmall~Per-axis Fourier & \thirdF{27.85} & $0.82$ & $0.20$ & $29.89$ & $0.90$ & $0.16$ \\
        & Vanilla \rightarrowRHDSmall~Exponential \cite{mildenhall2020nerf} & \secondF{27.94} & $0.82$ & $0.20$ & \thirdF{29.99} & $0.90$ & $0.16$ \\
        \hline
        + \qnn & \mathDash & \firstF{27.96} & \firstF{0.83} & $0.20$ & \secondF{30.11} & $0.90$ & $0.16$\\
    \end{tabular}
    }
    \vspace{-0.65cm}
\end{table}

    %============================================================================
    %============================================================================
    \subsection{Ablation Studies}\label{sec:ablation} 

        \cref{tab:ablation} ablates the design choices of \qnn using the \gs on \mipNerf across both validation and training splits.

        %============================================================================
        \noIndentHeading{Core Elements.}
            We first evaluate the three pillars of \qnn: coordinate queries (Coord), the \lf signal, and the \neighborhood. 
            \cref{tab:ablation} directly validates \cref{th:monotone_feat}, as training \psnr increases monotonically from pointwise \lf refinement (No Coord, No Nbd: $29.69$\,\db) to \lf convolutions (No Coord: $30.08$\,\db) and full \qnn ($30.11$\,\db). 
            Moreover, \cref{tab:ablation,fig:nvs_ablation_qualitative} show that ablating coordinates (No Coord), the \lf signal (No \lf), or $3{\times}3$ receptive fields (No Nbd / No \Neighborhood) degrades both splits and impairs thin-structure recovery.
     
        %============================================================================
        \noIndentHeading{Queries.}
            We next explore four query alternatives:
            \begin{itemize}
                \item \threeD location: Unprojecting pixel locations into \threeD via camera parameters and splatted depth.
                \item \threeD location $\concat$~direction: Appending \threeD view directions.
                \item \Raymap: Using \threeD camera centers with view directions, as in \nerf \cite{mildenhall2020nerf,gao2024cat3d}.
                \item \Plucker: Using \Plucker ray coordinates \cite{plucker1828analytisch,grossberg2001general}.
            \end{itemize}
            While unprojected \threeD locations achieve the highest training \psnr ($30.74$\,\db), they overfit training views and drop to $27.59$\,\db on validation due to depth inaccuracies, whereas \twoD queries generalize best on validation ($27.96$\,\db).

        %============================================================================
        \noIndentHeading{Activation.}
            We then replace \relu activations in \qnn with \erf \cite{yang2019fine}, \sinc \cite{saratchandran2024activation}, and \finer \cite{liu2024finer}. 
            \cref{tab:ablation} shows that \relu outperforms all specialized activations across both validation and training splits.

        %============================================================================
        \noIndentHeading{Encoding.}
            Finally, replacing the default Vanilla encoding in \qnn with randomized Fourier \cite{tancik2020fourier}, per-axis Fourier, or exponential \nerf encodings \cite{mildenhall2020nerf} decreases \nvs performance on both validation and training sets.

\begin{figure}[!tb]
    \centering
    \includegraphics[width=0.8\linewidth]{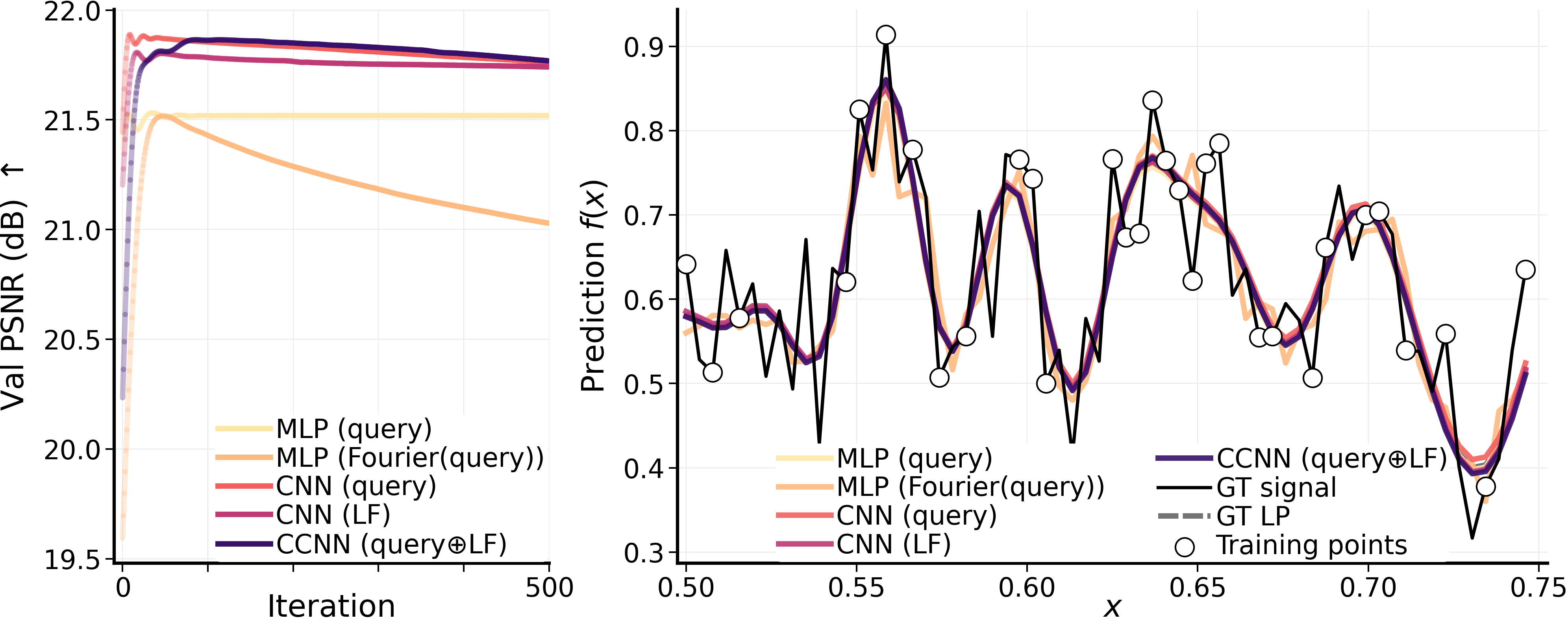}
    \caption{\textbf{1D Regression.} 
    This toy experiment compares networks which take the \oneD queries and \lf signal as input.
    The standard \mlp-based networks including Fourier encodings take \oneD coordinates as inputs.
    \qnn changes the linear layer to a \oneD convolutional layer and also takes the \lf signal in addition to the \oneD queries.
    \textbf{\qnn outperforms \mlp-based architectures including Fourier \mlps}.
    [Key: $\concat$= Concatenation].
    }
    \label{fig:1d_reg}
    \vspace{-0.5cm}
\end{figure}

%         %============================================================================
%         \noIndentHeading{Feedforward \nvs.}\label{sec:nvs_feedforward}
%             We next evaluate the feedforward \nvs like \longLRM \cite{ziwen2025longlrm} on \dlThreedv dataset in \cref{tab:nvs_feedforward_results}.

% \begin{table}[t]
%     \caption{
%         \textbf{Feedforward \nvs Results.}
%         \textbf{Adding \qnn outperforms} baselines on feedforward task. 
%         [Key: \firstBText{Best}~, \secondBText{Second-best}].
%     }
%     \label{tab:nvs_feedforward_results}
%     \centering
%     \vspace{-0.3cm}
%     \small
%     \scalebox{.7}{    
%         \begin{tabular}{l@{\hspace{.2cm}}mc@{\hspace{.2cm}}c@{\hspace{.2cm}}c}
%         & \psnr ($\uparrowRHDSmall$) & \ssim ($\uparrowRHDSmall$) & \lpips ($\downarrowRHDSmall$) \\
%         \myTopRule
%         \longLRM & $22.05$ & $0.70$ & $0.27$ \\
%         \longLRM + \qnn &  \\
%         \end{tabular}
%     }
%     \vspace{-0.3cm}
% \end{table}

    %============================================================================
    %============================================================================
    \subsection{Other Applications}\label{sec:other_tasks}

        We next evaluate on \oneD and \twoD regression and \sr tasks. 
        See \cref{tab:exp_setting} for experimental settings and \cref{sec:app_imp_details} for details. 

        %============================================================================
        \noIndentHeading{1D Regression.}\label{sec:1dreg}
            We now compare the performance of \mlp, \cnn and \qnn on a toy \oneD regression task.
            % involving a \hf target signal.
            The task involves taking \oneD coordinates (and \lf signal) as input to produce \oneD output.
            We use the synthetic \oneD signal from \cite{tancik2020fourier}.
            \cref{fig:1d_reg} confirms that all \cnn models including \qnn outperform \mlps, including Fourier \mlp, in \psnr.

        %============================================================================
        \noIndentHeading{2D Regression.}\label{sec:2dreg}
            We next evaluate the \twoD residual image regression task, where the networks take the \twoD coordinates (and \lf image) as input to output \threeD RGB values.
            We use the \mipNerf dataset \cite{barron2022mipnerf360} for this experiment.
            \cref{tab:img_reg_results} confirms that adding \qnn benefits learning details for \twoD residual image regression task on both val and train splits.

\begin{table}[t]
    \caption{
        \textbf{2D Regression Results.}
        \textbf{Adding \qnn outperforms} baselines on \twoD regression task. 
        [Key: \firstBText{Best}~, \secondBText{Second-best}].
    }
    \label{tab:img_reg_results}
    \centering
    \vspace{-0.3cm}
    \small
    \scalebox{.7}{    
        \begin{tabular}{l@{\hspace{.2cm}}c@{\hspace{.2cm}}ml@{\hspace{.2cm}}l@{\hspace{.2cm}}ml@{\hspace{.2cm}}l}
        \multirow{2}{*}{Encoding $(\encoding)$} & \multirow{2}{*}{\qnn} & \multicolumn{2}{cm}{\bfseries Val} & \multicolumn{2}{c}{\bfseries Train}\\
        & & \psnr ($\uparrowRHDSmall$) & \ssim ($\uparrowRHDSmall$) & \psnr ($\uparrowRHDSmall$) & \ssim ($\uparrowRHDSmall$) \\
        \myTopRule
        \LowFreq images & \mathDash & $29.61$ & $0.89$ & $29.65$ & $0.89$ \\
        \HashGrid \cite{muller2022instant}~~~~~~~~~~~& \xmark & $29.35$ & $0.88$ & $30.34$ & $0.88$ \\
        \hline
        \multirow{2}{*}{Vanilla} & \xmark & $29.69$ & $0.92$ & $29.69$ & $0.92$ \\ 
        & \cmark & \secondF{32.48}& \firstF{0.96} & \secondF{32.51} & \firstF{0.96} \\
        \hline
        \multirow{2}{*}{Fourier \cite{tancik2020fourier}} & \xmark & $29.79$ & $0.92$ & $29.80$ & $0.92$\\
        & \cmark & \firstF{32.87} & \secondF{0.94} & \firstF{32.93} & \secondF{0.94}\\
        \end{tabular}
    }
    \vspace{-0.3cm}
\end{table}

\begin{table}[t] 
    \caption{
        \textbf{2D \sr Results.}
        \textbf{Adding \qnn outperforms} baselines on both \sr tasks. 
        [Key: \firstBText{Best}~, \secondBText{Second-best}, \retrained= Retrained]. % \released= Reported, \firstKey{Best}, 
    }
    \label{tab:sr_results}
    \centering
    \vspace{-0.3cm}
    \small
    \scalebox{.7}{    
        \begin{tabular}{l@{\hspace{.2cm}}c@{\hspace{.2cm}}ml@{\hspace{.2cm}}l@{\hspace{.2cm}}ml@{\hspace{.2cm}}l}
        \multirow{2}{*}{Method} & \multirow{2}{*}{\qnn} & \multicolumn{2}{cm}{\textbf{Face \sr}} & \multicolumn{2}{c}{\textbf{Natural \sr}} \\
        & & \psnr ($\uparrowRHDSmall$) & \ssim ($\uparrowRHDSmall$) & \psnr ($\uparrowRHDSmall$) & \ssim ($\uparrowRHDSmall$) \\
        \myTopRule
        PULSE \cite{menon2020pulse} & \mathDash & $16.88$ & $0.44$ & \mathDash & \mathDash\\
        FSRGAN \cite{chen2018fsrnet} & \mathDash & $23.01$ & $0.62$ & \mathDash & \mathDash\\
        % \srThree \released \cite{saharia2022image} & \xmark & $23.04$ & $0.65$ & \mathDash & \mathDash\\
        \srThree \retrained \cite{saharia2022image} & \xmark & \secondF{23.51} & \secondF{0.68} & \mathDash & \mathDash \\
        \srThree \cite{saharia2022image} & \cmark & \firstF{23.63} & \firstF{0.69} & \mathDash & \mathDash \\
        \myTopRule
        % \realEsrgan \released \cite{wang2021realesrgan} & \xmark & \mathDash & \mathDash & $24.84$ & $0.71$ \\
        \realEsrgan \retrained \cite{wang2021realesrgan} & \xmark & \mathDash & \mathDash & \secondF{24.41} & \firstF{0.70} \\
        \realEsrgan \cite{wang2021realesrgan} & \cmark & \mathDash & \mathDash & \firstF{24.63} & \firstF{0.70}\\
        \end{tabular}
    }
    \vspace{-0.5cm}
\end{table}

        %============================================================================
        \noIndentHeading{2D Super-Resolution (\sr).}\label{sec:sr}
            \cref{tab:sr_results} further evaluates on Face and Natural image \sr tasks.
            The findings confirm that adding \qnns consistently increases the fidelity for both \sr sub-tasks. 
            \cref{fig:qualitative_sr} in \supp shows qualitative results.

%============================================================================
%============================================================================
%============================================================================
\section{Conclusions}

    Pushing Gaussian splatting to arbitrary fidelity via brute-force primitive scaling faces a strict mathematical barrier: as we prove, raising training \psnr requires an exponentially increasing number of primitives, triggering optimization issues and validation overfitting. 
    To alleviate this primitive bloat, we proved that detail recovery requires jointly convolving \lf signals, local \neighborhood, and coordinate queries, motivating our use of existing \qnn. 
    By convolving rendered \lf signals with \twoD pixel coordinate queries to predict residuals, \qnn elevates five diverse \gsOnly baselines to a new \sotaFull fidelity in real-world \nvs benchmarks, while retaining real-time rendering speeds and minimal overhead. 
    Beyond \nvs, \qnns consistently boost \oneD/\twoD regression and \twoD super-resolution tasks. 
    Future work includes deriving finite-sample generalization bounds and extending \qnns to dynamic \fourD scene reconstruction and generative world modeling.

    %============================================================================
    \noIndentHeading{Limitations.}
        \qnn cannot be used in tasks without structured \neighborhoods, such as casting isolated rays in \nerf \cite{mildenhall2020nerf} or in \sdf regression \cite{park2019deepsdf}. 
        Evaluating the \qnn reduces the rendering speed compared to pure splatting (\cref{tab:nvs_edge_memory_inference}).
        Quantized fine-tuning \cite{chen2023mobilenerf} and fused kernels \cite{muller2022instant} can further reduce latency.

    %============================================================================
    \noIndentHeading{\llm Usage Statement.}
        We used \llm{}s and generative AI tools such as Google Gemini in correcting technical writing, editing, and grammar refinement of the paper.
        No AI tools were used for formulating the core idea.
        
{
    \small
    \bibliographystyle{ieeenat_fullname}
    \bibliography{references}
}

%============================================================================
%============================================================================
%============================================================================
\clearpage
\appendix
\setcounter{page}{13}
\maketitlesupplementary

{
\begingroup
\let\clearpage\relax
\hypersetup{linkcolor=bestBlueDark}
\tableofcontents
\endgroup
}
\renewcommand{\theHsection}{A\arabic{section}} % Makes hyperlink unique for appendix

\addtocontents{toc}{\protect\setcounter{tocdepth}{3}}
%============================================================================
%============================================================================
%============================================================================
\section{Qualitative Results}

    %============================================================================
    %============================================================================
    \subsection{3D \nvs}

        %============================================================================
        \noIndentHeading{More Qualitative Results.}
            We show more qualitative results of adding \qnn to \nht \cite{condor2026nht} baseline on \mipNerf \cite{barron2022mipnerf360} and \tAndT \cite{knapitsch2017tanks} datasets in \cref{fig:qualitative_nvs_1} as well as \deepBlend \cite{hedman2018deep} and \ommo \cite{lu2023large} datasets in \cref{fig:qualitative_nvs_2}.
            Adding \qnn reconstructs details resulting in higher fidelity synthesis visually.

        %============================================================================
        \noIndentHeading{Video Comparisons.}
            We also attach rendered novel-view videos from \nht \cite{condor2026nht}, and \nht + \qnn methods on \mipNerf \cite{barron2022mipnerf360} and \ommo \cite{lu2023large} scenes in the supplementary material.
            Please unzip the \texttt{supplementary.zip}, navigate to the unzipped directory and follow \texttt{README.md} to see these qualitative video results. 
            Specifically, double click the \texttt{index.html} page to view in your browser and compare the rendered videos. 
            We provide the blue sliders and full-screen options to inspect details.
            These rendered videos show that adding \qnn, faithfully renders the scene with higher fidelity compared to \nht baseline.
        
\begin{figure*}[!t]
    \centering
    \sbox{\nvsfigbox}{\includegraphics[width=\linewidth]{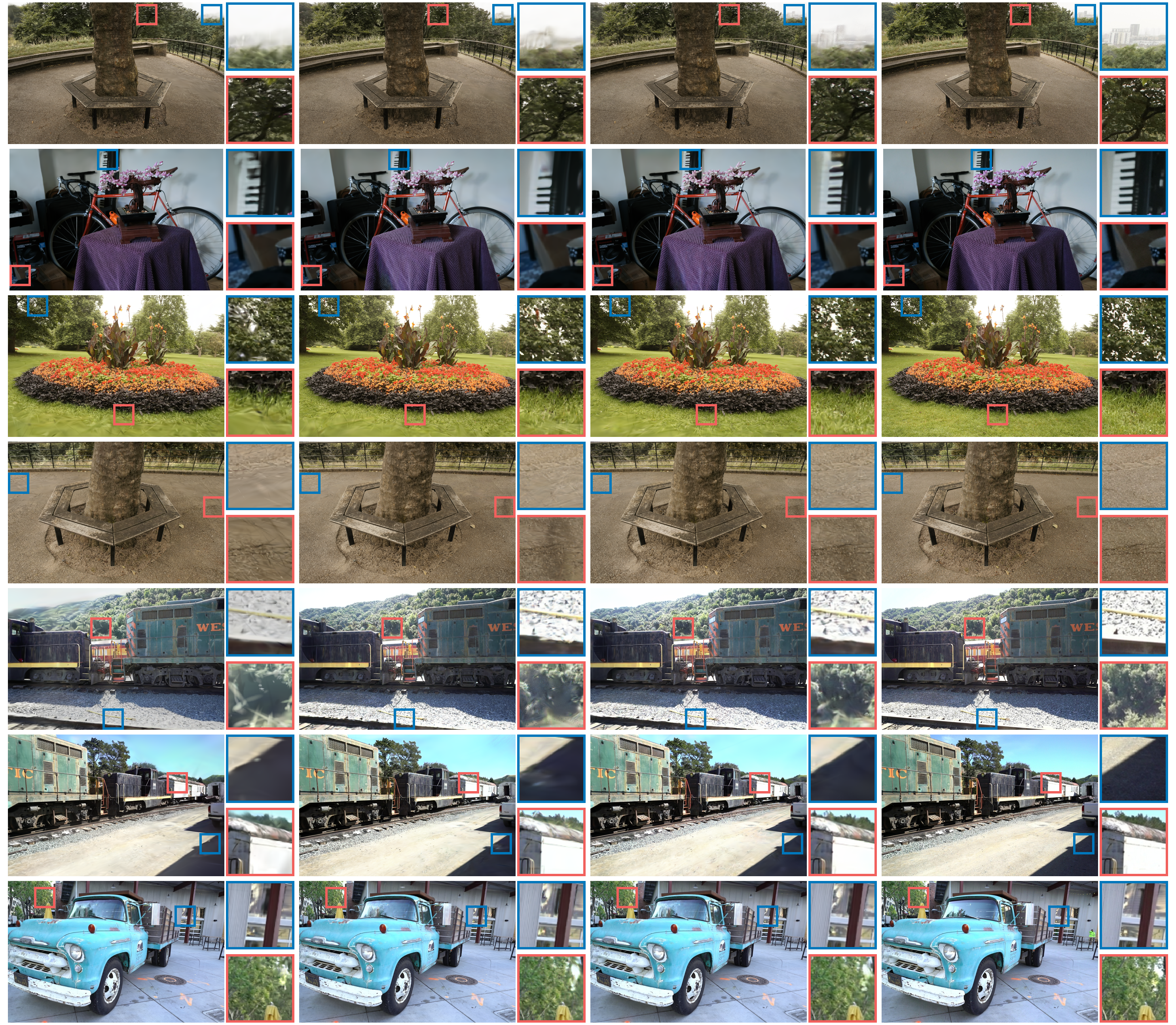}}%
    \setlength{\imgh}{\ht\nvsfigbox}%
    \addtolength{\imgh}{\dp\nvsfigbox}%

    \hspace*{0.02\linewidth}%
    \makebox[0.243\linewidth]{\small Spherical Voronoi \cite{di2026spherical}}%
    \makebox[0.243\linewidth]{\small \nht \cite{condor2026nht}}%
    \makebox[0.243\linewidth]{\small \nht+~\qnn}%
    \makebox[0.243\linewidth]{\small Ground Truth}\par
    \vspace{0.05cm}%
    \begin{minipage}[b][\imgh][c]{0.02\linewidth}
        \fontsize{7.5}{9}\selectfont\centering
        \parbox[c][\dimexpr\imgh/7\relax][c]{\linewidth}{\centering\rotatebox{90}{Treehill \cite{barron2022mipnerf360}}}\par
        \parbox[c][\dimexpr\imgh/7\relax][c]{\linewidth}{\centering\rotatebox{90}{Bonsai \cite{barron2022mipnerf360}}}\par
        \parbox[c][\dimexpr\imgh/7\relax][c]{\linewidth}{\centering\rotatebox{90}{Flowers \cite{barron2022mipnerf360}}}\par
        \parbox[c][\dimexpr\imgh/7\relax][c]{\linewidth}{\centering\rotatebox{90}{Treehill \cite{barron2022mipnerf360}}}\par
        \parbox[c][\dimexpr\imgh/7\relax][c]{\linewidth}{\centering\rotatebox{90}{Train \cite{knapitsch2017tanks}}}\par
        \parbox[c][\dimexpr\imgh/7\relax][c]{\linewidth}{\centering\rotatebox{90}{Train \cite{knapitsch2017tanks}}}\par
        \parbox[c][\dimexpr\imgh/7\relax][c]{\linewidth}{\centering\rotatebox{90}{Truck \cite{knapitsch2017tanks}}}%
    \end{minipage}%
    \hspace*{-1pt}%
    \usebox{\nvsfigbox}
    \caption{\textbf{\nvs Qualitative Results}.
        We provide more comparisons between previous works and adding \qnn with \nht \cite{condor2026nht} baseline on \mipNerf \cite{barron2022mipnerf360} and \tAndT \cite{knapitsch2017tanks} datasets. 
        \textbf{Adding \qnn results in higher fidelity synthesis visually}.
        Inset figures highlight differences.
    }
    \label{fig:qualitative_nvs_1}
    \vspace{-0.3cm}
\end{figure*}

\begin{figure*}[!t]
    \centering
    \sbox{\nvsfigbox}{\includegraphics[width=\linewidth]{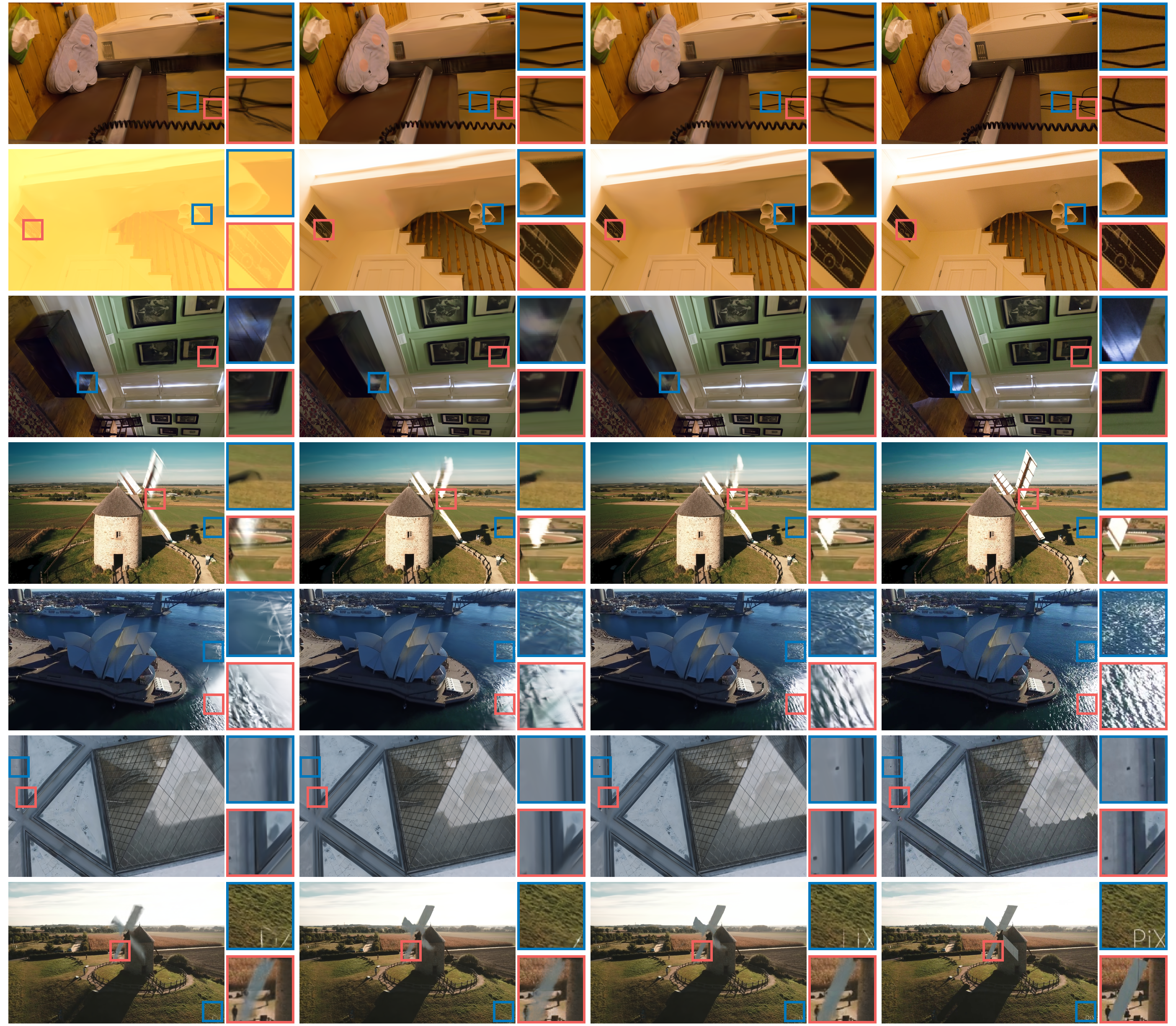}}%
    \setlength{\imgh}{\ht\nvsfigbox}%
    \addtolength{\imgh}{\dp\nvsfigbox}%

    \hspace*{0.02\linewidth}%
    \makebox[0.243\linewidth]{\small Spherical Voronoi \cite{di2026spherical}}%
    \makebox[0.243\linewidth]{\small \nht \cite{condor2026nht}}%
    \makebox[0.243\linewidth]{\small \nht+~\qnn}%
    \makebox[0.243\linewidth]{\small Ground Truth}\par
    \vspace{0.05cm}%
    \begin{minipage}[b][\imgh][c]{0.02\linewidth}
        \fontsize{7.5}{9}\selectfont\centering
        \parbox[c][\dimexpr\imgh/7\relax][c]{\linewidth}{\centering\rotatebox{90}{Playroom \cite{hedman2018deep}}}\par
        \parbox[c][\dimexpr\imgh/7\relax][c]{\linewidth}{\centering\rotatebox{90}{Playroom \cite{hedman2018deep}}}\par
        \parbox[c][\dimexpr\imgh/7\relax][c]{\linewidth}{\centering\rotatebox{90}{DrJohnson \cite{hedman2018deep}}}\par
        % Last 4 rows: OMMO
        \parbox[c][\dimexpr\imgh/7\relax][c]{\linewidth}{\centering\rotatebox{90}{06 \cite{lu2023large}}}\par
        \parbox[c][\dimexpr\imgh/7\relax][c]{\linewidth}{\centering\rotatebox{90}{05 \cite{lu2023large}}}\par
        \parbox[c][\dimexpr\imgh/7\relax][c]{\linewidth}{\centering\rotatebox{90}{13 \cite{lu2023large}}}\par
        \parbox[c][\dimexpr\imgh/7\relax][c]{\linewidth}{\centering\rotatebox{90}{06 \cite{lu2023large}}}%
    \end{minipage}%
    \hspace*{-1pt}%
    \usebox{\nvsfigbox}
    \caption{\textbf{\nvs Qualitative Results}.
        We provide more comparisons between previous works and adding \qnn with \nht \cite{condor2026nht} baseline on \deepBlend \cite{hedman2018deep} and \ommo \cite{lu2023large} datasets. 
        \textbf{Adding \qnn results in higher fidelity synthesis visually}.
        Inset figures highlight differences.
    }
    \label{fig:qualitative_nvs_2}
    \vspace{-0.3cm}
\end{figure*}

    %============================================================================
    %============================================================================
    \subsection{2D Image \sr}\label{sec:supp_sr}

        We provide some qualitative results of the \twoD image \sr experiment, with \realEsrgan \cite{wang2021realesrgan} as the baseline in \cref{fig:qualitative_sr}.
        Adding \qnn to \realEsrgan \cite{wang2021realesrgan} faithfully reconstructs details in various regions and results in higher fidelity synthesis visually.

\begin{figure*}[!t]
    \centering
    \makebox[0.266\linewidth]{\small \realEsrgan \cite{wang2021realesrgan}}%
    \makebox[0.266\linewidth]{\small \realEsrgan+ \qnn}%
    \makebox[0.266\linewidth]{\small Ground Truth}
    \vspace{0.1cm} 
    \includegraphics[trim=0 0 0 0.3cm, clip, width=0.8\linewidth]{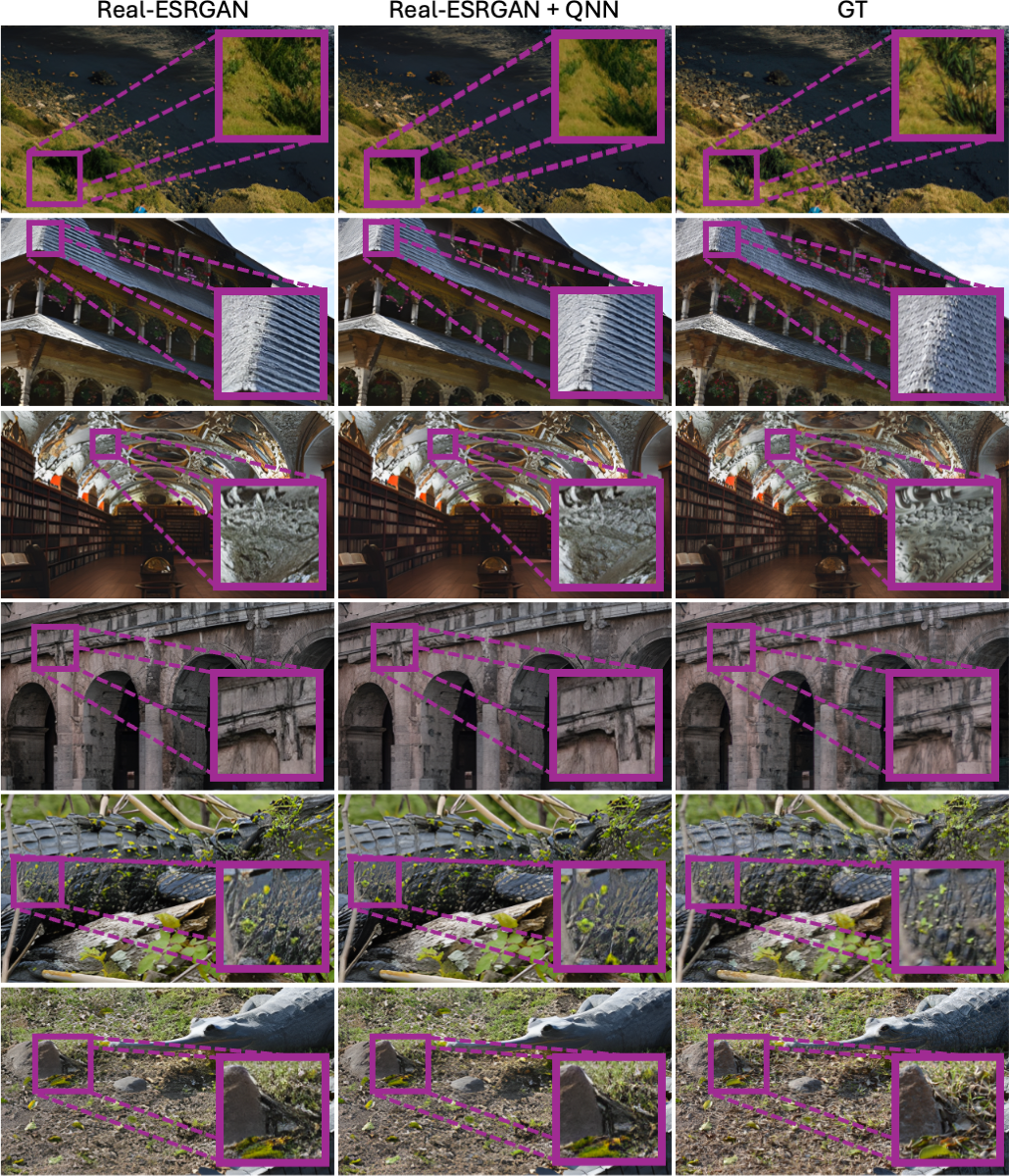}
    \caption{
        \textbf{\sr Results} of \divTwoK \val images.
        Adding \qnn to \realEsrgan \cite{wang2021realesrgan} faithfully reconstructs details in various regions and results in higher fidelity synthesis visually.
        We highlight the differences in inset figures.
    }
    \label{fig:qualitative_sr}
    \vspace{-0.6cm}
\end{figure*}

\newpage
\null
\newpage
\null
\newpage
\null
\newpage
\null
\newpage
\null
\newpage
\null
\newpage
% \null
%============================================================================
%============================================================================
%============================================================================
\section{Theoretical Results}

    We next include some proofs, which we could not put in the main paper due to space constraints.

    %============================================================================
    %============================================================================
    \subsection{Scaling Bottleneck of Gaussian Primitives}\label{sec:shortcoming_gaussian}

        We next present a theoretical result on the scaling bottleneck of the Gaussian primitives.
        We show that increasing training fidelity (\psnr) for a Gaussian-based representation requires an exponential increase in the number of \threeD Gaussians.

        \addtocounter{theorem}{-1}   % Bump theorem counter from 3 to 1
        \begin{theorem}[\textbf{Approximation Error with Wavelets}, Section III.B, \cite{maiorov2002lower}]\label{th:approx_wavelet}
            Let $\Omega\subset\realDomain^\inDim$ be a compact domain and let $\measFunc_p^{\sobolev,\inDim}$ denote the Sobolev class of real-valued measurable functions on $\Omega$ with smoothness parameter $\sobolev$ for some $\sobolev>0$.
            Next, consider the set of $\numPrim$-term $\wavelet$ wavelet approximations of the form:
            \begin{align}
                \nFieldScalar_\numPrim^l(\wavelet) &= \sum_{\gind=1}^\numPrim \mixCoeff_\gind \wavelet(\mixAffineMat_\gind\coordinates + \mixAffineBias_\gind), 
                \label{eq:wavelet_approx}
            \end{align}
            with the mixing parameters $\mixCoeff_\gind \in \realDomain$, affine terms $\mixAffineMat_\gind \in \realDomain^{l\times\inDim}$ and $\mixAffineBias_\gind \in \realDomain^{l}$.
            For real $\sobolev>0$, $1\le p,q<\infty$ satisfying $\frac{r}{\inDim}>\left(\frac{1}{p}{-}\frac{1}{q}\right)_+$ and integers $1\le\numPrim,\inDim<\infty$, the approximation error in $L_q$ norm is bounded from below by
            \begin{align}
                E_{\numPrim,q} := L_q(\measFunc_p^{\sobolev,\inDim}, \nFieldScalar_\numPrim^l(\wavelet)) &\ge \frac{\constant}{(\numPrim\log \numPrim)^{\sobolev/\inDim}},
                \label{eq:wavelet_approx_error}
            \end{align}
            where $\constant > 0$ is a constant independent of $\numPrim$.
        \end{theorem}

        %============================================================================
        \noIndentHeading{Projected Gaussian Primitives as a Dictionary Model.}
            For a fixed camera view, \gs \cite{kerbl2023gaussians} projects each \threeD Gaussian primitive $\gind \in \{1,\dots,\numPrim\}$ with \threeD mean and \threeD covariance onto the image plane ($\coordinates \in \realDomain^2$, so $\inDim=2$ in \cref{th:approx_wavelet}) via local affine EWA splatting \cite{zwicker2001ewa} to yield a projected mean $\vecCenter_\gind \in \realDomain^2$ and a positive-definite \twoD covariance $\matCov_\gind \in \realDomain^{2\times 2}$.
            To cast this into the dictionary framework of \cref{eq:wavelet_approx}, we set $l=2$ with the isotropic Gaussian generator:
            \begin{align}
                \wavelet(\vecZ)
                &=\exp\left(-\frac{\|\vecZ\|_2^2}{2}\right),
                \qquad \vecZ \in \realDomain^2,
                \label{eq:gauss_generator}
            \end{align}
            and define the affine parameters in \cref{eq:wavelet_approx} as
            \begin{align}
                \mixAffineMat_\gind &= \matCov_\gind^{-1/2} \in \realDomain^{2\times 2},
                \qquad
                \mixAffineBias_\gind = -\mixAffineMat_\gind\vecCenter_\gind \in \realDomain^2.
                \label{eq:affine_cov_map}
            \end{align}
            This parameterization satisfies $\mixAffineMat_\gind^\top\mixAffineMat_\gind = \matCov_\gind^{-1}$, where $\mixAffineMat_\gind$ controls the primitive's anisotropic shape, scale, and orientation on the image plane, and $\mixAffineBias_\gind$ shifts its center to $\vecCenter_\gind$.
            Under additive rendering, each projected primitive thus takes the exact form $\mixCoeff_\gind\wavelet(\mixAffineMat_\gind\coordinates+\mixAffineBias_\gind)$ in \cref{eq:wavelet_approx}.
            Note that while $\wavelet(\vecZ)$ has a non-zero integral, the lower bound in \cref{th:approx_wavelet} \cite{maiorov2002lower} relies on the finite-grid pseudo-dimension of the parameterized manifold $\nFieldScalar_\numPrim^l(\wavelet)$ and is explicitly proven for Gaussian and exponential-polynomial generators \cite[Sec. III-B]{maiorov2002lower}, without requiring classical zero-mean wavelet admissibility.

            Substituting \cref{eq:affine_cov_map} into the argument of \cref{eq:gauss_generator} gives
            \begin{align}
                \|\mixAffineMat_\gind\coordinates+\mixAffineBias_\gind\|_2^2 
                &= \|\matCov_\gind^{-1/2}(\coordinates - \vecCenter_\gind)\|_2^2 \nonumber\\
                &= (\coordinates - \vecCenter_\gind)^\top \matCov_\gind^{-1} (\coordinates - \vecCenter_\gind).
            \end{align}
            
            For a single color channel $\scalarColor$, the scalar coefficient $\mixCoeff_\gind \in \realDomain$ combines the primitive opacity $\opacity_\gind \in [0,1]$ and its view-dependent channel color $\scalarColor_\gind$:
            \begin{align}
                \mixCoeff_\gind &= \opacity_\gind\scalarColor_\gind.
                \label{eq:mix_coeff_def}
            \end{align}
            In standard \gs \cite{kerbl2023gaussians}, $\scalarColor_\gind$ is computed via Spherical Harmonics (SH) evaluated along the viewing ray from the camera center $\mathbf{o} \in \realDomain^3$ to the \threeD primitive center $\vecCenter^{3D}_\gind$, \ie, $\mathbf{d}_\gind = (\vecCenter^{3D}_\gind - \mathbf{o}) / \|\vecCenter^{3D}_\gind - \mathbf{o}\|_2$.
            Because $\mathbf{d}_\gind$ depends only on $\vecCenter^{3D}_\gind$ and $\mathbf{o}$ rather than the pixel coordinate $\coordinates$, the SH-evaluated color $\scalarColor_\gind(\mathbf{d}_\gind)$ is spatially constant across the entire image plane $\coordinates$ for any given view.
            Combining \cref{eq:gauss_generator,eq:affine_cov_map,eq:mix_coeff_def}, each projected primitive evaluates to
            \begin{align}
                \mixCoeff_\gind\wavelet(\mixAffineMat_\gind\coordinates+\mixAffineBias_\gind)
                &=\opacity_\gind\scalarColor_\gind\exp\left[
                    -\frac{1}{2}(\coordinates{-}\vecCenter_\gind)^\top
                    \matCov_\gind^{-1}(\coordinates{-}\vecCenter_\gind)
                \right].\nonumber
            \end{align}

            Summing over all $\numPrim$ Gaussian primitives yields the $\numPrim$-term additive dictionary rendering model on the image plane:
            \begin{align}
                \nFieldScalar(\coordinates) &= \sum_{\gind=1}^\numPrim \opacity_\gind\exp\left[
                    -\frac{1}{2}(\coordinates{-}\vecCenter_\gind)^\top
                    \matCov_\gind^{-1}(\coordinates{-}\vecCenter_\gind)
                    \right]
                    \scalarColor_\gind.
                    \label{eq:gs_addrender}
            \end{align}
            For comparison, standard alpha-composited \gsOnly rendering (\cref{eq:gs_alphacomposite} in \cref{sec:background}) for a single color channel $\scalarColor$ is given by:
            \begin{align}
                \nFieldScalar(\coordinates) &= \sum_{\gind=1}^\numPrim \opacityGT_\gind(\coordinates) \transmittance_\gind(\coordinates) \scalarColor_\gind \nonumber \\
                    &= \sum_{\gind=1}^\numPrim \opacity_\gind \exp\left[
                    -\frac{1}{2}(\coordinates{-}\vecCenter_\gind)^\top
                    \matCov_\gind^{-1}(\coordinates{-}\vecCenter_\gind)
                    \right]
                    \transmittance_\gind(\coordinates) 
                    \scalarColor_\gind,
                    \label{eq:gauss_addrender_singlecolor}
            \end{align}
            where $\opacityGT_\gind(\coordinates) = \opacity_\gind \exp\left[-\frac{1}{2}(\coordinates{-}\vecCenter_\gind)^\top\matCov_\gind^{-1}(\coordinates{-}\vecCenter_\gind)\right]$ is the effective opacity of Gaussian $G_\gind$ at a pixel $\coordinates$ and $\transmittance_\gind(\coordinates){=}\prod_{j=1}^{\gind-1} \left(1{-}\opacityGT_j(\coordinates)\right)$ is the transmittance.

            Comparing \cref{eq:gs_addrender,eq:gauss_addrender_singlecolor}, the additive dictionary rendering model (\cref{eq:gs_addrender}) corresponds to \textbf{a linear idealization of} alpha-composited \gsOnly rendering under unit transmittance ($\transmittance_\gind(\coordinates) = 1$, or more generally when $\transmittance_\gind(\coordinates) \approx \transmittance_\gind(\vecCenter_\gind)$ is spatially constant over the compact footprint of primitive $\gind$ in the large-$\numPrim$ regime and absorbed into $\mixCoeff_\gind$).
            Moreover, whereas \gsOnly restricts $\opacity_\gind, \transmittance_\gind \in [0,1]$, the dictionary $\nFieldScalar_\numPrim^l(\wavelet)$ in \cref{eq:wavelet_approx} permits unconstrained mixing weights $\mixCoeff_\gind \in \realDomain$. 
        \begin{lemma}[\textbf{Single-View \mse}]\label{lem:single_view_mse}
            Let a \twoD image be a function belonging to and attaining the minimax complexity of the fractional Sobolev class $\measFunc_p^{\sobolev,2}$ of smoothness $\sobolev > 0$ on the normalized image domain $\Omega = [0,1]^2 \subset \realDomain^2$, meaning its image values and fractional derivatives up to order $\sobolev$ have bounded $L_p$ norms for some fixed $p > \max\!\left(1, \frac{2}{1+\sobolev}\right)$ (\eg, any $p \ge 2$). 
            For an additive dictionary of $\numPrim$ projected Gaussian primitives, the worst-case best-approximation \mse for an image in $\measFunc_p^{\sobolev,2}$, denoted by $\risk_{\numPrim}$, is bounded from below by
            \begin{align}
                \risk_{\numPrim} \ge \frac{\constant^2}{(\numPrim\log \numPrim)^{\sobolev}}, 
                \label{eq:mse_lower}
            \end{align}
            where $\constant>0$ is a constant independent of $\numPrim$.
        \end{lemma}

        \begin{proof}
            For the projected image-plane primitives described above, the domain and affine dimensions are $\inDim=l=2$. 
            Consider a \twoD image attaining the minimax complexity of the fractional Sobolev class $\measFunc_p^{\sobolev,2}$ with smoothness $\sobolev > 0$. 
            Setting $q=2$, the condition $\sobolev/\inDim > (1/p - 1/q)_+$ required by \cref{th:approx_wavelet} becomes $\sobolev/2 > (1/p - 1/2)_+$, which is satisfied for any $p > \max\!\left(1, \frac{2}{1+\sobolev}\right)$ (\eg, any $p \ge 2$). 
            Consequently, using \cref{th:approx_wavelet}, the lower bound on the worst-case $L_2$ approximation error is given by:
            \begin{align}
                E_{\numPrim,2} \ge \frac{\constant}{(\numPrim\log \numPrim)^{\sobolev/2}}.
            \end{align}
            Because the normalized image domain $\Omega = [0,1]^2$ has unit measure, the squared $L_2$ error is identically the \mse. 
            Defining the worst-case best-approximation \mse as $\risk_{\numPrim} := E_{\numPrim,2}^2$, we get
            \begin{align}
                \risk_{\numPrim} &\ge \left[ \frac{\constant}{(\numPrim\log \numPrim)^{\sobolev/2}} \right]^2, \nonumber \\
                \implies \risk_{\numPrim} &\ge \frac{\constant^2}{(\numPrim\log \numPrim)^{\sobolev}},
            \end{align}
            which establishes the stated lower bound.
        \end{proof}

        \noIndentHeading{Remarks.}               
            \begin{itemize}
                \item \textbf{Fundamental Limit:} The error bound in \cref{eq:mse_lower} establishes a fundamental limit for the Sobolev class as a whole.
                While a specific simple image (\eg, a single \twoD Gaussian) can be fit exactly with a small Gaussian count $\numPrim$, there exists at least one image in this class whose approximation error \textbf{cannot} fall below \cref{eq:mse_lower}.
                Moreover, the real scenes do not contain single Gaussian.
                \item \textbf{Fractional Regularity:} DeVore et al. \cite[Section I, page 721]{devore1992image} prove that any function containing intensity jump discontinuities across any one-dimensional curves is excluded from Sobolev spaces $H^{\sobolev}$ for $\sobolev \ge 0.5$. 
                Consequently, natural scenes with sharp boundaries inhabit the fractional smoothness regime ($\sobolev \in (0, 0.5)$), corroborating our empirical finding of $\sobolev = 0.24$ in \cref{fig:mcmc_scale_gaussians_trainval}.
            \end{itemize}

        \setcounter{corollary}{0}
        \begin{corollary}[\textbf{Number of Gaussians for Multi-View Training \psnr}]
            Let a \threeD scene be rendered into $\numViews$ training views under an additive Gaussian dictionary model, where the normalized GT image for each training view $\view \in \{1, \dots, \numViews\}$ attains the minimax complexity of the fractional Sobolev class with smoothness $\sobolev > 0$ satisfying the conditions of \cref{lem:single_view_mse}. 
            To guarantee a mean training \psnr across views of at least $\psnrTarget$, the number of \threeD Gaussians $\numPrim$ must satisfy
            \vspace{-0.2cm}
            \begin{align}
                \numPrim &\ge \exp\left[ \lambert_0\left( \left[ \constant^2 10^{\frac{\psnrTarget}{10}} \right]^{\frac{1}{\sobolev}} \right) \right],
            \end{align}
            where $\lambert_0$ is the principal branch of the Lambert $\lambert$ function and $\constant>0$ is a constant independent of $\numPrim$.
        \end{corollary}

\begin{table*}[!t]
    \centering
    \caption{\textbf{Assumptions in \cref{th:num_gaussians} \vs \gs.} Real-world constraints such as local minima further restrict this lower bound.}
    \label{tab:theory_vs_gs}
    \scalebox{0.9}{
    \setlength{\tabcolsep}{0.25cm}
    \begin{tabular}{l m c c}
        \textbf{Aspect} & \textbf{\cref{th:num_gaussians}} & \textbf{\gs} \cite{kerbl2023gaussians} \\
        \myTopRule
        Target Metric & Train \psnr & Train and Val \psnr \\
        Image Class & Sobolev ($\sobolev > 0$) & Natural images ($\sobolev \approx 0.24$) \\
        Rendering & Additive dictionary sum & $\alpha$-compositing \\
        Transmittance ($\transmittance_\gind(\coordinates)$) & Spatially constant & Spatially varying \\
        Color ($\scalarColor_\gind$) & Spatially constant per-view Scalar & View-dependent Vector \\
        Visibility per view & Active subset ($\le \numPrim$) & Frustum culling ($\le \numPrim$) \\
        Optimum & Global Optimum & Local Optimum (Adam) \\
    \end{tabular}
    }
\end{table*}

        \begin{proof}
            Consider a scene represented by $\numPrim$ \threeD Gaussian primitives. 
            Under perspective projection, these $\numPrim$ \threeD Gaussian primitives yield exactly $\numPrim$ projected \twoD Gaussian primitives. 
            However,  any specific camera view $\view$ uses at most $\numPrim$ \twoD Gaussians. 

            Since each individual GT view satisfies \cref{lem:single_view_mse}, the best-approximation \mse for each individual view, $\risk_{\numPrim,\view}$, is lower-bounded by 
            \begin{align}
                \risk_{\numPrim,\view} \ge \frac{\constant^2}{(\numPrim\log \numPrim)^{\sobolev}} \label{eq:mseview_lower}.
            \end{align}

            Note that this assumption is realistic since edge discontinuities are quite common in natural images \cite{devore1992image}.
            Also, cameras orbit a \threeD scene in real multi-view captures, and no training camera is pointing at a synthetic, noiseless mathematical plane where MSE is literally $0$.

            The overall error $\mseMean$ is the arithmetic mean of the \mse across all $\numViews$ views, and is given by
            \begin{align}
                \mseMean &= \frac{1}{\numViews} \sum_{\view=1}^\numViews \risk_{\numPrim,\view}.
            \end{align}
            It inherits the lower bound of \cref{eq:mseview_lower}:
            \begin{align}
                \mseMean \ge \frac{\constant^2}{(\numPrim\log \numPrim)^{\sobolev}} \label{eq:meanmse_lower}.
            \end{align}

            For multi-view evaluation, we compute the average \psnr of a \threeD scene derived from the arithmetic mean of the \mse across all views. 
            This aligns with the default reduction behavior in the \texttt{torchmetrics}\footnote{\url{https://github.com/Lightning-AI/torchmetrics/blob/6c551c6469dc04fe92291c000ce2607860df92d2/src/torchmetrics/functional/image/psnr.py\#L54}} library.
            Hence, for normalized images, the aggregate \psnr is $- 10 \log_{10} {\mseMean} $. 
            Guaranteeing an aggregate \psnr of at least $\psnrTarget$ imposes an upper bound on the average error: 
            \begin{align}
                \mseMean &\le 10^{-\psnrTarget/10}.
            \end{align}

            Combining this upper bound with the theoretical lower bound of \cref{eq:meanmse_lower} yields:
            \begin{align}
                10^{-\psnrTarget/10} &\ge \frac{\constant^2}{(\numPrim\log \numPrim)^{\sobolev}} \nonumber \\
                \implies (\numPrim\log \numPrim)^{\sobolev} &\ge \constant^2 10^{\psnrTarget/10} \nonumber \\
                \implies \numPrim\log \numPrim &\ge \left[ \constant^2 10^{\frac{\psnrTarget}{10}} \right]^{\frac{1}{\sobolev}}
            \end{align}

            To solve for number of Gaussians $\numPrim$, we use the Lambert $W$ function, which is defined as the inverse function of $f(w)=w\exp(w)$. 
            The Lambert $W$ function is the multi-valued inverse. 
            We rewrite the inequality in a form that allows us to apply this function.
            Let $u=\log\numPrim>0$, so that $\numPrim=\exp(u)$. Then,
            \begin{align}
                u \exp(u) &\ge \left[ \constant^2 10^{\frac{\psnrTarget}{10}} \right]^{\frac{1}{\sobolev}}.
            \end{align}

            The function $u\exp(u)$ is increasing for $u\ge0$, and we use the principal branch $\lambert_0$ for nonnegative arguments.
            Consequently,
            \begin{align}
                u
                &\ge\lambert_0\left(
                    \left[ \constant^2 10^{\frac{\psnrTarget}{10}} \right]^{\frac{1}{\sobolev}}
                \right).
            \end{align}
            Exponentiating both sides results in:
            \begin{align}
                \numPrim &\ge\exp\left[
                    \lambert_0\left(
                        \left[ \constant^2 10^{\frac{\psnrTarget}{10}} \right]^{\frac{1}{\sobolev}}
                    \right)
                \right],
            \end{align}
            establishing the required exponential scaling of the \threeD Gaussian count $\numPrim$ with respect to the target multi-view \psnr $\psnrTarget$.
        \end{proof}

        %============================================================================
        \noIndentHeading{Remarks.}
        \begin{itemize}
            \item \textbf{Scaling Bottleneck:} By rearranging \cref{th:num_gaussians},
            \begin{align}
                \numPrim &\ge \exp\left[ \lambert_0\left( \left[ \constant^2 10^{\frac{\psnrTarget}{10}} \right]^{\frac{1}{\sobolev}} \right) \right] \nonumber\\
                \implies \psnrTarget &\le \sobolev~ 10 \log_{10} (\numPrim \log \numPrim) - 10 \log_{10} (\constant^2),
            \end{align}
            it is straightforward to see that training \psnr scales logarithmically with respect to $\numPrim$ (equivalently, primitive count scales exponentially with target PSNR). \cref{fig:mcmc_scale_gaussians_trainval} visualizes this relationship, showing that empirical training \psnr follows the expected log-linear trajectory before encountering the optimization issues beyond $10$M Gaussians.
            Also, DeVore et al. \cite[Section I, page 721]{devore1992image} show that natural scenes with sharp boundaries inhabit the fractional smoothness regime ($\sobolev \in (0, 0.5)$).
            Hence, 
            \begin{align}
                \psnrTarget &< 0.5\times 10 \log_{10} (\numPrim \log \numPrim) - 10 \log_{10} (\constant^2)
            \end{align}
            \item \textbf{Representational Capacity and Optimization:} 
            \cref{th:num_gaussians} bounds the representational capacity of discrete \threeD Gaussian primitives under an additive Gaussian dictionary approximation. 
            By characterizing the best-approximation error across target views, it establishes a theoretical ceiling on achievable training \psnr assuming a perfect solver capable of finding the global optimum for the $\numPrim$ parameters. 
            In practice, gradient-based optimization in \gsOnly could converge to sub-optimal local minima, causing empirical training \psnr to fall short of this ceiling. 
            Finally, bounding validation \psnr requires separate analysis to account for the generalization gap, which we leave for future work.
        \end{itemize}

    %============================================================================
    %============================================================================
    \subsection{Monotone Risk Reduction with \qnn}\label{sec:supp_qnn_theory}

        We next show that adding \neighborhood and coordinate queries provides the best possible risk as stated in \cref{th:monotone_feat}.

        %============================================================================
        \noIndentHeading{Setup.}
            Let $\queryInput$ be an input space (either $\realDomain^{\inDim}$ or a discrete integer lattice) equipped with a probability distribution $P_{\queryInput}$, and let $\coordinates \sim P_{\queryInput}$.
            Let $\nFieldGTScalar : \queryInput \to [0, 1]$ be a deterministic, measurable target function.
            A deterministic, measurable feature map $\parametrize : \queryInput \to \mathcal{Z}$ maps $\queryInput$ to a feature space $\mathcal{Z}$, inducing a random vector $\vecZ = \parametrize(\coordinates)$ distributed according to $P_{\vecZ}$.
            The hypothesis class $\mathcal{H}(\parametrize)$ consists of all deterministic, measurable functions $h : \mathcal{Z} \to [0, 1]$.
            The risk of a hypothesis $h \in \mathcal{H}(\parametrize)$ under expected squared (\mse) loss is
            \begin{align*}
                \mathcal{R}(h, \parametrize, \nFieldGTScalar) = \mathop{\expect}_{P_{\queryInput}} \left[ \left(\nFieldGTScalar(\coordinates) - h(\parametrize(\coordinates))\right)^2 \right],
            \end{align*}
            and the optimal achievable risk for $\parametrize$ is $\riskOptimal(\parametrize, \nFieldGTScalar) = \min_{h \in \mathcal{H}(\parametrize)} \mathcal{R}(h, \parametrize, \nFieldGTScalar)$.

        %============================================================================
        \begin{lemma}[\textbf{Best achievable risk}]\label{lem:lower-risk}
            The lowest achievable risk for a feature map $\parametrize$ is
            \begin{align*}
                \riskOptimal(\parametrize, \nFieldGTScalar) = \mathop{\expect}_{P_{\vecZ}}\left[\mathbb{V}_{P_{\coordinates \mid \vecZ}}\left(\nFieldGTScalar(\coordinates) \mid \vecZ\right)\right].
            \end{align*}
        \end{lemma}
        \begin{proof}
            Under squared loss, since $\nFieldGTScalar$ is deterministic, the Bayes-optimal predictor is the conditional expectation \cite[Eq. (3.36)]{bishop2006pattern}:
            \begin{align*}
                h^*(\vecZ) = \mathop{\expect}_{P_{\coordinates \mid \vecZ}}\left[\nFieldGTScalar(\coordinates) \mid \vecZ\right].
            \end{align*}
            Because $\nFieldGTScalar$ and $\parametrize$ are deterministic and measurable, $\mathop{\expect}_{P_{\coordinates \mid \vecZ}}[\nFieldGTScalar(\coordinates) \mid \vecZ]$ exists, is a deterministic measurable function of $\vecZ$, and satisfies $h^*(\vecZ) \in [0, 1]$ since $\nFieldGTScalar(\coordinates) \in [0, 1]$.
            Hence, $h^* \in \mathcal{H}(\parametrize)$.
            By the law of total expectation, the minimal risk is
            \begin{align*}
                \riskOptimal(\parametrize, \nFieldGTScalar)
                &= \mathcal{R}(h^*, \parametrize, \nFieldGTScalar) \\
                &= \mathop{\expect}_{P_{\vecZ}}\left[\mathop{\expect}_{P_{\coordinates \mid \vecZ}}\left[\left(\nFieldGTScalar(\coordinates) - h^*(\vecZ)\right)^2 \;\middle|\; \vecZ\right]\right] \\
                &= \mathop{\expect}_{P_{\vecZ}}\left[\mathbb{V}_{P_{\coordinates \mid \vecZ}}\left(\nFieldGTScalar(\coordinates) \mid \vecZ\right)\right],
            \end{align*}
            where the last equality follows from the definitions of $h^*$ and conditional variance.
        \end{proof}

        %============================================================================
        \begin{lemma}[\textbf{Feature augmentation cannot worsen the best achievable risk}]\label{lem:feature_aug}
            Consider two feature maps $\parametrize_i$ and $\parametrize_j$ such that $\parametrize_j(\coordinates) = (\parametrize_i(\coordinates), \tilde{\parametrize}(\coordinates))$, where $\tilde{\parametrize}$ provides additional deterministic, measurable features. Then
            \begin{align*}
                \riskOptimal(\parametrize_i, \nFieldGTScalar) \geq \riskOptimal(\parametrize_j, \nFieldGTScalar).
            \end{align*}
        \end{lemma}
        \begin{proof}
            Let $F = \nFieldGTScalar(\coordinates)$, $\vecZ_i = \parametrize_i(\coordinates)$, $\tilde{\vecZ} = \tilde{\parametrize}(\coordinates)$, and $\vecZ_j = \parametrize_j(\coordinates) = (\vecZ_i, \tilde{\vecZ})$.
            The projection $\pi(\bm{u}, \bm{v}) = \bm{u}$ recovers $\vecZ_i = \pi(\vecZ_j)$ deterministically from $\vecZ_j$, so $F \perp \vecZ_i \mid \vecZ_j$ (conditioning on $\vecZ_j$ already determines $\vecZ_i$).
            Applying the data processing inequality for Bayes risk \cite{xu2022minimum} yields
            \begin{align*}
                \riskOptimal(\parametrize_i, \nFieldGTScalar) \geq \riskOptimal(\parametrize_j, \nFieldGTScalar).
            \end{align*}
        \end{proof}

        %============================================================================
        \begin{lemma}[\textbf{Including coordinates $\coordinates$ in the features yields zero optimal risk}]\label{lem:best_risk_using_x}
            Let $\vecZ = \parametrize(\coordinates) = (\coordinates, \tilde{\parametrize}(\coordinates))$, where $\tilde{\parametrize}$ is any deterministic, measurable feature map. Then
            \begin{align*}
                \riskOptimal(\parametrize, \nFieldGTScalar) = 0.
            \end{align*}
        \end{lemma}
        \begin{proof}
            By \cref{lem:lower-risk}, $\riskOptimal(\parametrize, \nFieldGTScalar) = \mathop{\expect}_{P_{\vecZ}}[\mathbb{V}_{P_{\coordinates \mid \vecZ}}(\nFieldGTScalar(\coordinates) \mid \vecZ)]$.
            Since $\vecZ$ contains $\coordinates$, conditioning on $\vecZ$ determines $\coordinates$ exactly. Because $\nFieldGTScalar$ is a deterministic function of $\coordinates$, we have $\mathbb{V}_{P_{\coordinates \mid \vecZ}}(\nFieldGTScalar(\coordinates) \mid \vecZ) = 0$, and therefore
            \begin{align*}
                \riskOptimal(\parametrize, \nFieldGTScalar) = \mathop{\expect}_{P_{\vecZ}}\left[\mathbb{V}_{P_{\coordinates \mid \vecZ}}\left(\nFieldGTScalar(\coordinates) \mid \vecZ\right)\right] = \mathop{\expect}_{P_{\vecZ}}[0] = 0.
            \end{align*}
        \end{proof}

        %============================================================================
        \noIndentHeading{Proof of \cref{th:monotone_feat}.}
        \begin{proof}
            For a fixed $\receptive > 0$, the nested feature maps are
            $\parametrize_1(\coordinates) = (\nFieldGTLPScalar(\coordinates))$,
            $\parametrize_2(\coordinates) = (\parametrize_1(\coordinates - \receptive), \parametrize_1(\coordinates), \parametrize_1(\coordinates + \receptive))$, and
            $\parametrize_3(\coordinates) = (\parametrize_2(\coordinates), \coordinates - \receptive, \coordinates, \coordinates + \receptive)$.

            Applying \cref{lem:feature_aug} twice, first to $(\parametrize_1, \parametrize_2)$ and then to $(\parametrize_2, \parametrize_3)$, gives
            \begin{align*}
                \riskOptimal(\parametrize_1, \nFieldGTScalar) \geq \riskOptimal(\parametrize_2, \nFieldGTScalar) \geq \riskOptimal(\parametrize_3, \nFieldGTScalar).
            \end{align*}
            Since $\parametrize_3(\coordinates)$ includes $\coordinates$, applying \cref{lem:best_risk_using_x} with $\parametrize = \parametrize_3$ yields $\riskOptimal(\parametrize_3, \nFieldGTScalar) = 0$.
            Combining these results proves the theorem:
            \begin{align*}
                \riskOptimal(\parametrize_1, \nFieldGTScalar) \geq \riskOptimal(\parametrize_2, \nFieldGTScalar) \geq \riskOptimal(\parametrize_3, \nFieldGTScalar) = 0.
            \end{align*}
        \end{proof}

        %============================================================================
        \noIndentHeading{Remark.}
            \cref{th:monotone_feat} bounds representational capacity under infinite data and compute.
            In practice, gradient-based optimization may converge to sub-optimal local minima, causing empirical training risk to fall short of this ceiling.

\begin{table*}[!t]
    \caption{\textbf{Near-far Planes and Initialization} details for \nvs datasets.
    }
    \label{tab:nvs_dataset}
    \centering
    \scalebox{0.8}{
    \setlength\tabcolsep{0.2cm}
    \begin{tabular}{l m c m c m c m c }
        & \textbf{\mipNerf  \cite{barron2022mipnerf360}, \tAndT \cite{knapitsch2017tanks},} & \multirow{2}{*}{\textbf{\ommo \cite{lu2023large}}} & \multirow{2}{*}{\textbf{\shelly \cite{wang2023adaptiveshells}}} & \multirow{2}{*}{\textbf{\synNerf \cite{mildenhall2020nerf}}}\\
        & \textbf{\deepBlend \cite{hedman2018deep}} & & & \\
        \myTopRule
        Near Plane & $0.01$  & $0.01$ & $0.00$ & $2.00$\\
        Far Plane & $10^{10}$ & $10^{10}$ & $10^{10}$ & $6.00$\\
        Normalize Space & \cmark  & \cmark & \xmark & \xmark \\
        Alpha Mask & \xmark & \xmark & \cmark & \cmark \\
        Initialization & \sfm & \sfm & Random & Random \\
        Random Initial Extent & $3.00$ & $3.00$ & $1.50$ & $0.50$ \\
        \hline
        \#Primitives for \mcmc \cite{kheradmand20243d}  & $4$M & $4$M & Same as \gs & $1$M\\
        \#Primitives for \dbs  \cite{liu2025deformable} & Scene-specific & $4$M & $300$K & $1$M\\
        \#Primitives for \sv   \cite{di2026spherical}   & Scene-specific & $4$M & $300$K & Scene-specific\\
        \#Primitives for \nht   \cite{condor2026nht}    & Scene-specific & $4$M & $300$K & $300$K\\
        % \#Primitives for \powerfoam\cite{govindarajan2026power} & Scene-specific & $1.2$M~~~ & $300$K & $300$K\\
    \end{tabular}
    }
\end{table*}
\newpage

%============================================================================
%============================================================================
%============================================================================
\section{Implementation Details}\label{sec:app_imp_details}

    We now provide some additional implementation details.

    %============================================================================
    %============================================================================
    \subsection{3D \nvs}\label{sec:supp_impl_nvs}

        %============================================================================
        \noIndentHeading{Datasets.}
            Our \nvs experiments use six diverse datasets with a total of $35$ scenes: \mipNerf \cite{barron2022mipnerf360}, \tAndT \cite{knapitsch2017tanks}, \deepBlend \cite{hedman2018deep}, \ommo \cite{lu2023large}, \shelly \cite{wang2023adaptiveshells} and \synNerf (Blender) \cite{mildenhall2020nerf}.
            This selection includes a mix of synthetic, real, bounded indoor, and unbounded outdoor scenes to ensure a comprehensive evaluation.
            Specifically, we use all nine scenes from \mipNerf, the same two scenes each from \tAndT and \deepBlend, and all eight scenes from \synNerf, consistent with the \gs \cite{kerbl2023gaussians}.
            We also test on all eight scenes from the \ommo dataset, which provides scenes with distant objects \cite{kheradmand20243d} and all six scenes from the \shelly dataset, known for its fuzzy surfaces and complex geometries \cite{sharma2024volumetric}.

        %============================================================================
        \noIndentHeading{Data Splits.}
            We take every 8th image for test and remaining as train for \mipNerf \cite{barron2022mipnerf360}, \tAndT \cite{knapitsch2017tanks}, \deepBlend \cite{hedman2018deep} and \ommo \cite{lu2023large} datasets, as in \cite{barron2022mipnerf360}.
            We use the provided train/test JSONs for the \shelly \cite{wang2023adaptiveshells} and \synNerf \cite{mildenhall2020nerf} datasets.

        %============================================================================
        \noIndentHeading{Evaluation Metrics.}
            We use Peak Signal-to-Noise Ratio (\psnr), Structural Similarity Index Metric (\ssim) \cite{wang2004image} and VGG-16 based  Learned Perceptual Image Patch Similarity (\lpips) \cite{zhang2018unreasonable} metrics as in \gs \cite{kerbl2023gaussians}.
            We then average the scores across all scenes of the dataset to report a single value for each dataset.

        %============================================================================
        \noIndentHeading{Data Preprocessing.}
            For data preprocessing, we follow established practices:
            \begin{itemize}
                \item \textit{Single-scale training and single-scale testing experiments} of \cref{tab:nvs_results} and elsewhere: 
                We downsample indoor \mipNerf \cite{barron2022mipnerf360} scenes by a factor of two and outdoor scenes by four \cite{kerbl2023gaussians}. 
                We downsample the \ommo \cite{lu2023large} scene \#01 images by a factor of four as in \cite{kheradmand20243d}, while all other scenes and datasets use their original image resolutions.
                \item \textit{Single-scale training and multi-scale testing experiments} of \cref{tab:nvs_mip_stmt}: 
                we train on all \mipNerf \cite{barron2022mipnerf360} scenes downsampled by a factor of eight and render at successively higher resolutions ($1 \times$, $2 \times$, $4 \times$, and $8 \times$) following \mipSplat \cite{yu2024mip}.
            \end{itemize}

        %============================================================================
        \noIndentHeading{Near-far Planes and Initialization  Details.}
            We next provide dataset configurations such as near/far plane settings as well as initialization details for the datasets in \cref{tab:nvs_dataset}.
            The datasets are grouped based on their near-far planes and initialization parameters. 
            Most datasets such as \mipNerf \cite{barron2022mipnerf360}, \tAndT \cite{knapitsch2017tanks}, \deepBlend \cite{hedman2018deep} and \ommo \cite{lu2023large} utilize the default near/far planes and are initialized using standard Structure-from-Motion (\sfm) data. 
            The \shelly \cite{wang2023adaptiveshells} and \synNerf \cite{mildenhall2020nerf} datasets adopt custom near and far plane settings from QFields \cite{sharma2024volumetric} and \nerf \cite{mildenhall2020nerf} respectively, and are initialized randomly due to the absence of \sfm points.

\begin{figure*}[!t]
\begin{lstlisting}[style=qnnpython, basicstyle=\ttfamily\footnotesize\color{Black}, caption={PyTorch implementation of \qnn residual refinement with 2D queries.}, label={lst:qnn_minimal}]
import torch
import torch.nn as nn

class CCNNResidual(nn.Module):
    """CCNN residual refiner with 2D queries."""

    def __init__(self, hidden_dim: int = 64, num_layers: int = 4):
        super().__init__()
        layers = []
        for i in range(num_layers):
            in_ch  = 5 if i == 0 else hidden_dim
            out_ch = 3 if i == num_layers - 1 else hidden_dim
            layers.append(nn.Conv2d(in_ch, out_ch, kernel_size=3, padding=1))
            if i < num_layers - 1:
                layers.append(nn.ReLU(inplace=True))
        self.net = nn.Sequential(*layers)

    def forward(self, rgb: torch.Tensor) -> torch.Tensor:
        b, h, w, _ = rgb.shape                              # rgb: [B, H, W, 3]
        u = torch.linspace(0.0, 1.0, w, device=rgb.device, dtype=rgb.dtype)
        v = torch.linspace(0.0, 1.0, h, device=rgb.device, dtype=rgb.dtype)
        qv, qu = torch.meshgrid(v, u, indexing="ij")
        q = torch.stack([qu, qv], dim=-1)[None].expand(b, h, w, 2)

        x = torch.cat([rgb, q], dim=-1).permute(0, 3, 1, 2)  # x: [B, 5, H, W]

        return rgb + self.net(x).permute(0, 2, 3, 1)
\end{lstlisting}
\end{figure*}

        %============================================================================
        \noIndentHeading{\gsOnly{}+\qnn Architecture.}
            \cref{fig:gs_qnn} shows the architecture of combining Gaussian-Splatting with \qnn to enhance image fidelity. 
            We integrate \qnn into \gsOnly as follows:
            \begin{itemize}
                \item The \gsOnly renders the initial \lf image $\nFieldPdLP(\query)$.
                \item The \qnn concatenates the query $\query$ and the \lf splatted image $\nFieldPdLP(\query)$.
                \item The \qnn passes the concatenation through a \cnn to produce a residual image $\cnn(\query\concat \nFieldPdLP(\query))$. 
                \item We then add the \lf splatted image $\nFieldPdLP$ to the residual image  $\cnn(\query\concat\nFieldPdLP(\query))$ and produce the final predicted novel view image \\$\nFieldGSQNN(\query)=\nFieldPdLP(\query) + \cnn(\query\concat\nFieldPdLP(\query))$.
            \end{itemize}

            The final prediction formulation depends on whether the dataset includes an alpha mask for opacity.
            Note that the final \gsOnly{}+\qnn model uses \twoD coordinates as queries. So, we substitute $\query=\coordinates$ in the following paragraphs.

            \begin{itemize}
                \item \textit{Dataset without an Alpha Mask}: 
                This applies to datasets such as \mipNerf, \tAndT, \deepBlend and \ommo datasets. 
                The outputs of \gsOnly model $\nFieldGS(\coordinates)$, \gsOnly{}+\qnn model $\nFieldGSQNN(\coordinates)$ and the GT $\nField^{GT}(\coordinates)$ are as follows:
                \begin{align}
                    \nFieldGS(\coordinates) &= \nFieldPdLP(\coordinates)\\
                    \nFieldGSQNN(\coordinates) &= \nFieldPdLP(\coordinates) + \cnn(\coordinates \concat \nFieldPdLP(\coordinates)) \\
                    \nFieldGT^{GT}(\coordinates) &= \nFieldGT(\coordinates) 
                \end{align}
                \item \textit{Dataset with an Alpha Mask}: 
                For \shelly \cite{wang2023adaptiveshells} and \synNerf \cite{mildenhall2020nerf} datasets, which provide an alpha mask, we follow the \nerfstudio \cite{tancik2023nerfstudio} convention where the predicted or rendered RGB is considered pre-multiplied by the predicted opacity $\opacityPd$ and so we do not multiply the first term by predicted opacity $\opacityPd$.
                We multiply the GT RGB with GT opacity $\opacityGT$. 
                In other words, the output of \gsOnly model $\nFieldGS(\coordinates)$, \gsOnly{}+\qnn model $\nFieldGSQNN(\coordinates)$ and the GT $\nField^{GT}(\coordinates)$ are as follows:
                \begin{align}
                    \nFieldGS(\coordinates)      &= ~~~\nFieldPdLP(\coordinates) + (1 - \opacityPd(\coordinates))*\background\\
                    \nFieldGSQNN(\coordinates)   &= ~~~\nFieldPdLP(\coordinates) \nonumber\\
                                                 &~ \quad + \cnn([\opacityPd(\coordinates){*}\coordinates] \concat \nFieldPdLP(\coordinates)) \nonumber \\ 
                                                 &~ \quad + (1 - \opacityPd(\coordinates))*\background\\
                    \nFieldGT^{GT}(\coordinates) &= \opacityGT(\coordinates) \nFieldGT(\coordinates) ~+ (1 - \opacityGT(\coordinates))*\background,
                \end{align}
                with $\background$ being the background color, which is set to white $(1.0, 1.0, 1.0)$ for all pixels during both training and inference.
            \end{itemize}

        %============================================================================
        \noIndentHeading{Implementation.}
            \gs \cite{kerbl2023gaussians} and \mcmc \cite{kheradmand20243d} implementations use the \textit{gsplat} library\footnote{\url{https://github.com/nerfstudio-project/gsplat}} \cite{ye2024gsplat}.
            \dbs \cite{liu2025deformable}, \sv \cite{di2026spherical}, and \nht \cite{condor2026nht} follow their official released repositories\footnote{\url{https://github.com/RongLiu-Leo/beta-splatting/}}\textsuperscript{,}\footnote{\url{https://github.com/sphericalvoronoi/sphericalvoronoi}}\textsuperscript{,}\footnote{\url{https://github.com/nv-tlabs/neural-harmonic-textures}} using their highest-performing $\beta$-splat configurations.
            Our \qnn implementation uses PyTorch \cite{paszke2019pytorch}. 
            Traditional \mlps process coordinate queries with fully-connected layers to output a three-channel RGB residual image.
            In contrast, \qnn leverages \twoD convolutions to process both the \twoD coordinate queries and the \lf splatted image, and outputs a three-channel RGB residual image.

            All \mlps use $4$ linear layers with $256$ dimensions and \relus. 
            To ensure a fair comparison with \mlp, we implement \qnn with reduced hidden channels by a factor of $4$ (yielding fewer parameters than the \mlp).
            Thus, \qnn uses $4$ \twoD convolutional layers, a $3{\times}3$ kernel size, $64$ channels, \relu activations and utilizes vanilla encodings. 
            \cref{lst:qnn_minimal} shows the minimal PyTorch implementation.
            We experiment with other encodings for \qnn in the \cref{sec:ablation}: \textit{Encodings} paragraph.

        %============================================================================
        \noIndentHeading{Loss.}
            We compare the prediction from the \gsOnly{}+\qnn model $\nFieldGSQNN(\coordinates)$ against the GT image $\nFieldGT^{GT}(\coordinates)$ with the same loss as the corresponding baseline.

        %============================================================================
        \noIndentHeading{Optimization.}
            For \qnn models, we train the \qnn parameters with AdamW optimizer with a learning rate of $1e{-}4$.
            We train with all architectures in an end-to-end learning framework and apply \qnn after $75$\% of total iterations (\forExample: after $22{,}500$ iterations for total of $30{,}000$ iterations) without changing the total iterations.

    %============================================================================
    %============================================================================
    \subsection{1D Regression}\label{sec:supp_impl_1dreg}

        %============================================================================
        \noIndentHeading{Dataset.}
            Our experiments use the synthetic \oneD $1/\freqSym^\power$ signal, defined on the interval $[0, 1)$, with $\alpha=0.5$ \cite{tancik2020fourier}. 
            The authors generate this \oneD signal by sampling from a standard i.i.d. Gaussian vector of length $\numData$, scaling its $j$th entry by $1/j^\power$, and then taking the real component of its inverse Fourier transform. 
            To fully evaluate the networks' ability to learn \hf signals, we remove the signal's bandlimitedness.
            We then apply a low-pass filter with a normalized cutoff of $0.125$ to separate this signal into its \lf and residual components, and our task is regressing only the residual component.

        %============================================================================
        \noIndentHeading{Data Splits.}
            For training and testing, we randomly sample data points, using a signal length $\numData{=}32$ \cite{tancik2020fourier}. % and a batch size of $4$.

        %============================================================================
        \noIndentHeading{Evaluation Metrics.}
            We use the \psnr metric for evaluation.

        %============================================================================
        \noIndentHeading{Baselines.} 
            We use \mlp with vanilla and Fourier encodings \cite{tancik2020fourier} as our baselines.

        %============================================================================
        \noIndentHeading{Implementation.}
            The \mlp architecture consists of $4$ linear layers with $256$ dimensions and \relus \cite{tancik2020fourier} to output \oneD signal.
            The Fourier \mlps use the default setting of $256$-dimensional cosines and sines each randomly sampled from a zero-mean Gaussian distribution with standard deviation $10$, following \cite{tancik2020fourier}.
            The \mlp takes \oneD queries as input, and so we replace the linear layer in \mlp with \oneD convolution to construct \cnn and \qnn.
            Note that the \qnn takes both the \oneD coordinates and \lf component as inputs, while the \cnn takes only \lf component as input.
            We use \qnn with Vanilla encodings.
            To ensure a fair comparison, we reduce the hidden channels of \cnn and \qnn architectures by a factor of $\sqrt{3}$ to $147$ channels, which matches the parameter count with vanilla \mlp, as they use a \oneD convolution kernel size of $3$.
            Since all these networks predict the residual signal, we add the \lf component to this residual to obtain the final predicted signal, which we then use for \psnr metric calculation.
            Please refer to \cref{tab:exp_setting} for the experimental settings.

        %============================================================================
        \noIndentHeading{Loss.}
            We train with the $\lOne$ loss on the training samples to avoid data leakage.

        %============================================================================
        \noIndentHeading{Optimization.}
            We train all models with \adam optimizer with learning rate $1e{-}5$ for $500$ iterations.

    %============================================================================
    %============================================================================
    \subsection{2D Regression}\label{sec:supp_2dreg_impl}

        %============================================================================
        \noIndentHeading{Datasets.}
            We use the \mipNerf dataset \cite{barron2022mipnerf360} for this experiment, and the rendered output from the baseline \gs \cite{kerbl2023gaussians} as the \lowFid image input to the \qnn. 
            The \qnn learns to predict the residual error left by \gs with respect to the GT image.
            Specifically, the final prediction is the sum of \gs-splatted image (\lf baseline) plus the \qnn's predicted residual image. 

        %============================================================================
        \noIndentHeading{Data Splits.}
            We randomly sample training and validation pixels over images. % $6$ images per scene. 
            % As before, we only calculate the loss on the training pixels to avoid data leakage.
            
        %============================================================================
        \noIndentHeading{Evaluation Metrics.}
            We use the \psnr and \ssim \cite{wang2004image} metrics for evaluation as \cite{tancik2020fourier} over both validation and training pixels.
            We run over all scenes of the \mipNerf dataset \cite{barron2022mipnerf360} and average scores across scenes to report a single value for each model.

        %============================================================================
        \noIndentHeading{Baselines.}        
            We use the \mlp with vanilla encodings, Fourier encodings \cite{tancik2020fourier} and \hashGrids \cite{muller2022instant} as our baselines.

        %============================================================================
        \noIndentHeading{Implementation.}
            The \mlp architecture has $4$ linear layers, $256$ dimensions and \relus to output three-dimensional RGB signal.
            The Fourier \mlps use the default setting of $256$-dimensional cosines and sines each randomly sampled from the standard deviation of $10$ as \cite{tancik2020fourier}. 
            The \mlps take \twoD coordinate queries as input, and so, we replace the linear layer in \mlp with \twoD convolution and append \lf image to construct \qnn.
            We use \qnn with both Vanilla and Fourier encodings.
            To ensure a fair comparison, we reduce the hidden channels of \qnn architectures by a factor of $3$  to $85$ channels, which matches the parameter count with vanilla \mlp, as \qnn uses a \twoD convolution kernel size of $3{\times}3$.
            Since all these networks predict the residual output, we add the \lf component to these outputs to obtain the final predicted outputs, which we then use for metric calculation.
            Please refer to \cref{tab:exp_setting} for the experimental settings.

        %============================================================================
        \noIndentHeading{Loss.}
            We calculate the $\lOne$ loss between the final prediction and the GT \mipNerf image.
            Note that we calculate the $\lOne$ loss on the training samples to avoid data leakage.

        %============================================================================
        \noIndentHeading{Optimization.}
            We train all models with \adam optimizer with learning rate $1e{-}3$ for $2{,}000$ iterations as \cite{tancik2020fourier}.

    %============================================================================
    %============================================================================
    \subsection{2D Image \sr}\label{sec:supp_sr_impl}

        Our experimental setup encompasses two distinct image \sr scenarios:
        \begin{itemize}
            \item Face \sr: This sub-task involves $8\times$ upscaling of face images from $16{\times}16$ to $128{\times}128$ resolution. 
            \item Natural \sr: This sub-task involves $4\times$ upscaling  of natural images from  $64{\times}64$ to $256{\times}256$ \cite{wang2021realesrgan}.
        \end{itemize}

        %============================================================================
        \noIndentHeading{Datasets.}
            The face \sr experiments use \ffhqFull dataset \cite{karras2019style} for training and \celebAHQ for evaluation as \cite{saharia2022image}.
            The natural \sr experiments use \divTwoK \cite{agustsson2017ntire}, \flikr \cite{timofte2017ntire} and \ost \cite{wang2018recovering} datasets for training and \divTwoK validation images for evaluation following \cite{wang2021realesrgan}.

        %============================================================================
        \noIndentHeading{Data Splits.}
            For face \sr, we leverage the \ffhq dataset \cite{karras2019style}, utilizing all its $70{,}000$ images for training, and test on $30{,}000$ \celebAHQ dataset \cite{karras2018progressive} as in \cite{saharia2022image}. 
            For natural \sr, our training follows \realEsrgan \cite{wang2021realesrgan}, comprising $13{,}774$ \hr images, with $800$ train images from \divTwoK \cite{agustsson2017ntire}, $2{,}650$ images from \flikr \cite{timofte2017ntire} and $10{,}324$ images from the \ost dataset \cite{wang2018recovering}.
            We evaluate on the paired $100$ \divTwoK validation images.

        %============================================================================
        \noIndentHeading{Evaluation Metrics.}
            We use the widely accepted \psnr and \ssim \cite{wang2004image} metrics \cite{saharia2022image} for evaluation.

        %============================================================================
        \noIndentHeading{Baselines.}  
            To thoroughly assess the impact of our proposed \qnn, we integrate it into two prominent architectural paradigms: diffusion models and \gan{}s. 
            We use \srThree \cite{saharia2022image}, a \ddpm-based \sr method and \realEsrgan \cite{wang2021realesrgan}, a \gan-based \sr method as our baselines.

        %============================================================================
        \noIndentHeading{Implementation.}
            For face \sr, we use an unofficial implementation\footnote{\url{https://github.com/Janspiry/Image-Super-Resolution-via-Iterative-Refinement} \srThree does not release its official code. We do not report \srThree's natural \sr performance since this unofficial code does not reproduce natural \sr results.} 
            of \srThree \cite{saharia2022image}, a \ddpm-based \sr method. 
            For natural \sr, we leverage the official \realEsrgan \cite{wang2021realesrgan} codebase\footnote{\url{https://github.com/xinntao/Real-ESRGAN}}. 
            We integrate \qnn into both \srThree \cite{saharia2022image} and \realEsrgan \cite{wang2021realesrgan} methods.
            It is crucial to note that both the original \srThree and \realEsrgan models use \cnn architectures and take the \lr input.
            We realise \qnn by augmenting these existing \cnn models with \twoD coordinate queries.
            Note that \qnn has slightly more parameters ($0.0068\%$ more for \realEsrgan generator) than \cnn because input to the \qnn is a $5$-channel image ($3$ RGB + $2$ coordinates) compared to $3$-channel RGB image as input to the \cnn.
            For the \sr experiments, we directly predict the desired signal without adding the \lf signal since the backbones contain residual connections.

        %============================================================================
        \noIndentHeading{Loss and Optimization.}
            We train the \cnn baseline and \qnn-augmented models with the exact same loss functions, iterations and other hyper-parameters.

\newpage
\null
\begin{figure*}[!t]
    \centering
    \includegraphics[width=0.8\linewidth]{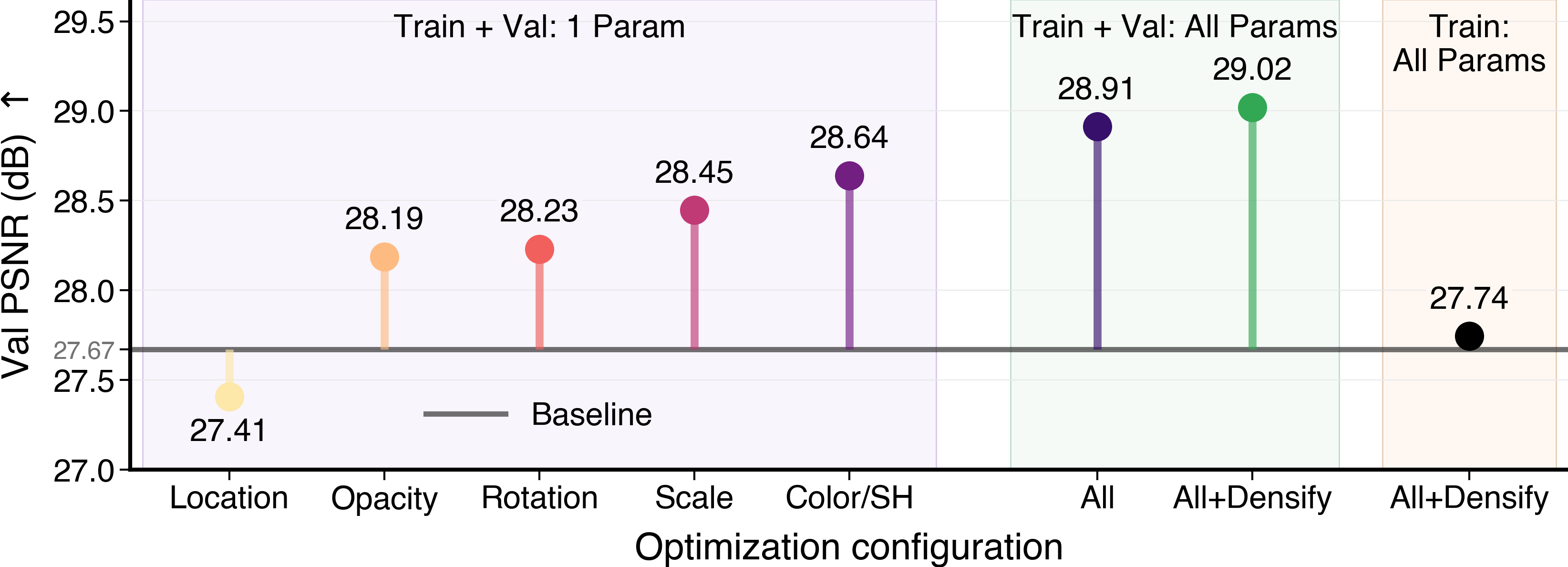}
    \caption{\textbf{Oracle Analysis of \gsOnly} with \gs \cite{kerbl2023gaussians} baseline on the \mipNerf dataset \cite{barron2022mipnerf360} finetuning one or all of the Gaussian primitive parameters.
    \textbf{\textcolor{sns_violet}{Color/\sh} parameters are the most important parameter} for improving the visual fidelity of \gsOnly model since finetuning color on the Train + \val images shows the maximum gain in \val \psnr ($28.64$\,\db) over the \gs baseline ($27.67$\,\db) trained on Train images.
    The figure also shows that densification (\textcolor{darkGreen3}{All+Densify}) also does not increase the \val \psnr significantly if we have finetuned all parameters (\textcolor{sns_blue}{All}) on the Train + \val images.
    Note that \textbf{longer training} only on the Train images (without the \val images) does not change the \val \psnr ($27.74$\,\db) over the \gs baseline ($27.67$\,\db) trained on Train images. 
    }
    \label{fig:oracle_gs}
    \vspace{-0.2cm}
\end{figure*}
\newpage

%============================================================================
%============================================================================
%============================================================================
\section{Experiments: 3D \nvs}

    We next provide additional details and results of the experiments evaluating \qnn{}’s performance.

    %============================================================================
    %============================================================================
    \subsection{Scaling Limits and Computational Efficiency}

        %============================================================================
        \noIndentHeading{Oracle Analysis of Gaussian Splatting.}
            To determine which parameters are most critical to the performance of \gsOnly, we conduct an ``oracle analysis'' with the \gs \cite{kerbl2023gaussians} baseline on the \mipNerf dataset \cite{barron2022mipnerf360}. 
            The \gs method represents a scene with a collection of \threeD Gaussian primitives, each defined by parameters for its location, rotation, scale, opacity, and color (represented by spherical harmonics or \sh).
            This analysis aims to reveal which parameters contribute most significantly to the \psnr performance. 
            The standard training procedure involves training the Gaussian representation on a set of training images (Train) for $30k$ iterations and reporting \psnr on validation images (\val).
            In our oracle study, we extend this and finetune the resulting model on the \textbf{combined} training and validation (Train + \val) images for an additional $10k$ iterations.
            To isolate the impact of each parameter, we finetune either a single parameter (e.g., only color) or all of them simultaneously.
            
            The results, presented in \cref{fig:oracle_gs}, show that the \textcolor{sns_violet}{Color/\sh} parameters are the most important parameter for improving the visual fidelity of \gsOnly model since finetuning color parameter on the Train + \val images shows the maximum gain in \val \psnr ($28.64$\,\db) over the \gs baseline ($27.67$\,\db) trained on Train images.
            This is why predicting colors in \gsOnly+ \qnn model (\cref{fig:gs_qnn}) improves the \psnr on the \mipNerf dataset \cite{barron2022mipnerf360}, while predicting opacity in \vdgs \cite{malarz2025gaussian} does not improve the \psnr significantly.
            The figure also shows that densification (\textcolor{darkGreen3}{All+Densify}) also does not increase the \val \psnr significantly if we have finetuned all parameters (\textcolor{sns_blue}{All}) on the Train + \val images.
            Note that \textbf{longer training} only on the Train images (without the \val images) does not change the \val \psnr ($27.74$\,\db) over the \gs baseline ($27.67$\,\db) trained on Train images. 

\begin{figure}[!t]
    \centering
    \includegraphics[width=0.85\linewidth]{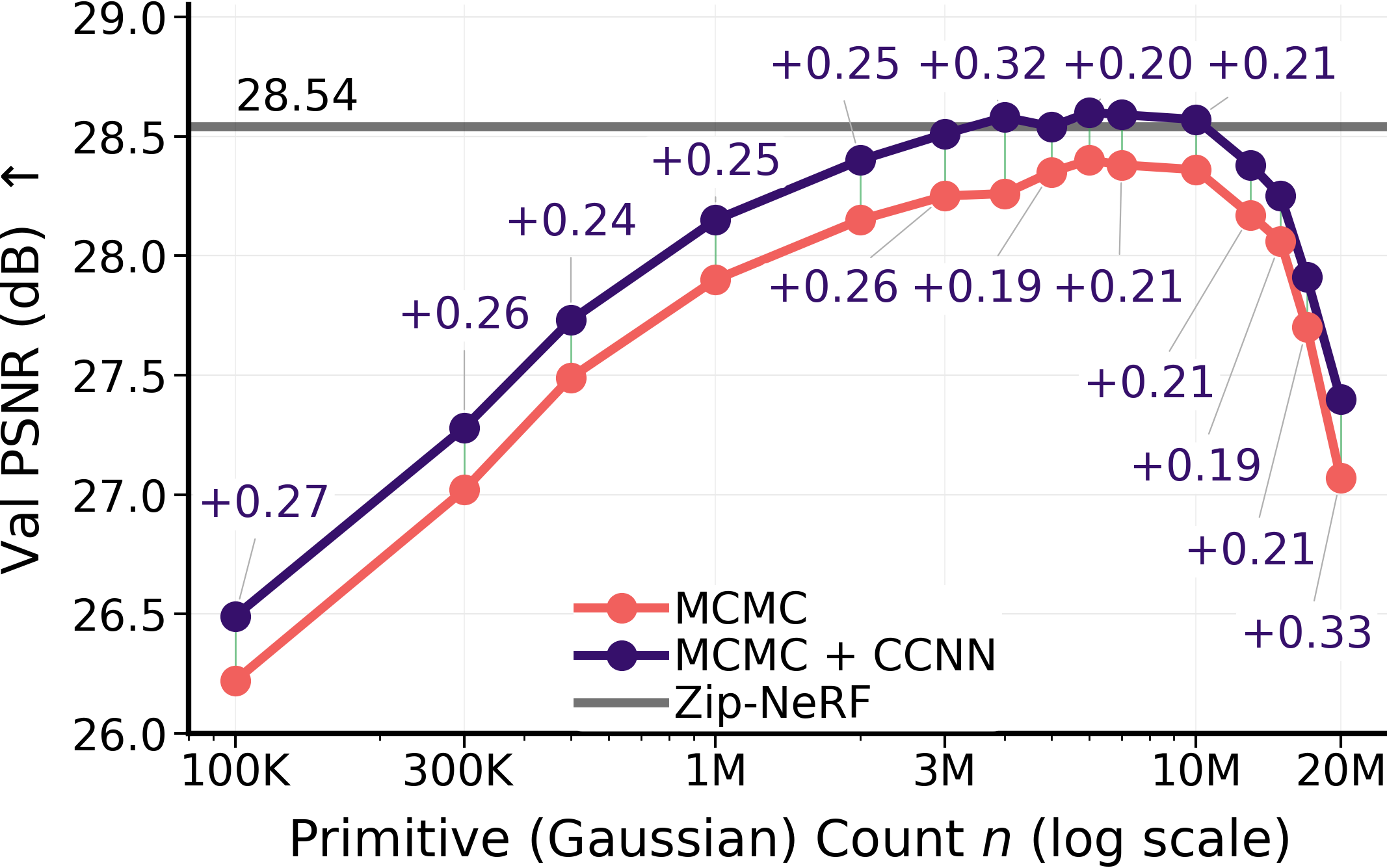}
    \caption{\textbf{Primitive Count.} 
    \mcmc \cite{kheradmand20243d} suffers from optimization difficulties beyond $4$M Gaussians on \mipNerf dataset.
    \textbf{Adding \qnn improves} \mcmc at all primitive (Gaussian) counts. 
    This is the zoomed validation version of \cref{fig:mcmc_scale_gaussians_trainval}.
    }
    \label{fig:mcmc_scale_gaussians}
\end{figure}

\begin{table*}[!t]
    \caption{
        \textbf{Generalizability of \qnn} to other \threeD representations.
        \textbf{Across every evaluated representation and dataset, equipping existing baselines with \qnn consistently improves rendering fidelity}.
        We would like to note that our results for \radMesh for \ommo ~\cite{mai2026radiance} are just a proof-of-concept and were not extensively optimized.
        We do not mark ties.
        [Key: \firstBText{Best}~in each pair of experiment, \retrained= Retrained].
    }
    \label{tab:nvs_results_other_rep}
    \centering
    \scalebox{0.85}{
    \setlength{\tabcolsep}{0.3cm}
    \begin{tabular}{l c c c c} 
        \myTopRule
        \multirow{2}{*}{Method}
        & \textbf{\mipNerf \cite{barron2022mipnerf360}} 
        & \textbf{\tAndT \cite{knapitsch2017tanks}} 
        & \textbf{\deepBlend \cite{hedman2018deep}} 
        & \textbf{\ommo \cite{lu2023large}} \\
        & \footnotesize{\psnr\!$\uparrowRHDSmallTwo$\sep\ssim\!$\uparrowRHDSmallTwo$\sep\lpips\!$\downarrowRHDSmallTwo$}
        & \footnotesize{\psnr\!$\uparrowRHDSmallTwo$\sep\ssim\!$\uparrowRHDSmallTwo$\sep\lpips\!$\downarrowRHDSmallTwo$}
        & \footnotesize{\psnr\!$\uparrowRHDSmallTwo$\sep\ssim\!$\uparrowRHDSmallTwo$\sep\lpips\!$\downarrowRHDSmallTwo$}
        & \footnotesize{\psnr\!$\uparrowRHDSmallTwo$\sep\ssim\!$\uparrowRHDSmallTwo$\sep\lpips\!$\downarrowRHDSmallTwo$} \\
        \myTopRule 
        \radMesh \retrained~\cite{mai2026radiance}         & $27.01$ \sep $0.82$ \sep $0.23$ & $23.16$ \sep $0.86$ \sep $0.16$ & $29.58$ \sep $0.90$ \sep $0.28$ & $24.79$ \sep $0.76$ \sep $0.35$\\
        \radMesh + \qnn                                    & \firstB{27.41} \sep $0.82$ \sep \firstB{0.22} & \firstB{23.78} \sep $0.86$ \sep $0.16$ & \firstB{29.63} \sep \firstB{0.91} \sep \firstB{0.27} & \firstB{25.26} \sep \firstB{0.78} \sep \firstB{0.32}\\
        \hline
        \powerfoam \retrained~\cite{govindarajan2026power} & $27.26$ \sep $0.78$ \sep $0.25$ & $22.33$ \sep $0.78$ \sep $0.25$ & $28.18$ \sep $0.87$ \sep $0.31$ & $27.59$ \sep $0.84$ \sep $0.21$\\
        \powerfoam + \qnn                                  & \firstB{27.65} \sep \firstB{0.79} \sep \firstB{0.23} & \firstB{23.35} \sep \firstB{0.80} \sep \firstB{0.23} & \firstB{28.41} \sep \firstB{0.88} \sep \firstB{0.30} & \firstB{28.25} \sep \firstB{0.86} \sep \firstB{0.18}\\
        \hline
        \twoDGs \retrained~\cite{huang20242dgs}            & $27.49$ \sep $0.82$ \sep $0.18$ & $22.84$ \sep $0.83$ \sep $0.15$ & $29.66$ \sep $0.90$ \sep $0.17$ & $25.97$ \sep $0.83$ \sep $0.23$\\
        \twoDGs + \qnn                                     & \firstB{27.75} \sep $0.82$ \sep $0.18$ & \firstB{23.48} \sep $0.83$ \sep $0.15$ & \firstB{29.77} \sep \firstB{0.91} \sep $0.17$ & \firstB{26.12} \sep $0.83$ \sep $0.23$\\
        \hline
        \gs \retrained~\cite{kerbl2023gaussians}            & $27.67$ \sep $0.82$ \sep $0.20$ & $23.52$ \sep $0.84$ \sep $0.18$ & $29.50$ \sep $0.90$ \sep \firstB{0.24} & $29.03$ \sep $0.90$ \sep $0.16$\\
        \gs + \qnn                                         & \firstB{27.96} \sep \firstB{0.83} \sep $0.20$ & \firstB{24.11} \sep \firstB{0.85} \sep \firstB{0.17} & \firstB{29.67} \sep \firstB{0.91} \sep $0.25$ & \firstB{29.37} \sep \firstB{0.91} \sep \firstB{0.15}\\
        \hline
        \nht \retrained~\cite{condor2026nht}               & $29.04$ \sep $0.85$ \sep $0.17$ & $25.72$ \sep $0.88$ \sep $0.12$ & $30.17$ \sep $0.91$ \sep $0.25$ & $31.50$ \sep $0.93$ \sep $0.12$\\
        \nht + \qnn                                        & \firstB{29.19} \sep $0.85$ \sep $0.17$ & \firstB{25.79} \sep $0.88$ \sep $0.12$ & \firstB{30.64} \sep \firstB{0.92} \sep \firstB{0.24} & \firstB{31.74} \sep $0.93$ \sep \firstB{0.11}\\
        \myTopRule 
    \end{tabular}
    }
\end{table*}

\begin{table*}[!t]
    \centering
    \captionof{table}{\textbf{Compute Overhead} of \qnn. 
    \qnn's parameter overhead is negligible when contrasted with the standard \gsOnly parameters.
    }
    \label{tab:compute_qnn}
    \scalebox{0.8}{
    \setlength\tabcolsep{0.15cm}
    \begin{tabular}{l m c c m c c m c}
        \textbf{Layer} & ~~~\textbf{In Tensor} & \textbf{Out Tensor} & \textbf{Weight Parameters} & \textbf{Biases} & \textbf{Matrix Multiplications} \\
        \myTopRule
        $1$ & $(\inChannels{+}\lowChannels){\times}H{\times}W$ & ~~~$\hidChannels{\times}H{\times}W$ & $(\inChannels{+}\lowChannels){\times}\hidChannels{\times}k{\times}k$ & $\hidChannels$ & $(\inChannels{+}\lowChannels){\times}\hidChannels{\times}H{\times}W{\times}k{\times}k$ \\
        $2$ to $\numLayer-1$ & ~~~~~~~~~~~$\hidChannels{\times}H{\times}W$ & $\quad\quad\hidChannels{\times}H{\times}W$  &  ~~~~~~~~~~~~$\hidChannels{\times}\hidChannels{\times}k{\times}k$ & $\hidChannels$ & ~~~~~~~~~~~~~~~~$\hidChannels{\times}\hidChannels{\times}H{\times}W{\times}k{\times}k$\\
        $\numLayer$ & ~~~~~~~~~~~$\hidChannels{\times}H{\times}W$ & $\outChannels{\times}H{\times}W$   & ~~~~~~~~~~$\outChannels{\times}\hidChannels{\times}k{\times}k$ & ~~$\outChannels$ & ~~~~~~~~~~~~~~$\outChannels{\times}\hidChannels{\times}H{\times}W{\times}k{\times}k$  \\
        \hline 
        \multirow{2}{*}{\textbf{Total}} & \multirow{2}{*}{\mathDash} & \multirow{2}{*}{\mathDash} & $[\inChannels{+}\lowChannels{+}(\numLayer-2)\hidChannels$ & ~$(\numLayer{-}1)\hidChannels$ & $[\inChannels{+}\lowChannels{+}(\numLayer-2)\hidChannels$ \\
        & & & ~~~~~~${+}\outChannels]{\times}\hidChannels{\times}k{\times}k$ & ${+}\outChannels$ & ~~~~~~~~~~${+}\outChannels]{\times}\hidChannels{\times}H{\times}W{\times}k{\times}k$ 
    \end{tabular}
    }
\end{table*}

        %============================================================================
        \noIndentHeading{Increasing Primitive Count.}
        \cref{fig:mcmc_scale_gaussians} shows the results with \mcmc \cite{kheradmand20243d} on the \mipNerf dataset \cite{barron2022mipnerf360}.
        This is the zoomed validation version of \cref{fig:mcmc_scale_gaussians_trainval}.
        The figure confirms that \qnn improves the \mcmc baseline at all primitive (Gaussian) counts.

        %============================================================================
        \noIndentHeading{Compute Overhead of \qnn.}
            We next report the computational overhead associated with integrating a \qnn consisting of $\numLayer$ layers onto a \lf image of size $H \times W$ in \cref{tab:compute_qnn}.
            In this formulation, $\inChannels$, $\lowChannels$, and $\outChannels$ represent the query dimensions, the channels of the \lf input, and the channels of the predicted image, respectively. 
            The \qnn's architecture is further defined by the number of hidden channels $\hidChannels$ and the kernel size $k$.
            For the \nvs task, the \qnn has the following parameters:
            \begin{itemize}
                \item Layers: $\numLayer{=}4$.
                \item Channels: $\inChannels{=}2$, $\lowChannels{=}3$, and $\outChannels{=}3$.
                \item Architecture: $\hidChannels{=}64$ and $k{=}3$.
            \end{itemize}
            Substituting these values into the expressions in \cref{tab:compute_qnn}, the \qnn introduces approximately $0.31$ MB (in \textit{float32}) of weight and bias parameters. 
            This parameter overhead is negligible when contrasted with the standard \gsOnly parameters, which typically exceed $700$ MB. 
            This small increase explains the Stor MB numbers of \cref{tab:nvs_edge_memory_inference}, where including \qnn results in a minimal increase in overall storage footprint.

    %============================================================================
    %============================================================================
    \subsection{Generalizability to Other 3D Representations}

        \threeD scene representations for \nvs span two orthogonal design axes: \emph{what} is represented (a set of discrete primitives \vs continuous fields) and \emph{how} it is rendered (rasterization \vs raytracing) \cite{mai2026radiance}, giving four regimes:
        \begin{itemize}
            \item Primitive rasterization (\gs \cite{kerbl2023gaussians}),
            \item Primitive raytracing (\nht \cite{condor2026nht} built on \threeDGut \cite{wu20253dgut}),
            \item Field rasterization (\radMesh \cite{mai2026radiance}), and
            \item Field raytracing (\nerf \cite{mildenhall2020nerf}).
        \end{itemize}
        In addition, \powerfoam \cite{govindarajan2026power} bridges these paradigms. 
        While \cref{tab:nvs_results} focuses on primitive rasterization (\gs \cite{kerbl2023gaussians}) and raytracing (\nht \cite{condor2026nht}), we test whether \qnn generalizes beyond standard \gsOnly by pairing it with \radMesh \cite{mai2026radiance}, \powerfoam \cite{govindarajan2026power}, and \twoDGs \cite{huang20242dgs} (another discrete primitive rasterizer) in \cref{tab:nvs_results_other_rep}. 
        Because \qnn cannot be used without structured \neighborhoods, it cannot be applied to isolated raytracing in \nerf, which we therefore omit.
        Across every evaluated representation and dataset in \cref{tab:nvs_results_other_rep}, equipping existing baselines with \qnn consistently improves rendering fidelity.

    %============================================================================
    %============================================================================
    \subsection{Ablations and Sensitivity Analysis}

        %============================================================================
        \noIndentHeading{Comparison of \qnn with a Sky Model.}
            A related baseline is the sky model \cite{rematas2022urban,sun2024pointnerf++}, which is an \mlp network with a sigmoid that maps view directions to a residual color.
            We compare \qnn with vanilla sky and Fourier encoded sky models in \cref{tab:nvs_sky}.
            \cref{tab:nvs_sky} shows that \qnn outperforms all \mlp-based sky models.

        %============================================================================
        \noIndentHeading{Residual Ablation.}
            \qnn adds the residual signal to the \lf splatted signal in \cref{fig:gs_qnn}.
            We next ablate this design in \cref{tab:nvs_residual_ablate} with no residuals (None) and multiplying residuals (Multiply).
            \cref{tab:nvs_residual_ablate} results show that \qnn model with adding the residual performs the best.

        %============================================================================
        \noIndentHeading{Sensitivity to Number of Layers.}
            Finally, we experiment with the number of convolution layers to pinpoint the optimal \qnn depth in \cref{tab:nvs_conv_layers}.
            We found that reducing the number of layers below four negatively impacted performance, while increasing the count beyond four did not provide any additional benefit.
            Hence, \qnn for \nvs uses four layers.

        %============================================================================
        \noIndentHeading{Sensitivity to \lf Signal Quality.}
            We used the final number of Gaussians in the \mcmc baseline \cite{kheradmand20243d} on \mipNerf dataset \cite{barron2022mipnerf360} as a proxy for the quality of the \lf signal, with a lower number of gaussians representing lower-fidelity signal, and a  higher number of gaussians representing higher-fidelity signal.
            As shown in \cref{fig:mcmc_scale_gaussians}, adding \qnn provides consistent \psnr improvements at varying quality of \lf signal. 
            This demonstrates the robustness of \qnn, affirming its ability to work with varying \lf signal quality.

        %============================================================================
        \noIndentHeading{Sensitivity to Initialization.}
            \gsOnly methods are sensitive to initialization \cite{kerbl2023gaussians}.
            We, therefore, analyze the impact of \sfm and random initialization on all nine scenes of the \mipNerf dataset \cite{barron2022mipnerf360}.
            We experiment with three baselines: \gs \cite{kerbl2023gaussians}, \mcmc \cite{kheradmand20243d} and \nht \cite{condor2026nht} in \cref{tab:nvs_init}.
            \cref{tab:nvs_init} results show that \qnn remains effective for both these initializations for all three baselines.

\begin{table}[!t]
    \centering
    \caption{\textbf{Comparison of adding sky model} \cite{rematas2022urban} and adding \qnn to the \gs baseline \cite{kerbl2023gaussians} on the \mipNerf dataset.
    \textbf{Adding \qnn outperforms \mlp-based sky models}.
    [Key: \firstBText{Best}~, \secondBText{Second-best}~, \thirdBText{Third-best}~].
    }
    \label{tab:nvs_sky}
    \scalebox{0.7}{
    \setlength\tabcolsep{0.1cm}
    \begin{tabular}{l m ccc m ccc }
        \multirow{2}{*}{Method} &  \multicolumn{3}{cm}{\textbf{Val}} & \multicolumn{3}{c}{\textbf{Train}}\\
        & \psnr \uparrowRHDSmall & \ssim \uparrowRHDSmall & \lpips \downarrowRHDSmall & \psnr \uparrowRHDSmall & \ssim \uparrowRHDSmall & \lpips \downarrowRHDSmall\\
        \myTopRule
        \gs \cite{kerbl2023gaussians} & \thirdF{27.67} & \secondF{0.82} & \firstF{0.20} & \secondF{29.71} & \firstF{0.90} & \firstF{0.16}\\
        \gs + Vanilla Sky & $27.63$ & \secondF{0.82} & \firstF{0.20} & \thirdF{29.67} &  \secondF{0.89} &	\firstF{0.16}\\
        \gs + Fourier Sky & \secondF{27.69} & \secondF{0.82} & \firstF{0.20} & \thirdF{29.67} &  \firstF{0.90} & \firstF{0.16}\\
        \hline
        \gs + \qnn & \firstF{27.96} & \firstF{0.83} & \firstF{0.20} & \firstF{30.11} & \firstF{0.90} & \firstF{0.16}\\
    \end{tabular}
    }
\end{table}

\begin{table}[!t]
    \centering
    \caption{\textbf{Residual ablation} of the \gs baseline \cite{kerbl2023gaussians} on the \mipNerf dataset.
    \textbf{Adding residuals within \qnn performs the best.}
    [Key: \firstBText{Best}~, \secondBText{Second-best}~, \thirdBText{Third-best}~].
    }
    \label{tab:nvs_residual_ablate}
    \scalebox{0.7}{
    \setlength\tabcolsep{0.1cm}
    \begin{tabular}{l l m ccc m ccc }
        \multirow{2}{*}{Method} & \multirow{2}{*}{Residual} &  \multicolumn{3}{cm}{\textbf{Val}} & \multicolumn{3}{c}{\textbf{Train}}\\
        & & \psnr \uparrowRHDSmall & \ssim \uparrowRHDSmall & \lpips \downarrowRHDSmall & \psnr \uparrowRHDSmall & \ssim \uparrowRHDSmall & \lpips \downarrowRHDSmall\\
        \myTopRule
        \gs \cite{kerbl2023gaussians} & \mathDash & \secondF{27.67} & \secondF{0.82} & \firstF{0.20} & \secondF{29.71} & \firstF{0.90} & \firstF{0.16}\\
        \gs + \qnn & None & 27.40 & \secondF{0.82} & \firstF{0.20} & 29.32 & \firstF{0.90} & \firstF{0.16}\\
        \gs + \qnn & Multiply & \thirdF{27.58} & \secondF{0.82} & \firstF{0.20} & \thirdF{29.56} & \firstF{0.90} & \firstF{0.16}\\
        \hline
        \gs + \qnn & Add & \firstF{27.96} & \firstF{0.83} & \firstF{0.20} & \firstF{30.11} & \firstF{0.90} & \firstF{0.16}\\
    \end{tabular}
    }
\end{table}

    %============================================================================
    %============================================================================
    \subsection{Extended \nvs Results}

        %============================================================================
        \noIndentHeading{Detailed \nvs Results.}
            We report a more comprehensive table for all scenes in \cref{tab:nvs_results_detail}.
            \cref{tab:nvs_results_detail} is thus the detailed version of \cref{tab:nvs_results}.

        %============================================================================
        \noIndentHeading{Indoor Outdoor Analysis.}
            We next investigate the question whether gains are more on indoor scenes or outdoor scenes in \cref{tab:nvs_indoor_outdoor}.
            This table shows the \psnr, \ssim and \lpips with \gs \cite{kerbl2023gaussians}, \mcmc \cite{kheradmand20243d} and \nht \cite{condor2026nht} baselines and their \qnn-augmented models on the \mipNerf dataset \cite{barron2022mipnerf360} and further categorized into indoor and outdoor scenes.
            \cref{tab:nvs_indoor_outdoor} confirms larger improvements are observed specifically in indoor scenes.

        %============================================================================
        \noIndentHeading{Spatial Breakdown of \psnr Gains.}
            To understand where \qnn improves rendering fidelity across, \cref{fig:psnr_gainloss} visualizes pixel-wise $\Delta$\psnr (\psnr{}$_{\text{+ \qnn}}~- $ \psnr{}$_{\text{base}}$) averaged across test views for \gs \cite{kerbl2023gaussians}, \mcmc \cite{kheradmand20243d}, and \nht \cite{condor2026nht} baselines on \mipNerf \cite{barron2022mipnerf360}. 
            We resize the gain maps to a common resolution, average across test views within each scene, and then average across scenes.
            Across all three baselines, the dominant blue regions show that \qnn yields spatially consistent \psnr gains rather than isolated artifact fixes, with the strongest improvements concentrated in upper and peripheral image regions where finite primitive budgets typically under-represent background structures and details. 
            Notably, \gs and \mcmc exhibit pronounced, widespread gains across the frame, whereas \nht, which is the stronger baseline, shows more localized improvements alongside minor localized trade-offs (faint red regions) near upper background regions. 
            Overall, these spatial distributions confirm that \qnn improves across the image plane with all baselines, particularly in upper left background regions.

        %============================================================================
        \noIndentHeading{Per-scene \nvs Results.}
            We report all numbers for all scenes in \cref{tab:nvs_results_perscene}.
            \cref{tab:nvs_results_perscene} is thus the per-scene version of \cref{tab:nvs_results}.

        %============================================================================
        \noIndentHeading{Variance Analysis.}
            We next check the statistical significance of our results by running with three different seeds and reporting the mean and standard deviation on the \mipNerf dataset \cite{barron2022mipnerf360} in \cref{tab:nvs_seed}.
            The table shows that integrating \qnn into both the \gsOnly baselines yields statistically significant improvements.

    %============================================================================
    %============================================================================
    \subsection{Single-Scale Training and Multi-Scale Testing}
        % %============================================================================
        % \noIndentHeading{Single-scale Training and Multi-scale Testing.}
            Contrary to most prior works that evaluate models trained on single-scale data at the same scale, we consider the new setting from \mipSplat \cite{yu2024mip} that involves training on downsampled resolution and testing on upsampled resolution. 
            To simulate zoom-in effects, we train models on data downsampled by a factor of $8$ and rendered at successively higher resolutions ($1 \times$, $2 \times$, $4 \times$, and $8 \times$) in \cref{tab:nvs_mip_stmt} following \mipSplat \cite{yu2024mip}. 
            \cref{tab:nvs_mip_stmt} results show that adding \qnn outperforms the baselines at the training scale ($1 \times$) and performs competitive to all \sota methods at higher resolutions. 
            Notably, adding \qnn to the \mcmc baseline again outperforms the \zipNerf baseline in multi-scale testing as \cref{tab:nvs_results} at increased resolutions, likely due to the \mlps' inability to extrapolate to out-of-distribution details, which agrees with \cref{tab:nvs_results}. 
            While adding \qnn performs sub-optimal to the \mipSplat \cite{yu2024mip}, these two directions are orthogonal to each other and one can add \qnn to the \mipSplat \cite{yu2024mip} as well.

\begin{table}[!t]
    \centering
    \caption{\textbf{Sensitivity to number of layers} in \qnn with \gs \cite{kerbl2023gaussians} on the \mipNerf dataset.
    \qnn uses $4$ layers.
    We refrain from highlighting if there are more than $5$ entries with the same number.
    [Key: \firstBText{Best}~, \secondBText{Second-best}~, \thirdBText{Third-best}~].
    }
    \label{tab:nvs_conv_layers}
    \scalebox{0.75}{
    \setlength\tabcolsep{0.1cm}
    \begin{tabular}{c m ccc m ccc }
        \multirow{2}{*}{Layers} &  \multicolumn{3}{cm}{\textbf{Val}} & \multicolumn{3}{c}{\textbf{Train}}\\
        & \psnr (\uparrowRHDSmall) & \ssim (\uparrowRHDSmall) & \lpips (\downarrowRHDSmall) & \psnr (\uparrowRHDSmall) & \ssim (\uparrowRHDSmall) & \lpips (\downarrowRHDSmall)\\
        \myTopRule
        $0$ & $27.67$ & \secondF{0.82} & $0.20$ & $29.71$ & $0.90$ & $0.16$\\
        $2$ & $27.81$ & \secondF{0.82} & $0.20$ & $29.85$ & $0.90$ & $0.16$\\
        $3$ & $27.89$ & \firstF{0.83} & $0.20$ & $30.01$ & $0.90$ & $0.16$\\
        $4$ & \firstF{27.96} & \firstF{0.83} & $0.20$ & \thirdF{30.11} & $0.90$ & $0.16$\\
        $5$ & \secondF{27.91} & \firstF{0.83} & $0.20$ & \firstF{30.13} & $0.90$ & $0.16$\\
        $6$ & \thirdF{27.90} & \firstF{0.83} & $0.20$ & \secondF{30.12} & $0.90$ & $0.16$
    \end{tabular}
    }
\end{table}

\begin{table*}[!t]
    \centering
    \caption{\textbf{Sensitivity to initialization} on baseline and \qnn models on the \mipNerf dataset.
    \textbf{Adding \qnn outperforms} the baselines on both val and train images across initializations.
    [Key: \firstBText{Best}~, \secondBText{Second-best}~, \thirdBText{Third-best}~, Init= Initialization].
    }
    \label{tab:nvs_init}
    \scalebox{0.75}{
    \setlength\tabcolsep{0.2cm}
    \begin{tabular}{l l l m lll m lll }
        \multirow{2}{*}{Method} &  \multirow{2}{*}{Init} &  \multirow{2}{*}{\qnn}  & \multicolumn{3}{cm}{\textbf{Val}} & \multicolumn{3}{c}{\textbf{Train}}\\
        & & & \psnr (\uparrowRHDSmall) & \ssim (\uparrowRHDSmall) & \lpips (\downarrowRHDSmall) & \psnr (\uparrowRHDSmall) & \ssim (\uparrowRHDSmall) & \lpips (\downarrowRHDSmall)\\
        \myTopRule
        \multirow{4}{*}{\gs\cite{kerbl2023gaussians}} & \multirow{2}{*}{\sfm} & \xmark & $27.67$ & ${0.82}$ & $0.20$ & $29.71$ & $0.90$ & ${0.16}$\\
        &  & \cmark & ${27.96}$\good{}{0.29} & \thirdF{0.83}\good{}{0.01} & \nothing{0.20} & $30.11$\good{}{0.40} & \nothing{0.90} & ${0.16}$\nothing{}\\
        \hhline{|~|--------|}
        & \multirow{2}{*}{Random} & \xmark & $26.28$ & $0.77$ & $0.26$ & $28.19$ & $0.85$ & $0.21$\\
        &  & \cmark & $26.82$\good{}{0.54} & $0.79$\good{}{0.02} & $0.23$\good{}{0.03} & $29.85$\good{}{1.66} & ${0.89}$\good{}{0.04} & $0.17$\good{}{0.04}\\
        \hline
        \multirow{4}{*}{\mcmc\cite{kheradmand20243d}} & \multirow{2}{*}{\sfm}   & \xmark & ${28.26}$                   & \secondF{0.84}            & \secondF{0.17}             & ${30.31}$                   & \thirdF{0.91}               & \thirdF{0.14}\\
                                                      &                         & \cmark & ${28.58}$\good{}{0.32}      & \secondF{0.84}\nothing{}  & \firstF{0.16}\good{}{0.01} & ${30.69}$\good{}{0.38}      &  \thirdF{0.91}\nothing{}    & \thirdF{0.14}\nothing{}\\
        \hhline{|~|--------|}
                                                      & \multirow{2}{*}{Random} & \xmark & $27.28$                     & $0.81$                    &  \thirdF{0.20}   & $30.07$                      & ${0.90}$                   & ${0.15}$\\
                                                      &                         & \cmark & $27.51$\good{}{0.23}        & \nothing{0.81}            & \thirdF{0.20}\nothing{}    & ${30.38}$\good{}{0.31}       & \thirdF{0.91}\good{}{0.01} & ${0.15}$\nothing{}\\
        \hline
        \multirow{4}{*}{\nht\cite{condor2026nht}}     & \multirow{2}{*}{\sfm}   & \xmark & \secondF{29.04}             & \firstF{0.85}             & \secondF{0.17} & ${32.99}$ & \secondF{0.93}               & \firstF{0.12}\\
        &                                                                       & \cmark & \firstF{29.19}\good{}{0.15} & \firstF{0.85}\nothing{}   & \secondF{0.17}\nothing{}   & \thirdF{33.22}\good{}{0.23}  & \firstF{0.94}\good{}{0.01} & \firstF{0.11}\good{}{0.01}\\
        \hhline{|~|--------|}
                                                      & \multirow{2}{*}{Random} & \xmark & $28.86$                     & \secondF{0.84}             & \secondF{0.17}             & \secondF{33.24}             & \firstF{0.94}              & \firstF{0.11}\\
                                                      &                         & \cmark & \thirdF{29.03}\good{}{0.17} & \firstF{0.85}\good{}{0.01} & \firstF{0.16}\good{}{0.01} & \firstF{33.54}\good{}{0.30} & \firstF{0.94}\nothing{}    & \firstF{0.11}\nothing{} \\
    \end{tabular}
    }
\end{table*}

\begin{table*}[!t]
    \caption{
        \textbf{Detailed \nvs Results.} 
        \textbf{Adding \qnn outperforms all five baselines on all six datasets.}
        This is the detailed version of \cref{tab:nvs_results}.
        Numbers are from respective papers.
        Stable-GS, 3D-HGS and \mcmc do not report results on all nine scenes of \mipNerf dataset. 
        \zipNerf \synNerf numbers are from their official repository.
        [Key: \firstBText{Best}~, \secondBText{Second-best}~, \thirdBText{Third-best}~, \released= Reported, \retrained= Retrained, \taken= Reported in \gs \cite{kerbl2023gaussians}].
    }
    \label{tab:nvs_results_detail}
    \centering
    \scalebox{0.73}{
    \setlength{\tabcolsep}{0.1cm}
    \begin{tabular}{c @{\hspace{0.3cm}} lc c @{\hspace{0.3cm}} c c @{\hspace{0.3cm}} c @{\hspace{0.3cm}} c @{\hspace{0.3cm}} c} 
        \myTopRule
        & \multirow{2}{*}{Method} & \multirow{2}{*}{Venue} 
        & \textbf{\mipNerf \cite{barron2022mipnerf360}} 
        & \textbf{\tAndT \cite{knapitsch2017tanks}} 
        & \textbf{\deepBlend \cite{hedman2018deep}} 
        & \textbf{\ommo \cite{lu2023large}} 
        & \textbf{\shelly \cite{wang2023adaptiveshells}} 
        & \textbf{\synNerf \cite{mildenhall2020nerf}} \\
        & & & \footnotesize{\psnr\!$\uparrowRHDSmallTwo$\sep\ssim\!$\uparrowRHDSmallTwo$\sep\lpips\!$\downarrowRHDSmallTwo$}
        & \footnotesize{\psnr\!$\uparrowRHDSmallTwo$\sep\ssim\!$\uparrowRHDSmallTwo$\sep\lpips\!$\downarrowRHDSmallTwo$}
        & \footnotesize{\psnr\!$\uparrowRHDSmallTwo$\sep\ssim\!$\uparrowRHDSmallTwo$\sep\lpips\!$\downarrowRHDSmallTwo$}
        & \footnotesize{\psnr\!$\uparrowRHDSmallTwo$\sep\ssim\!$\uparrowRHDSmallTwo$\sep\lpips\!$\downarrowRHDSmallTwo$}
        & \footnotesize{\psnr\!$\uparrowRHDSmallTwo$\sep\ssim\!$\uparrowRHDSmallTwo$\sep\lpips\!$\downarrowRHDSmallTwo$}
        & \footnotesize{\psnr\!$\uparrowRHDSmallTwo$\sep\ssim\!$\uparrowRHDSmallTwo$\sep\lpips\!$\downarrowRHDSmallTwo$} \\
        \myTopRule 
        \multirow{11}{*}{\rotatebox{90}{\textbf{Ray}}}
        & Plenoxels\taken \cite{fridovich2022plenoxels} & \venue{\cvpr22} & $23.08$ \sep $0.63$ \sep $0.44$ & $21.08$ \sep $0.72$ \sep $0.38$ & $23.06$ \sep $0.80$ \sep $0.51$ & \mathDash & \mathDash & $31.76$ \sep ~~~\mathDash~~ \sep ~~\mathDash~~ \\
        & INGP-Big\taken \cite{muller2022instant}      & \venue{\sigg22} & $25.59$ \sep $0.75$ \sep $0.30$ & $21.92$ \sep $0.75$ \sep $0.31$ & $24.96$ \sep $0.82$ \sep $0.39$ & \mathDash & \mathDash & $33.18$ \sep ~~~\mathDash~~ \sep ~~\mathDash~~ \\
        & 3DGRT \cite{moenne2024gaussray} & \venue{\sigg25} & $27.20$ \sep $0.82$ \sep $0.25$ & $23.20$ \sep $0.83$ \sep $0.22$ & $29.23$ \sep \thirdB{0.90} \sep $0.32$ & \mathDash & \mathDash & \mathDash \\
        & 3DGUT \cite{wu20253dgut} & \venue{\cvpr25}  & $27.26$ \sep $0.81$ \sep $0.22$ & $22.90$ \sep $0.84$ \sep $0.17$  & \mathDash & \mathDash & \mathDash & \mathDash \\
        & VKRayGS \cite{bulo2025hardware} & \venue{\cvpr25} & $27.27$ \sep $0.82$ \sep $0.22$ & \mathDash & \mathDash & \mathDash & \mathDash & \mathDash \\
        & \ever~\cite{mai2025ever} & \venue{\iccv25}  & $27.51$ \sep \thirdB{0.83} \sep $0.23$ & \mathDash & \mathDash & \mathDash & \mathDash & \mathDash \\
        & Mip-\nerf\!\taken \cite{barron2021mip} & \venue{\cvpr22} & $27.69$ \sep $0.79$ \sep $0.24$ & $22.22$ \sep $0.76$ \sep $0.26$ & $29.40$ \sep \thirdB{0.90} \sep \thirdB{0.25}  & \mathDash & \mathDash & \mathDash \\
        & Adaptive \cite{wang2023adaptiveshells} & \venue{\sigg23} & \mathDash & \mathDash & \mathDash & \mathDash & $36.02$ \sep $0.95$ \sep $0.08$ & $32.51$ \sep $0.96$ \sep $0.05$ \\
        & QFields \cite{sharma2024volumetric} & \venue{\eccv24} & \mathDash & \mathDash & \mathDash & \mathDash & $37.29$ \sep $0.95$ \sep \thirdB{0.07} & $31.00$ \sep $0.95$ \sep $0.07$ \\
        & 3DGEER \cite{huang20263dgeer} & \venue{\iclr26} & $27.76$ \sep $0.82$ \sep $0.21$ & \mathDash & \mathDash & \mathDash & \mathDash & \mathDash \\
        & \zipNerf \!\cite{barron2023zip}           & \venue{\iccv23} & $28.54$ \sep \thirdB{0.83} \sep ${0.19}$ & \mathDash & \mathDash & \mathDash & \mathDash & $33.69$ \sep \thirdB{0.97} \sep ~~\mathDash~~ \\
        \myTopRule
        \multirow{28}{*}{\rotatebox{90}{\textbf{Raster}}}
        & GES \cite{hamdi2024ges} & \venue{\cvpr24} & $26.91$ \sep $0.79$ \sep $0.25$ & $23.35$ \sep $0.84$ \sep $0.20$ & $29.68$ \sep \thirdB{0.90} \sep \thirdB{0.25} & \mathDash & \mathDash & \mathDash \\
        & Stable-GS \cite{wang2025stablegs} & \venue{\arxiv25} & \mathDash            & $24.04$ \sep \thirdB{0.86} \sep $0.16$ & $29.66$ \sep \secondB{0.91} \sep \secondB{0.24} & \mathDash & \mathDash & \mathDash \\
        & Convex \cite{held2025convex} & \venue{\cvpr25} & $27.29$ \sep $0.80$ \sep $0.21$ & $23.95$ \sep $0.85$ \sep $0.16$ & $29.81$ \sep \thirdB{0.90} \sep \secondB{0.24} & \mathDash & \mathDash & \mathDash \\
        & Vol3DGS \cite{talegaonkar2025volumetrically} & \venue{\cvpr25} & $27.30$ \sep $0.81$ \sep $0.21$ & $23.74$ \sep $0.85$ \sep $0.17$ & $29.72$ \sep \secondB{0.91} \sep \thirdB{0.25} & \mathDash & \mathDash & \mathDash\\
        & Tex-GS \cite{chao2025textured} & \venue{\cvpr25} & $27.35$ \sep \thirdB{0.83} \sep ${0.19}$ & $24.26$ \sep $0.85$ \sep $0.17$ & $28.33$ \sep ${0.89}$ \sep $0.36$ & \mathDash & \mathDash & $33.24$ \sep \thirdB{0.97} \sep $0.04$\\
        & DARB \cite{arunan2025darb} & \venue{\arxiv25} & $27.45$ \sep $0.81$ \sep $0.21$ & $23.64$ \sep $0.85$ \sep $0.17$ & $29.63$ \sep \thirdB{0.90} \sep \secondB{0.24} & \mathDash & \mathDash & \mathDash \\
        & 3D-HGS \cite{li2025half} & \venue{\cvpr25} & \mathDash           & ${25.08}$ \sep $0.84$ \sep \thirdB{0.14} & $29.80$ \sep \thirdB{0.90} \sep \thirdB{0.25} & \mathDash & \mathDash & \mathDash \\
        & Project \cite{huang2024error} & \venue{\eccv24} & $27.48$ \sep $0.82$ \sep $0.21$ & $23.43$ \sep ~~\mathDash~~ \sep ~~\mathDash~~ & $29.51$ \sep  ~~\mathDash~~ \sep ~~\mathDash~~ & \mathDash & \mathDash & \mathDash \\
        & \vdgs~\cite{malarz2025gaussian} & \venue{CVIU25} & $27.64$ \sep $0.81$ \sep $0.22$ & $24.02$ \sep $0.85$ \sep $0.18$ & $29.54$ \sep \secondB{0.91} \sep \secondB{0.24} & \mathDash & \mathDash & \thirdB{35.97} \sep \firstB{0.99} \sep \firstB{0.01} \\ 
        & Revise \cite{rota2024revising} & \venue{\eccv24} & $27.70$ \sep $0.82$ \sep $0.22$ & $24.10$ \sep \thirdB{0.86} \sep $0.18$ & $29.64$ \sep \secondB{0.91} \sep $0.30$ & \mathDash & \mathDash & \mathDash \\
        & HyRF \cite{wang2025hyrf} & \venue{\nips25} & $27.78$ \sep $0.82$ \sep $0.21$ & $24.02$ \sep $0.84$ \sep $0.18$ & \thirdB{30.37} \sep \secondB{0.91} \sep \secondB{0.24} & \mathDash & \mathDash & \mathDash \\
        & \mipSplat~\cite{yu2024mip} & \venue{\cvpr24} & $27.79$ \sep \thirdB{0.83} \sep $0.20$ & \mathDash & \mathDash & \mathDash & \mathDash & \mathDash \\
        & FreGS \cite{zhang2024fregs} & \venue{\cvpr24} & $27.85$ \sep \thirdB{0.83} \sep $0.21$ & $23.96$ \sep $0.85$ \sep $0.18$ & ${29.93}$ \sep \thirdB{0.90} \sep \secondB{0.24} & \mathDash & \mathDash & \mathDash \\
        \hhline{|~|--------|}
        & \gs \released \cite{kerbl2023gaussians}      & \venue{\sigg23} & $27.20$ \sep $0.82$ \sep $0.21$  & $23.14$ \sep $0.84$ \sep $0.18$ & $29.41$ \sep \thirdB{0.90} \sep \secondB{0.24} & \mathDash & \mathDash & $33.31$ \sep ~~\mathDash~~ \sep ~~\mathDash~~ \\
        & \gs \retrained \cite{kerbl2023gaussians} & \venue{\sigg23} & $27.67$ \sep $0.82$ \sep $0.20$  & $23.52$ \sep $0.84$ \sep $0.18$ & $29.50$ \sep \thirdB{0.90} \sep \secondB{0.24} & $29.03$ \sep ${0.90}$ \sep ${0.16}$ & ${37.56}$ \sep \thirdB{0.96} \sep ${0.08}$ & $34.83$ \sep \secondB{0.98} \sep \thirdB{0.03} \\
        & \gs + \qnn & \mathDash & $27.96$ \sep \thirdB{0.83} \sep $0.20$ & $24.11$ \sep $0.85$ \sep $0.17$  & $29.67$ \sep \secondB{0.91} \sep \thirdB{0.25} & $29.37$ \sep \thirdB{0.91} \sep ${0.15}$ & ${37.99}$ \sep \thirdB{0.96} \sep \thirdB{0.07} & $35.36$ \sep \secondB{0.98} \sep \thirdB{0.03} \\
        \hhline{|~|--------|}
        & \mcmc \released \cite{kheradmand20243d}       & \venue{\nips24} & \mathDash            & $24.29$ \sep \thirdB{0.86} \sep $0.19$ & $29.67$ \sep $0.89$ \sep $0.32$ & \thirdB{29.52} \sep \secondB{0.91} \sep $0.20$ & \mathDash & $33.80$ \sep \thirdB{0.97} \sep $0.04$ \\
        & \mcmc \retrained \cite{kheradmand20243d}  & \venue{\nips24} & $28.26$ \sep \secondB{0.84} \sep \secondB{0.17}  & ${24.53}$ \sep \secondB{0.87} \sep \secondB{0.13} & $29.32$ \sep \secondB{0.91} \sep \secondB{0.24} & ${30.10}$ \sep \secondB{0.92} \sep \secondB{0.12} & $35.14$ \sep ${0.94}$ \sep ${0.09}$ & \secondB{36.14} \sep \secondB{0.98} \sep \secondB{0.02} \\
        & \mcmc + \qnn                                   & \mathDash & $28.58$ \sep \secondB{0.84} \sep \firstB{0.16}  & ${24.87}$ \sep \firstB{0.88} \sep \firstB{0.12} & $29.76$ \sep \secondB{0.91} \sep \firstB{0.23} & \thirdB{30.34} \sep \secondB{0.92} \sep \secondB{0.12} & $35.57$ \sep $0.95$ \sep $0.09$ &  \firstB{36.58} \sep \secondB{0.98} \sep \secondB{0.02} \\
        \hhline{|~|--------|}
        & \dbs \released \cite{liu2025deformable} & \venue{\sigg25}  & $28.60$ \sep \secondB{0.84} \sep $0.18$ & $24.79$ \sep \secondB{0.87} \sep $0.15$ & $30.10$ \sep $0.91$ \sep$0.24$ & \mathDash  & \mathDash & $34.64$ \sep \thirdB{0.97} \sep $0.03$\\
        & \dbs \retrained \cite{liu2025deformable} & \venue{\sigg25}  & $28.38$ \sep \secondB{0.84} \sep \thirdB{0.18} & $24.70$ \sep \secondB{0.87} \sep $0.15$ & $29.75$ \sep \secondB{0.91} \sep $0.24$ & $28.67$ \sep $0.89$ \sep$0.15$ & \secondB{42.16} \sep \firstB{0.98} \sep \firstB{0.04} & $34.57$ \sep \thirdB{0.97} \sep \thirdB{0.03}\\
        & \dbs+ \qnn & \mathDash & $28.72$ \sep \secondB{0.84} \sep \secondB{0.17} & ${25.16}$ \sep \secondB{0.87} \sep \thirdB{0.14} & ${29.94}$ \sep \secondB{0.91} \sep \firstB{0.23} & ${28.97}$ \sep ${0.90}$ \sep ${0.14}$ & \firstB{42.98} \sep \firstB{0.98} \sep \firstB{0.04} & $34.76$ \sep \thirdB{0.97} \sep \thirdB{0.03}\\
        \hhline{|~|--------|}
        & \sv \released \cite{di2026spherical} & \venue{\cvpr26} & $28.71$ \sep \secondB{0.84} \sep $0.23$ & $25.00$ \sep \secondB{0.87} \sep $0.17$ & $30.63$ \sep $0.92$ \sep $0.30$ & \mathDash & \mathDash &  $34.58$ \sep \thirdB{0.97} \sep $0.03$\\ 
        & \sv \retrained \cite{di2026spherical} & \venue{\cvpr26} & $28.71$ \sep \secondB{0.84} \sep $0.19$ & $24.98$ \sep \secondB{0.87} \sep \thirdB{0.14} & $30.00$ \sep \secondB{0.91} \sep \secondB{0.24} & $29.94$ \sep \thirdB{0.91} \sep \thirdB{0.13} & $41.17$ \sep \secondB{0.97} \sep \firstB{0.04} & $34.82$ \sep \thirdB{0.97} \sep \thirdB{0.03}\\
        & \sv + \qnn & \mathDash & \thirdB{28.83} \sep \secondB{0.84} \sep $0.19$ & \thirdB{25.33} \sep \firstB{0.88} \sep \secondB{0.13} & \firstB{30.71} \sep \firstB{0.92} \sep \firstB{0.23} & $30.19$ \sep \thirdB{0.91} \sep \thirdB{0.13} & \thirdB{41.58} \sep \secondB{0.97} \sep \firstB{0.04} & $34.90$ \sep \thirdB{0.97} \sep \thirdB{0.03}\\
        \hhline{|~|--------|}
        & \nht \released \cite{condor2026nht} & \venue{\eccv26} & $28.74$ \sep \secondB{0.84} \sep $0.22$ & $25.68$ \sep \firstB{0.88} \sep $0.14$ & $30.94$ \sep $0.92$ \sep $0.30$  & \mathDash & \mathDash & \mathDash\\
        & \nht \retrained \cite{condor2026nht} & \venue{\eccv26} & \secondB{29.04} \sep \firstB{0.85} \sep \secondB{0.17} & \secondB{25.72} \sep \firstB{0.88} \sep \firstB{0.12} & ${30.17}$ \sep \secondB{0.91} \sep \thirdB{0.25} & \secondB{31.50} \sep \firstB{0.93} \sep \secondB{0.12} 
        & $37.64$ \sep \thirdB{0.96} \sep \secondB{0.06} & $33.19$ \sep \thirdB{0.97} \sep $0.04$ \\
        & \nht + \qnn & \mathDash                                & \firstB{29.19} \sep \firstB{0.85} \sep \secondB{0.17} & \firstB{25.79} \sep \firstB{0.88} \sep \firstB{0.12} & \secondB{30.64} \sep \firstB{0.92} \sep \secondB{0.24} & \firstB{31.74} \sep \firstB{0.93} \sep \firstB{0.11} & $37.67$ \sep \thirdB{0.96} \sep \secondB{0.06} & $33.33$ \sep \thirdB{0.97} \sep $0.04$ \\
        \myTopRule
    \end{tabular}
    }
    \vspace{-0.3cm}
\end{table*}

\begin{table*}[!t]
    \caption{\textbf{Indoor-Outdoor Analysis} for the \nvs task on the \mipNerf dataset.
    \textbf{Larger improvements are observed in indoor scenes}
    [Key: \firstBText{Best}~, \secondBText{Second-best}~, \thirdBText{Third-best}~].
    }
    \label{tab:nvs_indoor_outdoor}
    \centering
    \scalebox{0.75}{
    \setlength{\arrayrulewidth}{0.01cm}
    \setlength{\tabcolsep}{0.12cm}
    \begin{tabular}{l l c cc cc cc}
        \myTopRule
        & \multirow{2}{*}{Scene}
        & \textbf{\zipNerf \cite{barron2023zip}}
        & \textbf{\gs\!\retrained~\cite{kerbl2023gaussians}} 
        & \textbf{\gs + \qnn} 
        & \textbf{\mcmc\!\retrained~\cite{kheradmand20243d}} 
        & \textbf{\mcmc + \qnn}
        & \textbf{\nht\!\retrained~\cite{condor2026nht}} 
        & \textbf{\nht + \qnn}\\
        & & \footnotesize{\psnr\!$\uparrowRHDSmallTwo$\sep\ssim\!$\uparrowRHDSmallTwo$\sep\lpips\!$\downarrowRHDSmallTwo$}
        % & & \tiny{\psnr\!$\uparrow$\sep\ssim\!$\uparrow$\sep\lpips\!$\downarrow$}
        & \footnotesize{\psnr\!$\uparrowRHDSmallTwo$\sep\ssim\!$\uparrowRHDSmallTwo$\sep\lpips\!$\downarrowRHDSmallTwo$}
        & \footnotesize{\psnr\!$\uparrowRHDSmallTwo$\sep\ssim\!$\uparrowRHDSmallTwo$\sep\lpips\!$\downarrowRHDSmallTwo$}
        & \footnotesize{\psnr\!$\uparrowRHDSmallTwo$\sep\ssim\!$\uparrowRHDSmallTwo$\sep\lpips\!$\downarrowRHDSmallTwo$}
        & \footnotesize{\psnr\!$\uparrowRHDSmallTwo$\sep\ssim\!$\uparrowRHDSmallTwo$\sep\lpips\!$\downarrowRHDSmallTwo$}
        & \footnotesize{\psnr\!$\uparrowRHDSmallTwo$\sep\ssim\!$\uparrowRHDSmallTwo$\sep\lpips\!$\downarrowRHDSmallTwo$}
        & \footnotesize{\psnr\!$\uparrowRHDSmallTwo$\sep\ssim\!$\uparrowRHDSmallTwo$\sep\lpips\!$\downarrowRHDSmallTwo$}\\
        \myTopRule
        \multirow{11}{*}{\rotatebox{90}{\phantom{\textbf{\mipNerf \cite{barron2022mipnerf360}}}}}
        & Garden   & $28.20$ \sep $0.86$ \sep $0.12$ & $27.67$ \sep $0.87$ \sep $0.10$ & $27.88$ \sep $0.87$ \sep $0.10$ & $28.00$ \sep $0.88$ \sep $0.10$ & $28.25$ \sep $0.88$ \sep $0.09$ & $28.53$ \sep $0.88$ \sep $0.09$ & $28.55$ \sep $0.88$ \sep $0.09$\\
        & Bicycle  & $25.80$ \sep $0.77$ \sep $0.21$ & $25.60$ \sep $0.77$ \sep $0.20$ & $25.68$ \sep $0.77$ \sep $0.20$ & $25.97$ \sep $0.80$ \sep $0.17$ & $25.99$ \sep $0.80$ \sep $0.16$ & $26.01$ \sep $0.80$ \sep $0.17$ & $26.12$ \sep $0.80$ \sep $0.17$\\
        & Stump    & $27.55$ \sep $0.80$ \sep $0.19$ & $26.88$ \sep $0.78$ \sep $0.20$ & $27.18$ \sep $0.79$ \sep $0.20$ & $27.17$ \sep $0.80$ \sep $0.18$ & $27.56$ \sep $0.81$ \sep $0.18$ & $27.52$ \sep $0.81$ \sep $0.17$ & $27.65$ \sep $0.81$ \sep $0.17$\\
        & Bonsai   & $34.46$ \sep $0.95$ \sep $0.17$ & $32.11$ \sep $0.95$ \sep $0.18$ & $32.60$ \sep $0.95$ \sep $0.18$ & $33.25$ \sep $0.96$ \sep $0.15$ & $33.77$ \sep $0.96$ \sep $0.15$ & $35.27$ \sep $0.96$ \sep $0.14$ & $35.51$ \sep $0.96$ \sep $0.14$\\
        & Counter  & $29.38$ \sep $0.90$ \sep $0.19$ & $29.10$ \sep $0.92$ \sep $0.18$ & $29.28$ \sep $0.92$ \sep $0.18$ & $29.86$ \sep $0.93$ \sep $0.15$ & $29.98$ \sep $0.93$ \sep $0.15$ & $30.69$ \sep $0.93$ \sep $0.15$ & $30.86$ \sep $0.93$ \sep $0.15$\\
        & Kitchen  & $32.50$ \sep $0.93$ \sep $0.12$ & $31.39$ \sep $0.93$ \sep $0.11$ & $31.80$ \sep $0.93$ \sep $0.11$ & $32.30$ \sep $0.94$ \sep $0.10$ & $32.81$ \sep $0.94$ \sep $0.10$ & $33.69$ \sep $0.94$ \sep $0.10$ & $33.73$ \sep $0.95$ \sep $0.10$\\
        & Room     & $32.65$ \sep $0.93$ \sep $0.20$ & $31.59$ \sep $0.93$ \sep $0.19$ & $32.40$ \sep $0.93$ \sep $0.19$ & $32.28$ \sep $0.94$ \sep $0.17$ & $33.31$ \sep $0.94$ \sep $0.16$ & $33.64$ \sep $0.94$ \sep $0.16$ & $34.07$ \sep $0.95$ \sep $0.16$\\
        & Treehill & $23.89$ \sep $0.68$ \sep $0.27$ & $22.81$ \sep $0.65$ \sep $0.31$ & $22.75$ \sep $0.65$ \sep $0.31$ & $23.27$ \sep $0.67$ \sep $0.27$ & $23.34$ \sep $0.67$ \sep $0.27$ & $23.53$ \sep $0.68$ \sep $0.25$ & $23.61$ \sep $0.68$ \sep $0.25$\\
        & Flowers  & $22.40$ \sep $0.64$ \sep $0.24$ & $21.89$ \sep $0.62$ \sep $0.33$ & $21.97$ \sep $0.62$ \sep $0.32$ & $22.22$ \sep $0.65$ \sep $0.27$ & $22.23$ \sep $0.65$ \sep $0.27$ & $22.52$ \sep $0.66$ \sep $0.27$ & $22.61$ \sep $0.66$ \sep $0.27$\\
        \hhline{|~|--------|}
        \hline
        & \textbf{Indoor}   & $32.25$ \sep \thirdB{0.93} \sep \secondB{0.17} & $31.05$ \sep \thirdB{0.93} \sep \secondB{0.17} & $31.52$ \sep \thirdB{0.93} \sep \secondB{0.17} & $31.92$ \sep \secondB{0.94} \sep \firstB{0.14} & \thirdB{32.47} \sep \secondB{0.94} \sep \firstB{0.14} & \secondB{33.32} \sep \firstB{0.95} \sep \firstB{0.14} & \firstB{33.54} \sep \firstB{0.95} \sep \firstB{0.14}\\
        & \textbf{Outdoor} & \thirdB{25.57} \sep \thirdB{0.75} \sep \thirdB{0.21} & $24.97$ \sep $0.74$ \sep $0.23$ & $25.09$ \sep $0.74$ \sep $0.23$ & $25.33$ \sep \secondB{0.76} \sep \secondB{0.20} & $25.47$ \sep \secondB{0.76} \sep \firstB{0.19} & \secondB{25.62} \sep \firstB{0.77} \sep \firstB{0.19} & \firstB{25.71} \sep \firstB{0.77} \sep \firstB{0.19}\\
        \hline
        & Average~~~~~~& $28.54$ \sep \thirdB{0.83} \sep \thirdB{0.19} & $27.67$ \sep $0.82$ \sep $0.20$ & $27.96$ \sep \thirdB{0.83} \sep $0.20$ & $28.26$ \sep \secondB{0.84} \sep \secondB{0.17} & \thirdB{28.58} \sep \secondB{0.84} \sep \firstB{0.16} & \secondB{29.04} \sep \firstB{0.85} \sep \secondB{0.17} & \firstB{29.19} \sep \firstB{0.85} \sep \secondB{0.17}\\
        \myTopRule
    \end{tabular}
    }
\end{table*}

\begin{table*}[!t]
    \caption{
        \textbf{Per-scene \nvs Results.}
        This is the per-scene version of \cref{tab:nvs_results}.
        \textbf{Adding \qnn outperforms} the baselines across all datasets. 
        % \zipNerf \synNerf numbers are from their official repository.
        [Key: \firstBText{Best}~, \secondBText{Second-best}~, \thirdBText{Third-best}~, \retrained= Retrained, T\&T= Tanks and Temples, DB= \deepBlend{}].
    }
    \label{tab:nvs_results_perscene}
    \centering
    \vspace{-0.2cm}
    \scalebox{0.75}{
    \setlength{\arrayrulewidth}{0.01cm}
    \setlength{\tabcolsep}{0.12cm}
    \begin{tabular}{l l c cc cc cc}
        \myTopRule
        & \multirow{2}{*}{Scene}
        & \textbf{\zipNerf \cite{barron2023zip}}
        & \textbf{\gs\!\retrained~\cite{kerbl2023gaussians}} 
        & \textbf{\gs + \qnn} 
        & \textbf{\mcmc\!\retrained~\cite{kheradmand20243d}} 
        & \textbf{\mcmc + \qnn}
        & \textbf{\nht\!\retrained~\cite{condor2026nht}} 
        & \textbf{\nht + \qnn}\\
        & & \footnotesize{\psnr\!$\uparrowRHDSmallTwo$\sep\ssim\!$\uparrowRHDSmallTwo$\sep\lpips\!$\downarrowRHDSmallTwo$}
        % & & \tiny{\psnr\!$\uparrow$\sep\ssim\!$\uparrow$\sep\lpips\!$\downarrow$}
        & \footnotesize{\psnr\!$\uparrowRHDSmallTwo$\sep\ssim\!$\uparrowRHDSmallTwo$\sep\lpips\!$\downarrowRHDSmallTwo$}
        & \footnotesize{\psnr\!$\uparrowRHDSmallTwo$\sep\ssim\!$\uparrowRHDSmallTwo$\sep\lpips\!$\downarrowRHDSmallTwo$}
        & \footnotesize{\psnr\!$\uparrowRHDSmallTwo$\sep\ssim\!$\uparrowRHDSmallTwo$\sep\lpips\!$\downarrowRHDSmallTwo$}
        & \footnotesize{\psnr\!$\uparrowRHDSmallTwo$\sep\ssim\!$\uparrowRHDSmallTwo$\sep\lpips\!$\downarrowRHDSmallTwo$}
        & \footnotesize{\psnr\!$\uparrowRHDSmallTwo$\sep\ssim\!$\uparrowRHDSmallTwo$\sep\lpips\!$\downarrowRHDSmallTwo$}
        & \footnotesize{\psnr\!$\uparrowRHDSmallTwo$\sep\ssim\!$\uparrowRHDSmallTwo$\sep\lpips\!$\downarrowRHDSmallTwo$}\\[0.05cm]
        \hline
        \\[-0.25cm]
        \multirow{11}{*}{\rotatebox{90}{\textbf{\mipNerf \cite{barron2022mipnerf360}}}}
        & Garden   & $28.20$ \sep $0.86$ \sep $0.12$ & $27.67$ \sep $0.87$ \sep $0.10$ & $27.88$ \sep $0.87$ \sep $0.10$ & $28.00$ \sep $0.88$ \sep $0.10$ & $28.25$ \sep $0.88$ \sep $0.09$ & $28.53$ \sep $0.88$ \sep $0.09$ & $28.55$ \sep $0.88$ \sep $0.09$\\
        & Bicycle  & $25.80$ \sep $0.77$ \sep $0.21$ & $25.60$ \sep $0.77$ \sep $0.20$ & $25.68$ \sep $0.77$ \sep $0.20$ & $25.97$ \sep $0.80$ \sep $0.17$ & $25.99$ \sep $0.80$ \sep $0.16$ & $26.01$ \sep $0.80$ \sep $0.17$ & $26.12$ \sep $0.80$ \sep $0.17$\\
        & Stump    & $27.55$ \sep $0.80$ \sep $0.19$ & $26.88$ \sep $0.78$ \sep $0.20$ & $27.18$ \sep $0.79$ \sep $0.20$ & $27.17$ \sep $0.80$ \sep $0.18$ & $27.56$ \sep $0.81$ \sep $0.18$ & $27.52$ \sep $0.81$ \sep $0.17$ & $27.65$ \sep $0.81$ \sep $0.17$\\
        & Bonsai   & $34.46$ \sep $0.95$ \sep $0.17$ & $32.11$ \sep $0.95$ \sep $0.18$ & $32.60$ \sep $0.95$ \sep $0.18$ & $33.25$ \sep $0.96$ \sep $0.15$ & $33.77$ \sep $0.96$ \sep $0.15$ & $35.27$ \sep $0.96$ \sep $0.14$ & $35.51$ \sep $0.96$ \sep $0.14$\\
        & Counter  & $29.38$ \sep $0.90$ \sep $0.19$ & $29.10$ \sep $0.92$ \sep $0.18$ & $29.28$ \sep $0.92$ \sep $0.18$ & $29.86$ \sep $0.93$ \sep $0.15$ & $29.98$ \sep $0.93$ \sep $0.15$ & $30.69$ \sep $0.93$ \sep $0.15$ & $30.86$ \sep $0.93$ \sep $0.15$\\
        & Kitchen  & $32.50$ \sep $0.93$ \sep $0.12$ & $31.39$ \sep $0.93$ \sep $0.11$ & $31.80$ \sep $0.93$ \sep $0.11$ & $32.30$ \sep $0.94$ \sep $0.10$ & $32.81$ \sep $0.94$ \sep $0.10$ & $33.69$ \sep $0.94$ \sep $0.10$ & $33.73$ \sep $0.95$ \sep $0.10$\\
        & Room     & $32.65$ \sep $0.93$ \sep $0.20$ & $31.59$ \sep $0.93$ \sep $0.19$ & $32.40$ \sep $0.93$ \sep $0.19$ & $32.28$ \sep $0.94$ \sep $0.17$ & $33.31$ \sep $0.94$ \sep $0.16$ & $33.64$ \sep $0.94$ \sep $0.16$ & $34.07$ \sep $0.95$ \sep $0.16$\\
        & Treehill & $23.89$ \sep $0.68$ \sep $0.27$ & $22.81$ \sep $0.65$ \sep $0.31$ & $22.75$ \sep $0.65$ \sep $0.31$ & $23.27$ \sep $0.67$ \sep $0.27$ & $23.34$ \sep $0.67$ \sep $0.27$ & $23.53$ \sep $0.68$ \sep $0.25$ & $23.61$ \sep $0.68$ \sep $0.25$\\
        & Flowers  & $22.40$ \sep $0.64$ \sep $0.24$ & $21.89$ \sep $0.62$ \sep $0.33$ & $21.97$ \sep $0.62$ \sep $0.32$ & $22.22$ \sep $0.65$ \sep $0.27$ & $22.23$ \sep $0.65$ \sep $0.27$ & $22.52$ \sep $0.66$ \sep $0.27$ & $22.61$ \sep $0.66$ \sep $0.27$\\
        \hhline{|~|--------|}
        & Average          & $28.54$ \sep \thirdB{0.83} \sep \thirdB{0.19} & $27.67$ \sep $0.82$ \sep $0.20$ & $27.96$ \sep \thirdB{0.83} \sep $0.20$ & $28.26$ \sep \secondB{0.84} \sep \secondB{0.17} & \thirdB{28.58} \sep \secondB{0.84} \sep \firstB{0.16} & \secondB{29.04} \sep \firstB{0.85} \sep \secondB{0.17} & \firstB{29.19} \sep \firstB{0.85} \sep \secondB{0.17}\\[0.07cm]
        \myTopRule
        \\[-0.25cm]
        \multirow{3}{*}{\rotatebox{90}{\textbf{T\&T}}}
        & Truck   & \mathDash & $25.24$ \sep $0.88$ \sep $0.15$ & $25.97$ \sep $0.88$ \sep $0.14$ & $26.34$ \sep $0.90$ \sep $0.10$ & $26.81$ \sep $0.90$ \sep $0.09$ & $27.01$ \sep $0.90$ \sep $0.09$ & $27.16$ \sep $0.90$ \sep $0.09$\\
        & Train   & \mathDash & $21.81$ \sep $0.81$ \sep $0.21$ & $22.25$ \sep $0.81$ \sep $0.20$ & $22.71$ \sep $0.85$ \sep $0.15$ & $22.93$ \sep $0.85$ \sep $0.14$ & $24.43$ \sep $0.86$ \sep $0.15$ & $24.42$ \sep $0.87$ \sep $0.14$\\
        \hhline{|~|--------|}
        & Average & \mathDash & $23.52$ \sep $0.84$ \sep $0.18$ & $24.11$ \sep $0.85$ \sep \thirdB{0.17} & $24.53$ \sep \thirdB{0.87} \sep \secondB{0.13} & \thirdB{24.87} \sep \secondB{0.88} \sep \firstB{0.12} & \secondB{25.72} \sep \secondB{0.88} \sep \firstB{0.12} & \firstB{25.79} \sep \firstB{0.88} \sep \firstB{0.12} \\[0.07cm]
        \myTopRule
        \\[-0.25cm]
        \multirow{3}{*}{\rotatebox{90}{\textbf{DB \cite{hedman2018deep}}}}
        & Dr Johnson & \mathDash & $29.05$ \sep $0.90$ \sep $0.24$ & $29.20$ \sep $0.90$ \sep $0.24$ & $29.44$ \sep $0.91$ \sep $0.23$ & $29.34$ \sep $0.91$ \sep $0.23$ & $29.66$ \sep $0.91$ \sep $0.25$ & $30.07$ \sep $0.91$ \sep $0.25$\\
        & Playroom   & \mathDash & $29.95$ \sep $0.91$ \sep $0.25$ & $30.12$ \sep $0.91$ \sep $0.25$ & $29.20$ \sep $0.91$ \sep $0.25$ & $30.17$ \sep $0.91$ \sep $0.23$ & $30.69$ \sep $0.91$ \sep $0.25$ & $31.20$ \sep $0.92$ \sep $0.25$\\
        \hhline{|~|--------|}
        & Average    & \mathDash & $29.50$ \sep \thirdB{0.90} \sep \secondB{0.24} & $29.67$ \sep \secondB{0.91} \sep \thirdB{0.25} & $29.32$ \sep \secondB{0.91} \sep \secondB{0.24} & \thirdB{29.76} \sep \secondB{0.91} \sep \firstB{0.23} & \secondB{30.17} \sep \secondB{0.91} \sep \thirdB{0.25} & \firstB{30.64} \sep \firstB{0.92} \sep \secondB{0.24} \\[0.07cm]
        \myTopRule
        \\[-0.25cm]
        \multirow{9}{*}{\rotatebox{90}{\textbf{\ommo \cite{lu2023large}}}}
        & $01$    & \mathDash & $25.59$ \sep $0.78$ \sep $0.22$ & $25.60$ \sep $0.79$ \sep $0.21$ & $25.63$ \sep $0.79$ \sep $0.19$ & $25.80$ \sep $0.80$ \sep $0.18$ & $27.48$ \sep $0.83$ \sep $0.16$ & $27.57$ \sep $0.84$ \sep $0.15$\\
        & $03$    & \mathDash & $27.96$ \sep $0.90$ \sep $0.18$ & $29.66$ \sep $0.92$ \sep $0.15$ & $30.42$ \sep $0.94$ \sep $0.12$ & $31.54$ \sep $0.95$ \sep $0.11$ & $32.70$ \sep $0.95$ \sep $0.10$ & $33.62$ \sep $0.96$ \sep $0.09$\\
        & $05$    & \mathDash & $27.74$ \sep $0.86$ \sep $0.22$ & $28.10$ \sep $0.87$ \sep $0.22$ & $28.73$ \sep $0.88$ \sep $0.18$ & $28.99$ \sep $0.88$ \sep $0.17$ & $29.54$ \sep $0.89$ \sep $0.19$ & $29.77$ \sep $0.89$ \sep $0.18$\\
        & $06$    & \mathDash & $26.97$ \sep $0.92$ \sep $0.16$ & $27.02$ \sep $0.93$ \sep $0.16$ & $28.30$ \sep $0.95$ \sep $0.12$ & $28.40$ \sep $0.95$ \sep $0.12$ & $30.17$ \sep $0.95$ \sep $0.12$ & $30.25$ \sep $0.95$ \sep $0.12$\\
        & $10$    & \mathDash & $29.85$ \sep $0.90$ \sep $0.17$ & $29.95$ \sep $0.90$ \sep $0.16$ & $32.05$ \sep $0.92$ \sep $0.12$ & $31.91$ \sep $0.92$ \sep $0.12$ & $32.74$ \sep $0.92$ \sep $0.13$ & $32.90$ \sep $0.92$ \sep $0.13$\\
        & $13$    & \mathDash & $33.35$ \sep $0.96$ \sep $0.10$ & $33.29$ \sep $0.96$ \sep $0.10$ & $33.78$ \sep $0.96$ \sep $0.09$ & $33.83$ \sep $0.96$ \sep $0.08$ & $34.63$ \sep $0.96$ \sep $0.09$ & $34.72$ \sep $0.96$ \sep $0.09$\\
        & $14$    & \mathDash & $30.97$ \sep $0.95$ \sep $0.10$ & $31.23$ \sep $0.95$ \sep $0.09$ & $31.49$ \sep $0.95$ \sep $0.08$ & $31.70$ \sep $0.96$ \sep $0.08$ & $32.83$ \sep $0.96$ \sep $0.08$ & $33.06$ \sep $0.96$ \sep $0.08$\\
        & $15$    & \mathDash & $29.83$ \sep $0.93$ \sep $0.10$ & $30.16$ \sep $0.94$ \sep $0.09$ & $30.41$ \sep $0.94$ \sep $0.08$ & $30.56$ \sep $0.94$ \sep $0.08$ & $31.91$ \sep $0.95$ \sep $0.07$ & $32.01$ \sep $0.95$ \sep $0.07$\\
        \hhline{|~|--------|}
        & Average & \mathDash & $29.03$ \sep $0.90$ \sep $0.16$ & $29.37$ \sep \thirdB{0.91} \sep \thirdB{0.15} & $30.10$ \sep \secondB{0.92} \sep \secondB{0.12} & \thirdB{30.34} \sep \secondB{0.92} \sep \secondB{0.12} & \secondB{31.50} \sep \firstB{0.93} \sep \secondB{0.12} & \firstB{31.74} \sep \firstB{0.93} \sep \firstB{0.11} \\[0.07cm]
        \myTopRule
        \\[-0.25cm]
        \multirow{7}{*}{\rotatebox{90}{\textbf{\shelly \cite{wang2023adaptiveshells}}}}
        & Fern   & \mathDash & $38.69$ \sep $0.99$ \sep $0.03$ & $39.10$ \sep $0.99$ \sep $0.02$ & $37.55$ \sep $0.98$ \sep $0.03$ & $37.67$ \sep $0.98$ \sep $0.03$ & $38.99$ \sep $0.99$ \sep $0.03$ & $38.98$ \sep $0.99$ \sep $0.02$\\
        & Horse  & \mathDash & $44.56$ \sep $0.99$ \sep $0.03$ & $45.30$ \sep $0.99$ \sep $0.03$ & $40.25$ \sep $0.99$ \sep $0.04$ & $40.58$ \sep $0.99$ \sep $0.04$ & $43.02$ \sep $0.99$ \sep $0.03$ & $43.03$ \sep $0.99$ \sep $0.03$\\
        & Khady  & \mathDash & $32.42$ \sep $0.90$ \sep $0.16$ & $32.61$ \sep $0.90$ \sep $0.16$ & $33.32$ \sep $0.91$ \sep $0.14$ & $33.42$ \sep $0.91$ \sep $0.13$ & $32.98$ \sep $0.90$ \sep $0.12$ & $32.84$ \sep $0.90$ \sep $0.12$\\
        & Kitten & \mathDash & $39.64$ \sep $0.98$ \sep $0.04$ & $40.12$ \sep $0.98$ \sep $0.04$ & $39.48$ \sep $0.98$ \sep $0.04$ & $39.61$ \sep $0.98$ \sep $0.04$ & $40.04$ \sep $0.98$ \sep $0.04$ & $40.03$ \sep $0.98$ \sep $0.04$\\
        & Pug    & \mathDash & $37.36$ \sep $0.95$ \sep $0.10$ & $37.64$ \sep $0.95$ \sep $0.10$ & $26.65$ \sep $0.86$ \sep $0.21$ & $28.12$ \sep $0.87$ \sep $0.20$ & $36.45$ \sep $0.94$ \sep $0.09$ & $36.56$ \sep $0.95$ \sep $0.09$\\
        & Woolly & \mathDash & $32.71$ \sep $0.93$ \sep $0.11$ & $33.18$ \sep $0.94$ \sep $0.10$ & $33.61$ \sep $0.94$ \sep $0.10$ & $33.99$ \sep $0.95$ \sep $0.09$ & $34.42$ \sep $0.95$ \sep $0.07$ & $34.57$ \sep $0.95$ \sep $0.07$\\
        \hhline{|~|--------|}
        & Average & \mathDash & $37.56$ \sep \firstB{0.96} \sep \thirdB{0.08} & \firstB{37.99} \sep \firstB{0.96} \sep \secondB{0.07} & $35.14$ \sep \thirdB{0.94} \sep $0.09$ & $35.57$ \sep \secondB{0.95} \sep $0.09$ & \thirdB{37.64} \sep \firstB{0.96} \sep \firstB{0.06} & \secondB{37.67} \sep \firstB{0.96} \sep \firstB{0.06}\\[0.07cm]
        \myTopRule
        \\[-0.25cm]
        \multirow{9}{*}{\rotatebox{90}{\textbf{\synNerf \cite{mildenhall2020nerf}}}}
        & Chair     & $35.78$ \sep $0.99$ \sep ~~\mathDash~~ & $37.44$ \sep $0.99$ \sep $0.01$ & $37.71$ \sep $0.99$ \sep $0.01$ & $38.38$ \sep $0.99$ \sep $0.01$ & $38.74$ \sep $0.99$ \sep $0.01$ & $36.34$ \sep $0.99$ \sep $0.01$ & $36.35$ \sep $0.99$ \sep $0.01$\\
        & Drum      & $25.91$ \sep $0.95$ \sep ~~\mathDash~~ & $27.71$ \sep $0.96$ \sep $0.03$ & $27.84$ \sep $0.97$ \sep $0.03$ & $28.10$ \sep $0.97$ \sep $0.03$ & $28.23$ \sep $0.99$ \sep $0.03$ & $25.79$ \sep $0.95$ \sep $0.04$ & $25.74$ \sep $0.95$ \sep $0.04$\\
        & Ficus     & $34.72$ \sep $0.99$ \sep ~~\mathDash~~ & $35.49$ \sep $0.99$ \sep $0.01$ & $37.66$ \sep $0.99$ \sep $0.01$ & $36.54$ \sep $0.99$ \sep $0.01$ & $38.08$ \sep $0.99$ \sep $0.01$ & $35.24$ \sep $0.99$ \sep $0.01$ & $35.72$ \sep $0.99$ \sep $0.01$\\
        & Hotdog    & $38.05$ \sep $0.99$ \sep ~~\mathDash~~ & $38.72$ \sep $0.99$ \sep $0.02$ & $38.89$ \sep $0.99$ \sep $0.02$ & $39.78$ \sep $0.99$ \sep $0.01$ & $39.85$ \sep $0.99$ \sep $0.01$ & $38.29$ \sep $0.99$ \sep $0.02$ & $38.27$ \sep $0.99$ \sep $0.02$\\
        & Lego      & $35.79$ \sep $0.98$ \sep ~~\mathDash~~ & $37.63$ \sep $0.99$ \sep $0.01$ & $38.05$ \sep $0.99$ \sep $0.01$ & $39.26$ \sep $0.99$ \sep $0.01$ & $39.39$ \sep $0.99$ \sep $0.01$ & $36.34$ \sep $0.98$ \sep $0.01$ & $36.45$ \sep $0.99$ \sep $0.01$\\
        & Materials & $31.05$ \sep $0.97$ \sep ~~\mathDash~~ & $32.39$ \sep $0.98$ \sep $0.03$ & $33.02$ \sep $0.98$ \sep $0.02$ & $34.02$ \sep $0.98$ \sep $0.02$ & $34.50$ \sep $0.99$ \sep $0.02$ & $29.91$ \sep $0.96$ \sep $0.04$ & $30.03$ \sep $0.96$ \sep $0.04$\\
        & Mic       & $35.92$ \sep $0.99$ \sep ~~\mathDash~~ & $36.91$ \sep $0.99$ \sep $0.01$ & $37.13$ \sep $0.99$ \sep $0.01$ & $39.65$ \sep $1.00$ \sep $0.00$ & $40.14$ \sep $0.99$ \sep $0.00$ & $36.24$ \sep $0.99$ \sep $0.01$ & $36.35$ \sep $0.99$ \sep $0.01$\\
        & Ship      & $32.33$ \sep $0.94$ \sep ~~\mathDash~~ & $32.32$ \sep $0.92$ \sep $0.11$ & $32.54$ \sep $0.92$ \sep $0.11$ & $33.45$ \sep $0.94$ \sep $0.08$ & $33.72$ \sep $0.99$ \sep $0.07$ & $27.37$ \sep $0.87$ \sep $0.12$ & $27.72$ \sep $0.88$ \sep $0.12$\\
        \hhline{|~|--------|}
        & Average   & $33.69$ \sep \secondB{0.97} \sep ~~\mathDash~~ & $34.83$ \sep \firstB{0.98} \sep \secondB{0.03} & \thirdB{35.36} \sep \firstB{0.98} \sep \secondB{0.03} & \secondB{36.14} \sep \firstB{0.98} \sep \firstB{0.02} & \firstB{36.58} \sep \firstB{0.98} \sep \firstB{0.02} & $33.19$ \sep \secondB{0.97} \sep \thirdB{0.04} & $33.33$ \sep \secondB{0.97} \sep \thirdB{0.04} \\[0.07cm]
        \myTopRule
    \end{tabular}
    }
\end{table*}

\begin{figure*}[!t]
    \centering
    \makebox[0.32\linewidth]{\small \gs \cite{kerbl2023gaussians}}%
    \makebox[0.32\linewidth]{\small \mcmc \cite{kheradmand20243d}}%
    \makebox[0.32\linewidth]{\small \nht \cite{condor2026nht}}%
    \hspace*{0.04\linewidth}\par
    \vspace{-0.01cm}
    \includegraphics[width=\linewidth]{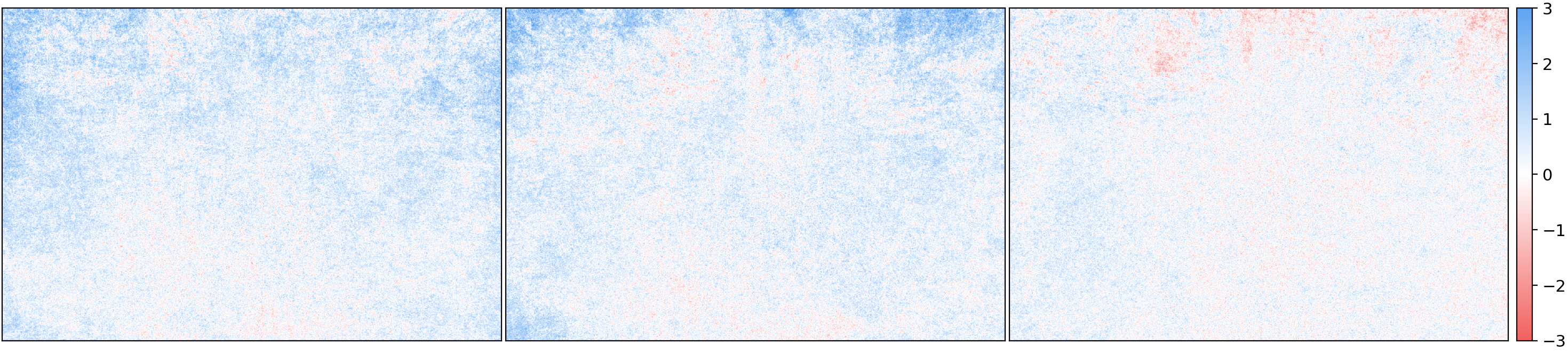}
    \caption{
        \textbf{Spatial distribution of \psnr gain/loss when adding \qnn} to \gs \cite{kerbl2023gaussians} (left), \mcmc \cite{kheradmand20243d} (middle), and \nht \cite{condor2026nht} (right) baselines on the \mipNerf dataset \cite{barron2022mipnerf360}. 
        Blue denotes positive \psnr gains (\psnr{}$_{\text{+ \qnn}} >$ \psnr{}$_{\text{base}}$) and red denotes \psnr degradation (\psnr{}$_{\text{+ \qnn}} <$ \psnr{}$_{\text{base}}$), clamped to $[-3, 3]\,\text{dB}$. 
        \textbf{\qnn improves across the image plane with all baselines, particularly in upper left background regions}.
    }
    \label{fig:psnr_gainloss}
\end{figure*}

\begin{table*}[!t]
    \centering
    \caption{\textbf{Variance Analysis} on the \mipNerf dataset.
    \textbf{\qnn outperforming} the baselines is statistically significant.
    [Key: \firstBText{Baseline+\qnn}~, \secondBText{Baseline}~, Avg= Average, $\sigma$= Standard Deviation].
    }
    \label{tab:nvs_seed}
    \scalebox{0.75}{
    \setlength\tabcolsep{0.2cm}
    \begin{tabular}{l l m lll m lll }
        \multirow{2}{*}{Method} & \multirow{2}{*}{Seed} & \multicolumn{3}{cm}{\textbf{Val}} & \multicolumn{3}{c}{\textbf{Train}}\\
        & & \psnr (\uparrowRHDSmall) & \ssim (\uparrowRHDSmall) & \lpips (\downarrowRHDSmall) & \psnr (\uparrowRHDSmall) & \ssim (\uparrowRHDSmall) & \lpips (\downarrowRHDSmall)\\
        \myTopRule
        \multirow{4}{*}{\gs\!\cite{kerbl2023gaussians}} & $222$ & $27.67$ & $0.82$ & $0.20$ & $29.71$ & $0.90$ & $0.16$\\
        & $111$ & $27.65$ & $0.82$ & $0.20$ & $29.69$ & $0.90$ & $0.16$\\
        & $555$ & $27.72$ & $0.82$ & $0.20$ & $29.69$ & $0.90$ & $0.16$\\
        \rowcolor{\secondColor}
        & \avg $\pm~\sigma$ & $27.68\pm0.03$ & $0.82\pm0.00$ & $0.20\pm0.00$ & $29.69\pm0.02$ & $0.90\pm0.00$ & $0.16\pm0.00$\\
        \myTopRule
        \multirow{4}{*}{\gs + \qnn} & $222$ & $27.96$ & $0.83$ & $0.20$ & $30.11$ & $0.90$ & $0.16$\\
        & $111$ & $27.92$ & $0.83$ & $0.20$ & $30.11$ & $0.90$ & $0.16$\\
        & $555$ & $27.93$ & $0.83$ & $0.20$ & $30.12$ & $0.90$ & $0.16$\\
        \rowcolor{\firstColor}
        & \avg $\pm~\sigma$ & $27.94\pm0.03$ & $0.83\pm0.00$ & $0.20\pm0.00$ & $30.12\pm0.02$ & $0.90\pm0.00$ & $0.16\pm0.00$\\
        \myTopRule
        \multirow{4}{*}{\mcmc\!\cite{kheradmand20243d}} & $222$ & $28.26$ & $0.84$ & $0.17$ & $30.31$ & $0.91$ & $0.14$\\
        & $111$ & $28.22$ & $0.84$ & $0.17$ & $30.32$ & $0.91$ & $0.13$\\
        & $555$ & $28.31$ & $0.84$ & $0.17$ & $30.31$ & $0.91$ & $0.13$\\
        \rowcolor{\secondColor}
        & \avg $\pm~\sigma$ & $28.26\pm0.03$ & $0.84\pm0.00$ & $0.17\pm0.00$ & $30.31\pm0.02$ & $0.91\pm0.00$ & $0.13\pm0.00$\\
        \myTopRule
        \multirow{4}{*}{\mcmc + \qnn} & $222$ & $28.58$ & $0.84$ & $0.16$ & $30.69$ & $0.91$ & $0.14$\\
        & $111$ & $28.53$ & $0.84$ & $0.16$ & $30.67$ & $0.91$ & $0.13$\\
        & $555$ & $28.54$ & $0.84$ & $0.16$ & $30.69$ & $0.91$ & $0.13$\\
        \rowcolor{\firstColor}
        & \avg $\pm~\sigma$ & $28.55\pm0.02$ & $0.84\pm0.00$ & $0.16\pm0.00$ & $30.67\pm0.02$ & $0.91\pm0.00$ & $0.13\pm0.00$\\
    \end{tabular}
    }
\end{table*}

\begin{table*}[!t]
    \caption{
    \textbf{Single-scale Training and Multi-scale Testing on the \mipNerf dataset.} 
    All methods are trained on the smallest scale ($1 \times$) and evaluated across four scales ($1 \times$, $2 \times$, $4 \times$, and $8 \times$), with evaluations at higher sampling rates simulating zoom-in effects.
    [Key: \firstBText{Best}~, \secondBText{Second-Best}~, \thirdBText{Third-Best}~, \retrained= Retrained, \taken= Reported in \mipSplat \cite{yu2024mip}, Avg= Average].
    }
    \label{tab:nvs_mip_stmt}
    \centering
    \scalebox{0.75}{
    \setlength\tabcolsep{0.2cm}
    \begin{tabular}{l l m cccc|c m cccc|c m cccc|c }
        & \multirow{2}{*}{Method} & \multicolumn{5}{cm}{\textbf{\psnr} (\uparrowRHDSmall)} & \multicolumn{5}{cm}{\textbf{\ssim} (\uparrowRHDSmall)} & \multicolumn{5}{c}{\textbf{\lpips} (\downarrowRHDSmall)}  \\
        & & $1 \times$ & $2 \times$ & $4 \times$ & $8 \times$ & \avg & $1 \times$ & $2 \times$ & $4 \times$ & $8 \times$ & \avg & $1 \times$ & $2 \times$ & $4 \times$ & $8 \times$ & \avg \\
        \myTopRule
        \multirow{3}{*}{\rotatebox{90}{\textbf{Ray}}} & INGP\taken~\cite{muller2022instant} &  $26.79$  &  \thirdF{24.76}  & \secondF{24.27}  & \secondF{24.27}  &  \thirdF{25.02}  & $0.75$  & $0.64$  & \thirdF{0.63}  & \thirdF{0.70}  & $0.68$  & $0.24$  & $0.37$  & $0.45$  & $0.48$  &  $0.38$\\
        & \mipNerfMethod{}\taken~\cite{barron2022mipnerf360}& $29.26$  & \secondF{25.18}  & \thirdF{24.16}  & \thirdF{24.10}  & \secondF{25.67}  & \thirdF{0.86}  & \secondF{0.73}  & \secondF{0.67} & \secondF{0.71}  & \secondF{0.74}  & \thirdF{0.12}  & $0.26$  & \secondF{0.37} & \secondF{0.43}  &  $0.30$\\
        & \zipNerf{}\taken~\cite{barron2023zip}& \secondF{29.66}  & $23.27$  & $20.87$  & $20.27$  & $23.52$  & \firstF{0.88}  & $0.70$  & $0.57$  & $0.56$  & $0.67$  &\firstF{0.10}  & $0.26$  & $0.42$  & $0.49$  & $0.32$\\
        % \gs~\cite{kerbl2023gaussians} + EWA~\cite{zwicker2001ewa} & 29.30  &\cellcolor{orange}25.90  & 23.70  & 22.81  &25.43  &\cellcolor{orange}0.880  &\cellcolor{orange}0.775  &0.667  & 0.643  &\cellcolor{orange}0.741  & 0.114  &\cellcolor{orange}0.236  &\cellcolor{orange}0.369  &0.449  &\cellcolor{orange}0.292\\
        \myTopRule
        \multirow{5}{*}{\rotatebox{90}{\textbf{Raster}}}
        & \mipSplat{}\taken~\cite{yu2024mip} & \thirdF{29.39}  & \firstF{27.39}  & \firstF{26.47}  &\firstF{26.22}  & \firstF{27.37}  & \firstF{0.88}  & \firstF{0.81} & \firstF{0.75}  & \firstF{0.77}  & \firstF{0.80}  & \secondF{0.11}  &\firstF{0.21}  & \firstF{0.31}  &\firstF{0.39}  & \firstF{0.25} \\
        & \gs\retrained~\cite{kerbl2023gaussians} & $28.91$  & $23.16$  & $20.39$  & $19.31$  & $22.94$  & \secondF{0.87}  &$0.71$  & $0.59$  & $0.60$  & $0.69$  & \secondF{0.11}  & $0.25$   & $0.40$   & $0.48$  & $0.31$ \\
        & \gs + \qnn & $29.30$ & $23.52$ & $20.85$ & $19.83$ & $23.37$ & \firstF{0.88} & \thirdF{0.72} & $0.60$ & $0.61$ & $0.71$ & \secondF{0.11} & \thirdF{0.24} & $0.40$ & $0.48$ & $0.31$\\
        \hhline{|~|----------------|}
        & \mcmc (1M)~\retrained \cite{kheradmand20243d} & $29.30$ & $23.67$ & $21.15$ & $20.22$ & $23.58$ & \firstF{0.88} & \secondF{0.73} & $0.62$ & $0.64$ & $0.72$ & \secondF{0.11} & \secondF{0.23} & \thirdF{0.38} & \thirdF{0.46} & \thirdF{0.30} \\
        & \mcmc (1M) + \qnn & \firstF{29.67} & $24.27$ & $21.96$ & $21.17$ & $24.26$ & \firstF{0.88} & \secondF{0.73} & \thirdF{0.63} & $0.66$ & \thirdF{0.73} & \firstF{0.10} & \secondF{0.23} & \thirdF{0.38} & \thirdF{0.46} & \secondF{0.29} \\
    \end{tabular}
    }
\end{table*}

\clearpage

%============================================================================
%============================================================================
%============================================================================
\section{Additional Architectural Discussions}\label{sec:supp_discussions}

        In this section, we discuss multi-view \threeD consistency, followed by significance of \psnr gains, and compatibility with standard graphics and \threeD editing pipelines.

        %============================================================================
        \subsection{Multi-view Consistency}
            Although \qnn operates without explicit cross-view regularization, integrating it into \gsOnly consistently improves held-out validation metrics (\psnr, \ssim, and \lpips) without hallucinating view-inconsistent artifacts. 
            This \threeD consistency stems from the following four reasons:
            \begin{itemize}
                \item \textbf{Single-Scene, Non-Generative Fitting:} Unlike pretrained generative priors or diffusion post-processors that can hallucinate plausible but view-inconsistent textures from external datasets, \qnn is a lightweight, scene-specific network trained strictly from scratch on the target scene's multi-view observations.
                \item \textbf{Bounded Residual Refinement over 3D-Consistent Base Signals:} \qnn predicts a zero-initialized, bounded residual atop the coarse \gsOnly rendering. 
                Because the underlying \threeD Gaussian primitives already enforce projective geometric consistency across views, the \qnn only refines local detail rather than synthesizing scene structure from scratch.
                \item \textbf{Multi-View Supervision:} Optimizing end-to-end against ground-truth multi-view images directly incentivizes the network to learn view-consistent refinements.
                \item \textbf{Smooth Interpolation Regime:} In standard \nvs benchmarks, test viewpoints lie within the interpolation hull of training cameras. Because coordinate-conditioned convolutions vary smoothly with respect to input queries and local \lf patches, \qnn interpolates seamlessly across adjacent viewpoints.
            \end{itemize}

        %============================================================================
        \subsection{Regularization to keep \qnn from Hallucinating View-inconsistent Detail}
            There is no explicit regularization to prevent the \qnn from hallucinating view-inconsistent details across different viewpoints.
            However, our improved empirical performance on novel views suggests that added details are not view-inconsistent hallucinations (see previous remark).

        %============================================================================
        \subsection{Significance of Gains}
            At the frontier, $28\text{--}29$\,\db on the \mipNerf dataset \cite{barron2022mipnerf360}, \psnr scales logarithmically while primitive count grows exponentially (\cref{th:num_gaussians}).
            \cref{fig:mcmc_scale_gaussians_trainval,fig:mcmc_scale_gaussians} show that scaling \mcmc \cite{kheradmand20243d} from $4$M to $20$M Gaussian primitives ($+16$M primitives, $>5\times$ memory) yields zero validation \psnr gain due to overfitting. 
            By contrast, a compact \qnn delivers $+0.2\text{--}0.3$\,\db gains across all budgets.
            Furthermore, adding \qnn consistently pushes \dbs, \sv, and \nht (\cref{tab:nvs_results}) baselines by $+0.1\text{--}0.3$\,\db.

        %============================================================================
        \subsection{Integration into Standard Graphics and Rasterization Pipelines}
            \qnn requires zero modification to \threeD sorting, tile rasterization, or $\alpha$-compositing CUDA kernels. 
            Operating strictly on the rasterized \lf framebuffer $\nFieldPdLP \in \realDomain^{3\times H\times W}$ and fixed \twoD coordinate grid $\query \in \realDomain^{2\times H\times W}$, \qnn acts as a modular, screen-space neural shader pass, analogous to post-rasterization supersampling \cite{xiao2020neural} or denoising passes in modern graphics engines.  
            Consequently, \qnn can be appended to any existing \gsOnly rasterizer in a few lines of code while maintaining real-time rendering frame rates (\cref{tab:nvs_edge_memory_inference}).

        %============================================================================
        \subsection{Integration into 3D Editing Pipelines}
            Our formulation $\nFieldGSQNN = \nFieldPdLP + \cnn(\query\concat\nFieldPdLP)$ also remains compatible with \threeD editing: 
            \begin{itemize}
                \item \textbf{Strict locality:} Small receptive fields ensure edits in one region never affect pixels outside that object's immediate boundary.
                \item \textbf{Base-driven updates:} Translating or deleting \threeD primitives immediately updates $\nFieldPdLP$ via the rasterizer. 
                \item \textbf{Interactive toggle:} Users can edit $\nFieldPdLP$ at native rasterization speeds and toggle the \qnn pass on demand.
            \end{itemize}

% \newpage
% {
%     \small
%     \bibliographystyle{ieeenat_fullname}
%     \bibliography{references}
% }

\end{document}